\documentclass[twoside]{article}
\usepackage[margin=1in]{geometry}
\usepackage[toc,page,header]{appendix}
\usepackage{etoc}
\usepackage[]{amsmath,amssymb,epsfig}
\usepackage{amsthm}
\usepackage{amsmath,mathtools}
\usepackage{amssymb}
\usepackage{graphicx}
\usepackage{epstopdf}
\usepackage{comment}
\usepackage{array}
\usepackage{algorithm}
\usepackage{algpseudocode}
\usepackage{url}
\usepackage{lmodern}

\usepackage{lscape}
\usepackage{algpseudocode}
\usepackage{setspace}
\usepackage{multicol}
\usepackage{multirow}
\usepackage{color}
\usepackage{colortbl}
\usepackage{xcolor}
\usepackage{hyperref}
\usepackage{booktabs}
\usepackage{pgf}
\usepackage{placeins}

\usepackage{calligra}

\usepackage[%
    font={small,sf},
    labelfont=bf,
    format=hang,    
    format=plain,
    margin=0pt,
    width=0.8\textwidth,
]{caption}
\usepackage[list=true]{subcaption}
\usepackage[round]{natbib}

\usepackage{titletoc}
\makeatletter
\let\oldaddcontentsline\addcontentsline 
\makeatother

\newtheorem{theorem}{Theorem}

\newtheorem{claim}{Claim}

\newtheorem{corollary}{Corollary}

\newtheorem{definition}{Definition}

\newtheorem{lemma}{Lemma}

\newtheorem{remark}{Remark}

\numberwithin{equation}{section}

\renewcommand{\algorithmicrequire}{\textbf{Input:}}
\renewcommand{\algorithmicensure}{\textbf{Output:}}

\DeclareMathOperator{\Tr}{Tr}

\DeclareMathOperator{\diag}{diag}
\DeclareMathOperator{\blkdiag}{blkdiag}

\DeclareMathOperator{\vect}{vec}

\newcommand{\calB}{\ensuremath{\mathcal{B}}}
\newcommand{\calC}{\ensuremath{\mathcal{C}}}

\newcommand{\calH}{\ensuremath{\mathcal{H}}}

\newcommand{\calR}{\ensuremath{\mathcal{R}}}
\newcommand{\calS}{\ensuremath{\mathcal{S}}}
\newcommand{\calM}{\ensuremath{\mathcal{M}}}
\newcommand{\calN}{\ensuremath{\mathcal{N}}}
\newcommand{\calT}{\ensuremath{\mathcal{T}}}
\newcommand{\calU}{\ensuremath{\mathcal{U}}}
\newcommand{\calK}{\ensuremath{\mathcal{K}}}

\newcommand{\calV}{\ensuremath{\mathcal{V}}}

\newcommand{\calE}{\ensuremath{\mathcal{E}}}

\newcommand{\norm}[1]{\big\|{#1}\big\|}

\newcommand{\ceil}[1]{\left\lceil{#1}\right\rceil}

\newcommand{\set}[1]{\left\{{#1}\right\}}
\newcommand{\dotprod}[2]{\Big\langle#1,#2\Big\rangle}
\newcommand{\est}[1]{\widehat{#1}}
\newcommand{\expec}{\ensuremath{\mathbb{E}}}
\newcommand{\matR}{\ensuremath{\mathbb{R}}}

\newcommand{\argmin}[1]{\underset{#1}{\operatorname{argmin}}}

\newcommand{\prob}{\ensuremath{\mathbb{P}}}

\newcommand{\frobsphere}{\mathbb{S}_n}
\newcommand{\frobball}{\calB_n}
\newcommand{\tanconeB}{\calT_{\calK,B}}
\newcommand{\tancone}[1]{\calT_{\calK,#1}}

\newcommand{\tanconeAstar}{\calT_{\calK,A^*}}
\newcommand{\conesetB}{\calC_{B}}
\newcommand{\conesetAstar}{\calC_{A^*}}

\newcommand{\gammafn}[1]{\ensuremath{\gamma_{#1}}}

\newcommand{\eventp}[1]{\ensuremath{\calE_{#1}}}
\newcommand{\matrixenc}[1]{\ensuremath{\begin{bmatrix} #1 \end{bmatrix}}}
\newcommand{\absenc}[1]{\ensuremath{\left | #1 \right |}}

\newcommand{\smpl}{P}
\newcommand{\optA}{\bar{A}}

\usepackage[preprint]{aistats2027}
\begin{document}

%

%

\twocolumn[

\aistatstitle{Learning structured linear dynamical systems from missing observations}

\aistatsauthor{ Aravinda K. Ruwanpathirana \And Hemant Tyagi \And  Sunny G.W. Wang}
\vspace{0.25cm}
\aistatsaddress{ Division of Mathematical Sciences, SPMS, NTU Singapore 637371 \\ \texttt{\{aravinda.r, hemant.tyagi, sunny.wanggw\}@ntu.edu.sg} } ]

\begin{abstract}
We consider the problem of learning structured linear dynamical systems over convex sets $\mathcal{K}$, where only a small subset of the observations are available at each time point. 
An estimator which minimizes a bias-corrected, potentially non-convex objective function is proposed. Non-asymptotic bounds are obtained for the statistical error, which depend on the local complexity of $\mathcal{K}$, the trajectory length $T$, and the sub-sampling probability $p$. Convergence of the projected gradient descent algorithm is also established. The general theory is applied to settings where (i) $\mathcal{K}$ is a subspace, (ii) $\mathcal{K}$ is the set of bi-isotonic matrices, and (iii) $\mathcal{K}$ is the set of matrices whose rows are formed by sampling Lipschitz functions. We show meaningful recovery of the transition matrix is possible for values of $T$ much smaller than what is required in the unconstrained case, and for $p = o(1)$.  
\end{abstract}

%
\section{Introduction} \label{sec:intro}
Consider the linear dynamical system\footnote{Also referred to as vector autoregressive (VAR) model of order $1$ in the literature} (LDS)
\begin{align} \label{eq:var_eq}
x_{t+1} = A^* x_t+ \eta_{t+1}; \ \ t = 0,\dots,T-1,
\end{align}
where $(\eta_t)_{t \geq 1}$ are centered i.i.d subgaussian random variables (process noise), $x_0 = 0$ and $A^* \in \matR^{n\times n}$ is the unknown system matrix. The problem of finite-time identification of $A^*$ given access to a single trajectory of states $(x_t)_{t=1}^T$ from \eqref{eq:var_eq} has attracted considerable attention in recent years, mainly in part due to numerous applications arising in, e.g., control systems theory and reinforcement learning, to name a few. A key difficulty in provably recovering $A^*$ is the dependencies between the observations $(x_t)_{t=1}^T$, which makes the theoretical analysis challenging. 

When no structural assumptions are made on $A^*$, a recent like of work \citep{Jedra20,Sarkar19,Simchowitz18a,FaraUnstable18} has analyzed the ordinary least-squares (OLS) estimator and derived non-asymptotic bounds  (up to log factors) of ${O}(1/\sqrt{T})$ which hold with high probability (w.h.p) when $T$ is suitably large w.r.t $n$. The results hold under different stability assumptions on $A^*$ as measured via its spectral radius $\rho(A^*)$ -- this was shown for the strictly stable regime $\rho(A^*) < 1$ by \cite{Jedra20}, and for the more challenging marginal stability regime $\rho(A^*) \leq 1$ by \cite{Sarkar19, simchowitz19a}. The unstable regime $\rho(A^*) > 1$ has been relatively less studied, see e.g., \cite{Sarkar19,FaraUnstable18}.

A number of recent works have  studied the setting where $A^*$ is structured (and strictly stable) under different structural assumptions, typically involving sparsity \citep{Basu15,Kock2015,song2011large,Melnyk16}, low-rankness \citep{Wang23} and low-rank plus sparse structures \citep{lv2021, Basu19}. These works typically show recovery-bounds for $A^*$ with $T$ having a much milder dependence on $n$ than in the unstructured case. 
The recent work of \cite{tyagi2024learning} considered a general setup where $A^*$ lies in a closed convex set $\calK$, in the spirit of shape-constrained regression \citep{Chatter15,Bellec18,neykov19b}. Meaningful recovery of $A^*$ was shown when the local size of $\calK$ at $A^*$ is ``small''.

\paragraph{Our setting.} The aforementioned works assume that $(x_t)_{t=1}^T$ is fully observed, however in practice, we often have \emph{missing values} in the data. This brings us to the setting of this paper: suppose some entries of $(x_t)_{t \geq 1}$ are missing, where the probability that the entry $(x_{t})_i$ is observed is $p$. A popular strategy in this scenario is to impute the missing entries by zero, leading to observations $(y_t)_{t=1}^T$, and then estimating $A^*$ using $(y_t)_{t=1}^T$. Results for provably estimating structured $A^*$ in this setting are relatively scarce in contrast to the full-information setting, and focus essentially on sparsity and low-rank assumptions (e.g., \cite{rao2016system, castro_missing25, JalaliWillett18}).  
\subsection{Our contributions}
In this work, we consider the setting where $A^* \in \calK$ where $\calK$ is a general closed convex set. The estimator we consider amounts to solving a ``bias-corrected'' least squares problem under the constraint $\calK$; see \eqref{eq:unb-est}. The objective function is potentially non-convex, and is motivated from the works of  \cite{loh2011high, JalaliWillett18}. 
\begin{enumerate}
    \item For the solution $\est{A}$ of \eqref{eq:unb-est}, we derive estimation error bounds in the Frobenius norm $\norm{\cdot}_F$ that hold w.h.p, and depend on the local complexity of the tangent cone of $\calK$ at any $B \in \calK$. The local complexity is measured via Talagrand's $\gamma_1$ and $\gamma_2$ functionals \citep{talagrand2014upper}; see Theorem \ref{thm:main-thm}.

    \item We show in Theorem \ref{thm:pgd-thm} that the iterates of the projected gradient descent (PGD) algorithm converge linearly (w.h.p) to $A^*$ -- up to the statistical error obtained in Theorem \ref{thm:main-thm}.

    \item The above results are instantiated for different choices of $\calK$, namely: subspace, set of bi-isotonic matrices (bi-isotonic regression), and set of matrices whose rows are formed by sampling Lipschitz functions (Lipschitz regression). 
\end{enumerate}
Theorems \ref{thm:main-thm} and \ref{thm:pgd-thm} hold provided $T$ and $p \in (0,1]$ are suitably large, and $A^*$ is strictly stable. As shown in the examples, these results do establish meaningful recovery of $A^*$ w.r.t $\norm{\cdot}_F$, when $T = o(n^2)$, and $p = o(1)$ w.r.t $T$. Note that if $p = 1$, and $A^*$ is unstructured, then $T \gtrsim n^2$ is necessary to ensure $\norm{\est{A} - A^*}_F$ is small. 

Our results for the aforementioned examples are new for learning structured LDSs from missing values -- existing works in this setting focus primarily on sparsity and low-rank based structures (as noted earlier). In essence, we extend the results of \cite{tyagi2024learning} to the setting of missing values. 

%
\subsection{Related work} \label{sec:rel_work}
\paragraph{High-dimensional regression with missing observations.} 
In the setting of high-dimensional regression with (typically) independent samples, a number of works have studied the effect of missing data. The seminal work of
\cite{loh2011high} considered sparse linear regression with missing or noisy covariates. They propose a bias-corrected lasso-type estimator with a potentially non-convex objective and show that PGD  converges to the true parameters, up to the statistical error. Our estimator is motivated from their work. While their framework accommodates dependent covariates generated via a \emph{stationary} LDS, this is different from our setting as rewriting \eqref{eq:var_eq} as a traditional linear model (by vectorizing $A^*$) leads to a model with dependencies between the design matrix and noise (while \cite{loh2011high} assume these to be independent). Addressing missingness in diverse regression settings, HMLasso estimates both the covariance matrix and coefficients for highly incomplete design matrices \citep{takada2019hmlasso}.  
Other approaches tackle missing responses via doubly robust high-dimensional M-estimation \citep{chakrabortty2019high}, and missing covariates in semi-supervised settings using Dantzig selectors \citep{risebrow2026semi}. Beyond linear models, sparse estimation has been adapted for nonlinear additive models via multiple imputation \citep{Shimazu2025}, and for simultaneous missingness in both predictors and responses using joint lasso penalties \citep{Muroya2026}. 

\paragraph{Estimating structured LDSs with missing observations.} A growing body of work has recently addressed the problem of estimating LDSs provably, under our Bernoulli sampling model (with zero-imputation). \cite{rao2016system, milind2017estimation} explore covariance estimation to recover $A^*$ (assuming stationarity) under structural assumptions such as bandedness, sparsity, and low-rankness. 
\cite{JalaliWillett18, wong2017lasso} propose a bias-corrected  (potentially non-convex) lasso-type estimator for estimating sparse $A^*$, in the spirit of \cite{loh2011high}. They provide error rates for the global minimizer $\est{A}$ w.r.t $\norm{\cdot}_F$, and also address recovering the support of $A^*$. But they do not show guarantees for solving for $\est{A}$. \cite{castro_missing25} study the Dantzig selector assuming $A^*$ is sparse, and provide minimax optimal recovery bounds w.r.t the entry-wise $\max$ norm. All of these works assume strict stability of $A^*$.

\paragraph{Other related works.} \cite{pmlr-v89-mark19a} consider autoregressive parameter estimation from missing data for discrete event streams; these concepts have also been applied to network topology inference from partial observations \citep{Ioannidis2019, ravazzi2021Ergodic, Cirillo2021Learning, zaman2023onlinejointtopologyidentification}. 
A related but distinct problem is system identification which seeks to estimate both $A^*$, and the \emph{unknown} observation matrix (which is \emph{unchanged} over time) mapping states to outputs (e.g., \citep{tsiamis2019finite, sun2025finitesampleanalysispoles, lee2020robust}). 
\cite{zhang2026provably} propose an active sensor allocation strategy and provide explicit non-asymptotic error bounds for system-identification from multiple trajectories. 
\cite{Rao2017_compmeas} studied a setting where compressive linear measurements of the states are available at each time $t$, and provide fundamental limits for covariance-based estimation.

\section{Problem setup and preliminaries} \label{sec:estimators}

\subsection{Notation} 
For a vector $x \in \matR^n$, let $\norm{x}_p$ denote the standard $\ell_p$ norm. For any matrix $B \in \matR^{n_1 \times n_2}$, $\norm{B}_2$ (resp. $\norm{B}_F$) denotes the spectral (resp. Frobenius) norm of $B$. If $n_1 = n_2$ then $\operatorname{diag}(B)$ keeps only the diagonal entries of $B$. Also, $\dotprod{B}{C} = \Tr\left(B^\top C\right)$ denotes the inner product between $B$ and $C$. The vector $\vect(B) \in \matR^{n_1n_2}$ is obtained by stacking the columns of $B$. The symbols $\otimes$ and $\odot$ denote the Kronecker and Hadamard products, respectively. For $n \times n$ matrices, we denote the unit Frobenius sphere by $\frobsphere$, the unit Frobenius ball by $\frobball$, and the identity matrix by $I_n$. 

Sets are denoted by calligraphic letters and for $\calS \subset \matR^{n_1 \times n_2}$ and $X \in \matR^{n_1 \times n_2}$, we denote $\calS-X := \set{A-X: A\in \calS}$. For a closed convex set $\calC \subset \matR^{m \times n}$, and $X \in \matR^{m \times n}$, we denote 
\begin{equation*}
    \Pi_{\calC}(X) := \arg\min_{Y \in \calC} \norm{Y - X}^2_F.
\end{equation*}
to be the (orthogonal) projection of $X$ onto $\calC$. Note that $\Pi_{\calC}(X)$ here is unique.
Throughout, the values of symbols used to denote constants (e.g., $C,c,c_1,$ etc.) may differ from line to line. For $a,b > 0$, we write $a \lesssim b$ if there exists a constant $C$ such that $a \le Cb$. If $a \lesssim b$ and $a \gtrsim b$, we write $a \asymp b$. 
The subgaussian norm of a random variable $X$ is defined as $\norm{X}_{\psi_2} \coloneq \sup_{l\ge 1} l^{-1/2} \left(\expec \absenc{X}^l\right)^{1/l},$ and we say that $X$ is $L$-subgaussian if $\norm{X}_{\psi_2} \le L$. 

\subsection{Problem setup and estimator}\label{sec:prob-est}
We consider $(x_t)_{t=1}^T$ to be generated via the VAR process in \eqref{eq:var_eq} where $A^* \in \calK$ for a known closed convex set $\calK$. We assume that $(\eta_t)_{t=1}^T$ are centered, i.i.d random variables where the entries of $\eta_t$ are i.i.d, $L$-subgaussian (for some constant $L$), and have unit variance. At each time $t=1,\dots,T$ we observe 
\begin{align*}
    y_t = \smpl_t x_t \text{ where } {(\smpl_t)}_{ii} \stackrel{\text{i.i.d}}{\sim} B(p) \text{ for } i=1,\dots,d.
\end{align*}
Here $(\smpl_t)_{t \geq 1}$ are diagonal matrices 
and $B(p)$ is the Bernoulli distribution with parameter $p \in (0,1]$. The goal is to learn $A^*$ given $(y_t, \smpl_t)_{t=1}^{T}$ and the set $\calK$.

Denote the stacked matrices
\[
X = \matrixenc{x_1 & x_2 & \dots & x_{T-1}} \; \text{and} \; \tilde{X} = \matrixenc{x_2 & x_3 & \dots & x_{T}},
\]
and let $\tilde{E} = \matrixenc{\eta_2 & \eta_3 & \dots & \eta_{T}}$. \eqref{eq:var_eq} can then be written as
$\tilde{X} = A^* X + \tilde{E}.$
Now denoting 
\begin{align*}
&Z = \matrixenc{\smpl_1 y_1 & \smpl_2 y_2 & \dots & \smpl_{T-1} y_{T-1}},\\ 
&\tilde{Z} = \matrixenc{\smpl_2 y_2 & \smpl_3 y_3 & \dots & \smpl_{T} y_{T}},
\end{align*}
and $D = \operatorname{diag}^{1/2}(ZZ^\top)$, we estimate $A^*$ via
%
\begin{align}
\widehat{A} &\in \argmin{A \in \calK} \ \underbrace{\norm{\tilde{Z}-AZ}_F^2-(1-p)\norm{A D}_F^2}_{=: g(A)}.\label{eq:unb-est} 
\end{align}
\begin{remark}
We assume knowledge of $p$ in \eqref{eq:unb-est} to establish the theoretical guarantees, akin to \cite{JalaliWillett18,loh2011high}. In practice, when $p$ is unknown we can use a hyperparameter tuning approach by varying $p$ and testing the recovered matrix on a downstream prediction task. 
\end{remark}

The objective function in \eqref{eq:unb-est} is derived in a similar manner as in \cite{JalaliWillett18}, except that we now have a general convex constraint $\calK$.
\begin{lemma}\label{lem:unb-est}
If $p > 0$ then
\begin{align*}
\argmin{A \in \calK} \ \expec \left[ g(A) \right] = \argmin{A \in \calK} \ \expec \left[\norm{\tilde{X}-AX}_F^2 \right].
\end{align*}
\end{lemma}
The proof is detailed in Appendix \ref{sec:unb-est}  and follows similar arguments outlined in \cite{JalaliWillett18}, which in turn were motivated by \cite{Loh12}. Thus minimizing the population loss in \eqref{eq:unb-est} is equivalent to minimizing the population loss in the full-information setting.

\paragraph{Projected gradient descent.} The objective $g(\cdot)$ can be nonconvex, hence we consider solving it via projected gradient descent (PGD) as in Algorithm \ref{algo:PGD}.

\begin{algorithm}
\caption{PGD for structured transition matrix}\label{algo:PGD}
    \begin{algorithmic}[1]
        \Require Initialization $A^{(0)}$, step-size $\gamma > 0$, number of iterates $K$
        \For{$k = 0$ \textbf{to} $K-1$}
        \State $A^{(k+1)} = \Pi_{\calK}\left(A^{(k)} - \gamma \nabla g(A^{(k)}) \right)$
        \EndFor
        \State \Return $A^{(K)}$
    \end{algorithmic}
\end{algorithm}

%
%
\subsection{Preliminaries}\label{sec:prelim-def}
\paragraph{Talagrand's $\gammafn{\alpha}$ functionals.} We start with restating the Talagrand's $\gammafn{\alpha}$ functionals \citep{talagrand2014upper}. 

\begin{definition}[Talagrand's $\gammafn{\alpha}$ Functionals \citep{talagrand2014upper}]\label{def:talgrand-gamma}
Let $(\calS,d)$ be a metric space. A sequence of subsets of $\calS$, $(\calS_r)_{r \ge 0}$, is an admissible sequence if $\absenc{\calS_0} = 1$ and $\absenc{\calS_r} \le 2^{2^r}$ for all $r \ge 1$. Then for any $0 < \alpha < \infty$, the $\gammafn{\alpha}$ function of $(\calS,d)$ is defined as
\begin{equation*}
    \gammafn{\alpha} (\calS,d) \coloneq \inf \sup_{s\in \calS} \sum_{r=0}^{\infty} 2^{r/\alpha} d(s,\calS_r),
\end{equation*}
where the $\inf$ is over the admissible sequences of $\calS$.
\end{definition}
\begin{definition}[Gaussian width] \label{def:tal_func}
Given a set $\calS \subset \matR^{n \times m}$, the Gaussian width $w(\calS)$ of $\calS$ is given by
\begin{equation*}
    w(\calS) \coloneq \expec \sup_{s\in \calS} \dotprod{G}{s},
\end{equation*}
where $G$ has i.i.d. standard Gaussian entries.
\end{definition}
By Talagrand's majorizing measure theorem \cite[Theorem 2.4.1]{talagrand2014upper}, it is well known that 
\begin{equation*}
 w(\calS) \asymp \gammafn{2}(\calS,\norm{\cdot}_F).
\end{equation*}
In general, $\gamma_{\alpha}(\calS, d)$ can be bounded using metric-entropy estimates of $(\calS, d)$; see Appendix \ref{sec:prelim-tools}.

\paragraph{Tangent cone.} Our  results will depend on the \emph{tangent cone} of the set $\calK$ at $A^*$ (or an appropriate $B \in \calK$), see e.g.,  \cite[Prop. 4.2.1]{Aubin2009}.
\begin{definition}[Tangent cone]
For a convex set $\calK$, the tangent cone of $\calK$ at $A \in \calK$ is defined as the set
\begin{equation*}
    \tancone{A} := \operatorname{cl}\{t(B-A):t\ge 0, B\in \calK\},
\end{equation*}
where $\operatorname{cl}(\cdot)$ denotes the closure operation.
\end{definition}

\paragraph{Stability of $A^*$.} We will require that  $A^*$ is strictly stable, i.e., its spectral radius satisfies $\rho(A^*) < 1$. Our analysis makes use of the quantity (introduced in \cite{Jedra20}) 
\begin{equation*}
J(A^*) \coloneq \sum_{i=0}^\infty \norm{(A^*)^i}_2.
\end{equation*}
If $\rho(A^*) < 1$, then $J(A^*)$ is bounded, although it may grow with the dimension $n$. Furthermore, if there exist $C>0$ and $\kappa \in (0,1)$ such that for all integers $t \geq 1$, $\norm{(A^*)^t}_2 \le C\kappa^{t},$ then $J(A^*) \le 1+\frac{C\kappa}{1-\kappa}$. This is also referred to as $(C,\kappa)$-stability in the literature (see e.g., \cite[Def. 7.2]{ziemannMix22}). Our results will be meaningful when $J(A^*)$ is ``small'', e.g., bounded by a constant, or grows mildly with dimension $n$.

%
\subsection{Main result}

Our first result is Theorem~\ref{thm:main-thm} which bounds $\norm{\widehat{A}-A^*}_F$ for any $A^* \in \calK$. The
proof is provided in Section~\ref{sec:tangent}. Let us denote $\tanconeB \cap \frobsphere$ by $\conesetB$, and $\tanconeAstar \cap \frobsphere$ by $\conesetAstar$.
The subgaussian-norm $L$ of $(\eta_{t,i})_{t,i}$ will be treated as an absolute constant in what follows.
\begin{theorem}\label{thm:main-thm}
There exist constants $c_1, c_2  > 1$ such that for any $\delta \in (0,1)$ and $B \in \calK$, if
\begin{align*}
T \ge \frac{c_1 J^4(A^*) \max \set{ \gammafn{2}^2(\conesetB,\norm{\cdot}_F),\log^2 \left(1/\delta\right)}}{p^4},
\end{align*}
then it holds with probability at least $1-\delta$,
\begin{align*}
&\norm{\widehat{A}-A^*}_F \le 
 c_2 \bigg[\frac{J^3(A^*)}{p^2} \bigg(\frac{\gammafn{1}(\conesetB,\norm{\cdot}_2)}{T}+\frac{\gammafn{2}(\conesetB,\norm{\cdot}_F)}{\sqrt{T}}\\
 &+\sqrt{\frac{\log \left(1/\delta\right)}{T}}\bigg)
+ \frac{J^2(A^*)}{p}\norm{A^*-B}_F \bigg]. 
\end{align*}

\end{theorem}
In Theorem~\ref{thm:main-thm}, we use a generic matrix $B \in \calK$. One potential choice is $B=A^*$ which removes the $\norm{A^*-B}_F$ term in  Theorem~\ref{thm:main-thm}. However, the terms $\gammafn{\alpha} \left(\conesetB, d\right)$ might be easier to bound for some $B \neq A^*$ where $\norm{A^*-B}_F$ is sufficiently small. This will be shown in Corollary \ref{cor:stat-lipschitz-regression}.  

\begin{remark}
When $p=1$, Theorem \ref{thm:main-thm} reduces to \cite[Theorem 1]{tyagi2024learning}, up to a $J(A^*)$ factor, and under slightly stronger conditions on $T$.
\end{remark}

Our second result shows that with PGD, we can find a solution that is close to $A^*$ up to a small additive error. The proof of Theorem~\ref{thm:pgd-thm} is provided in Section~\ref{app-sec:pgd-proof}. 
\begin{theorem}\label{thm:pgd-thm}
Suppose that the learning rate is $\gamma = \frac{1}{2p^2 T J^2(A^*) }.$ There exist constants $c, C, C_1, C_2 > 1$ such that for any $B \in \calK$ and $\delta \in (0,1)$, if
    \begin{align}
        T &\geq \frac{CJ^4(A^*)}{p^4}\left\{\gamma_2^2(\calC_{B}, \norm{.}_F) + \log^2(1/\delta) \right\} \nonumber 
        \\
        &+ \frac{C J^2(A^*)}{p^2}\left\{\gamma_1(\calC_{B}, \norm{.}_2) + \log(1/\delta)\right\}, \label{eq:T-PGD-cond}
    \end{align}
then there exists { $\rho \leq 1-\frac{1}{c J^2(A^*)}$} such that with probability at least $1 - \delta$, it holds for all $k = 0,1,\dots$ that
\begin{equation}\label{eq:pgd-thm-final-form}
    \norm{A^{(k)} - A^*}_F \leq \rho^k \norm{A^{(0)} - A^*}_F  
    + \frac{\varepsilon}{1 - \rho}.
\end{equation}
Here $\varepsilon$ is given by
\begin{align*}
    \varepsilon &= \varepsilon_T =
     \frac{C_1 J(A^*)}{p^2}\Bigg[\frac{\gamma_1(\calC_{B}, \norm{.}_2)}{T} + \frac{\gamma_2(\calC_{B}, \norm{.}_F)}{\sqrt{T}} \\
     &+ \sqrt{\frac{\log(1/\delta)}{T}} \Bigg] + \frac{C_2\norm{A^* - B}_F}{p}.
\end{align*}
Since $(1-\rho)^{-1} = O(J^2(A^*))$, $\varepsilon/(1-\rho)$ is of the same order as the statistical error in Theorem~\ref{thm:main-thm}.
\end{theorem}

\begin{remark}
$\gamma$ in Theorem~\ref{thm:pgd-thm} depends on $J(A^*)$ which is an unknown parameter defined on $A^*$. In practice, one can replace $J(A^*)$ by a known upper bound $\xi$, with no change in the PGD convergence rates.
\end{remark}

\begin{remark} \label{rem:utility_thms}
 Theorems \ref{thm:main-thm} and \ref{thm:pgd-thm} are complementary. Theorem \ref{thm:main-thm} gives statistical rates for the global minimizer of \eqref{eq:unb-est}, while Theorem \ref{thm:pgd-thm} shows that the iterates of PGD cluster in a small ball around $A^*$. Using triangle inequality, clearly $\norm{\widehat{A} - A^{(k)}}_F \lesssim \varepsilon$ under \eqref{eq:T-PGD-cond}.   
\end{remark}

\begin{remark}\label{rem:p_scaling}
    In Theorem~\ref{thm:main-thm}, for a fixed level $\delta$, we have $\norm{\widehat A - A^*}_F = o(1)$ even as $p = o(1)$, under the sufficient condition
    \begin{align}\label{eq:p-floor-thm1}
&J^{3/2} (A^*)\Bigg(\sqrt{\frac{\gamma_1(\calC_B, \norm{.}_2)}{T}} + \frac{\sqrt{\gamma_2(\calC_B, \norm{.}_F)}}{T^{1/4}}\nonumber \\ &+ \sqrt{\norm{A^* - B}_F} \Bigg)  = o(p), \quad \text{with} \quad p = o(1).
    \end{align}
For Theorem \ref{thm:pgd-thm}, if $\norm{A^{(0)} - A^*}_F = O(1)$, $p = o(1)$, and \eqref{eq:p-floor-thm1} holds, then for $k = k_T \gtrsim \log\left(1 / (\norm{A^* - B}_F + \varepsilon_T)\right)$ iterations, we have 
\begin{equation*}
    \norm{A^{(k_T)} - A^*}_F = O(\norm{A^* - B}_F + \varepsilon_T) = o(1).
\end{equation*}
\end{remark}

\begin{remark}\label{rem:sparse_case}
While Theorem~\ref{thm:main-thm} holds for general convex sets $\calK$, the bounds can be sub-optimal in some settings, such as $\calK = \set{A \in \matR^{n \times n}: \norm{\vect(A)}_{1} \le \norm{\vect(A^*)}_{1}}$, where $A^*$ is $k$-sparse, as bounding $\gammafn{1}$ is challenging (as noted in Remark 3 of \cite{tyagi2024learning}). Achieving optimal guarantees in this setting requires specialized analysis exploiting the properties of $A^*$ and the $\ell_1$-ball.
\end{remark}

\begin{remark}\label{rem:pgd_proof_comp}
    The proof of Theorem~\ref{thm:pgd-thm} follows a similar approach to
    \cite[Theorem 1]{lv2021}. While their result is stated for  specific convex sets with a convex objective, the techniques used directly apply to general convex sets. We extend their ideas to establish convergence guarantees for PGD for the (potentially non-convex) program \eqref{eq:unb-est}.
\end{remark}

\section{Proofs of main results}
\subsection{Proof sketch of Theorem~\ref{thm:main-thm}}
The proof essentially involves the following steps.
\begin{enumerate}
\item The starting point is the inequality $g(\widehat{A}) \le g(B)$ which holds since $B \in \calK$ is feasible. This implies 
\begin{align*}
&\underbrace{\norm{(B-\widehat{A})Z}_F^2-(1-p)\norm{(\widehat{A}-B)D}_F^2}_{(I)}\\
&\left.
\begin{aligned}
&\le 2\dotprod{\tilde{Z}-A^*Z}{(\widehat{A}-B)Z} \\
&\quad +2(1-p)\dotprod{A^* D}{(\widehat{A}-B) D}\\
&\quad +2\norm{\left(A^*-B\right)Z}_F\norm{(\widehat{A}-B)Z}_F\\
&\quad +2(1-p)\norm{(A^*-B)D}_F \norm{(\widehat{A}-B)D}_F.
\end{aligned}
\right\}(II)
\end{align*}
See Lemma~\ref{lem:B-bound} and its proof. 

\item In Lemma~\ref{lem:LHS-bound} we lower bound $(I)$ as follows. 
\begin{itemize}
\item Rewrite the individual terms in $(I)$ as p.s.d  quadratic forms $\eta^\top M \eta$ for appropriate matrices $M$, where the entries of $\eta$ are formed by stacking $\eta_1,\dots,\eta_T$.

\item Conditioning on $(\smpl_t)_{t=0}^T$ and applying Theorem \ref{thm:psd-sup} \cite[Theorem 3.1]{krahmer14} on appropriately defined sets of matrices, the infimum and supremum of $\eta^\top M \eta$ over each set is bounded with high probability. The bounds incorporate Talagrand's $\gammafn{2}(\cdot,\cdot)$ \citep{talagrand2014upper} for appropriate sets of matrices. Combining individual terms, we get a lower bound for $(I)$ depending on $(\smpl_t)_{t=0}^T$. 

\item Next, we convert terms depending on $(\smpl_t)_{t=0}^T$ to the form $q^\top M q$ (for appropriate matrices $M$) plus a linear term of the form $q^\top h$ (for an appropriate vector $h$), where $q$ is a random vector depending on $(\smpl_t)_{t=0}^T$. We then reapply Theorem \ref{thm:psd-sup} to uniformly control the quadratic terms over a set of matrices, and apply Theorem \ref{thm:bounded-func-conc} \cite[Chapter 3]{Wainwright2019} to uniformly control  the linear term. 

\item Combining each of these events where the desired bounds hold, and choosing $T$ to be sufficiently large, we guarantee that $(I)$ is strictly lower-bounded by $\frac{p^2 T}{2}\norm{\widehat{A}-B}_F^2$. 
\end{itemize}

\item We use Lemma~\ref{lem:RHS-bound} to upper bound  $(II)$ via a similar strategy to the proof of Lemma~\ref{lem:LHS-bound}.
\begin{itemize}
\item We first condition on  $(\smpl_t)_{t=0}^T$ and convert the terms in $(II)$ to terms of the form $\eta^\top M \eta$ for appropriately defined matrices $M$. 

\item For $(II)$, the matrices $M$ appearing in the quadratics are no longer p.s.d. To handle this, we apply Theorem \ref{thm:non-psd-sup} \citep{dirksen15} on an appropriate set of matrices to uniformly upper bound the quadratic terms over this set. Since the matrices are not p.s.d, the bound depends on both $\gammafn{1}(\cdot,\cdot)$ and $\gammafn{2}(\cdot,\cdot)$.

\item We next consider the terms dependent on $(\smpl_t)_{t=0}^T$ and write them in the form $q^\top M q$ and $q^\top h$. To obtain a uniform bound on these terms (over appropriate sets) we then use Theorem \ref{thm:non-psd-sup} on the quadratic terms and Theorem \ref{thm:bounded-func-conc} \cite[Chapter 3]{Wainwright2019} on the linear terms. 
\end{itemize}
\end{enumerate}
Combining these two high-probability bounds and rearranging the terms yields the desired final bound. See Appendix~\ref{sec:tangent} for details.  

\subsection{Proof sketch of Theorem~\ref{thm:pgd-thm}}
Let $H_B^{(k)} = (A^{(k)}-B)/\norm{A^{(k)} - B}_F$, $H_B^* = (A^* - B)/\norm{A^* - B}_F$, and recall that $D = \operatorname{diag}^{1/2}(ZZ^\top)$, $M = ZZ^\top - (1-p)D$.
After substituting the expression for the gradient, we obtain by Lemma \ref{lem:PGD-error-decomp}, 
\begin{align*}
    &\norm{A^{(k+1)} -B}_F  \\
    &\leq \norm{A^{(k)} - B}_F \sup_{V \in \calC_B } \dotprod{V}{H_B^{(k)}(I - 2\gamma \mathbb{E}[M]} 
    \\&+ 2\gamma\norm{A^{(k)} - B}_F \sup_{V \in \calC_B } \dotprod{V}{H_B^{(k)} (\mathbb{E}[M] - M)}  
    \\&+ 2\gamma \sup_{V \in \calC_B} \bigg\{ \norm{VZ}_F \norm{(A^* - B) Z}_F \\
    &+ (1-p)\norm{VD}_F \norm{(A^* - B)D}_F \\
    &+ \dotprod{VZ}{\widetilde Z - A^* Z} + (1-p)\dotprod{VD}{A^* D} \bigg\}.
\end{align*}
\begin{enumerate}
    \item The first term is bounded by Lemma \ref{lem:term-1-PGD-decomp}, using the sub-multiplicativity of norms and bounding the minimum and maximum eigenvalues of $\mathbb{E}[M]$.

    \item  The second term is bounded by Lemma \ref{lem:PGD-centered-term}. The idea is to rewrite it as the supremum of quadratic forms (involving independent sub-gaussian random variables) over a suitable set of matrices. This supremum is bounded using the concentration inequality in Theorem \ref{thm:non-psd-sup} \citep{dirksen15}.

    \item  The third term is controlled by Lemma \ref{lem:PGD-third-term}, which follows from an appropriate identification of terms in \eqref{eq:dotprod-PGD-unif-bound}. Equation \eqref{eq:dotprod-PGD-unif-bound} is obtained again from controlling the supremum of random quadratic forms using Theorem \ref{thm:non-psd-sup} \citep{dirksen15}.
\end{enumerate}

Grouping terms together and ensuring that $\rho < 1$  (through the sufficient condition \eqref{eq:T-PGD-cond}), we obtain the desired statement by unrolling the recurrence relation before applying the triangle inequality twice. 

%
%
\section{Corollaries of main results} \label{sec:cor_main_res}

We now present corollaries of Theorems \ref{thm:main-thm} and \ref{thm:pgd-thm} for the settings where $\calK$ is a subspace, or the set of bi-isotonic matrices. The setting of Lipschitz regression is treated in Appendix \ref{app-sec:lip_reg} due to space constraints.

%
%

\subsection{\texorpdfstring{$d$}{d}-dimensional subspace}\label{sec:subspace}
Let $\calK \subset \matR^{n \times n}$ be a $d$-dimensional subspace, hence for any $B \in \calK$ we have $\tanconeB = \calK$. We obtain the following Corollary of Theorem \ref{thm:main-thm}.
\begin{corollary}\label{cor:sub-stat}
There exist constants $c_1,c_2>1$ such that for any $\delta \in (0,1)$, if $T \geq c_1\frac{J^4(A^*)\max\set{d,\log^2 \left(1/\delta\right)}}{p^4}$ it holds with probability at least $1-\delta$,
\[
\norm{\widehat{A}-A^*}_F \leq c_2 \frac{J^3(A^*)}{p^2}  \bigg(\frac{\sqrt{d}+\sqrt{\log \left(1/\delta\right)}}{\sqrt{T}}\bigg).
\]
\end{corollary}
\begin{proof}[Proof of Corollary \ref{cor:sub-stat}]
Apply Theorem~\ref{thm:main-thm} with $B=A^*$, and the bounds $\gammafn{1}(\calK \cap \frobsphere, \norm{\cdot}_2) \lesssim d$ and $\gammafn{2}(\calK \cap \frobsphere, \norm{\cdot}_F) \lesssim \sqrt{d}$; see e.g., \cite{tyagi2024learning}.
\end{proof}

We next obtain the following Corollary of Theorem \ref{thm:pgd-thm}.
\begin{corollary}\label{cor:sub-pgd}
Suppose that $\gamma = \frac{1}{2p^2 T J^2(A^*) }$. Then there exist constants $c, C_1,C_2>1$ such that for any $\delta \in (0,1)$, if
%
 $       T \geq C_1\frac{J^4(A^*)\left(d + \log^2(1/\delta)\right)}{p^4},$
%
then with probability at least $1-\delta$, it holds for all $k=0,1,\dots$,
\begin{equation*}
    \norm{A^{(k)}-A^*}_F \le \rho^k \norm{A^{(0)}-A^*}_F+\frac{\varepsilon}{1-\rho_{}},
\end{equation*}
where $\rho \le 1-\frac{1}{cJ^2(A^*)}$ and

\begin{equation*}
\varepsilon = C_2 \frac{J(A^*)}{p^2}\left[
\frac{\sqrt{d}+\sqrt{\log(1/\delta)}}{\sqrt{T}} \right].
\end{equation*}

\end{corollary}

\begin{proof}[Proof of Corollary~\ref{cor:sub-pgd}]
Use Theorem~\ref{thm:pgd-thm} with $B=A^*$, and  bounds on $\gamma_1, \gamma_2$ from the proof of Corollary \ref{cor:sub-stat}.
\end{proof}

\begin{remark}
   For $p = 1$, Corollary \ref{cor:sub-stat} recovers the statement of  \cite[Corollary 1]{tyagi2024learning} up to constants and a $J^2(A^*)$ factor, and matches the error bound in \cite[Theorem 3(iii)]{zheng2021finite} up to an additional $\log(1/\delta)$ term.  
\end{remark}

\begin{remark}
    In Corollary \ref{cor:sub-stat}, for a fixed level $\delta$, we can have $\norm{\hat{A}-A^*}_F = o(1)$ with $p = o(1)$, as long as 
\begin{equation}\label{eq:p-floor-subspace}
    J^{3/2}(A^*)d^{1/4}T^{-1/4} = o(p).
\end{equation}
Moreover, if $\norm{A^{(0)} - A^*}_F = O(1)$, $p = o(1)$, and \eqref{eq:p-floor-subspace} holds, then with $k = k_T \gtrsim \log(1/\varepsilon_T)$ iterations,
\begin{equation*}
    \norm{A^{(k_T)} - A^*}_F = O(\varepsilon_T) = o(1).
\end{equation*}

\end{remark}

\begin{remark}
    When $\calK$ is spanned by $d$ orthonormal matrices, as considered in the experiments, the computational and space complexities of the projection step are $O(n^2 d)$ (due to computation of $d$ inner products).
\end{remark}

\subsection{Bi-isotonic regression}
Let $\calK$ be the class of bi-isotonic matrices
\begin{equation*}
        \calK := \left\{A \in \mathbb{R}^{n \times n}: A_{ij} \leq A_{i+1,j}, A_{ij} \leq A_{i, j+1} \right\}.
\end{equation*}
\begin{corollary}\label{cor:isotone-stat}
There exist constants $c_1,c_2 > 1$ such that for any $\delta \in (0,1)$, if 
\begin{equation*}
    T \geq c_1 \frac{J^4(A^*)\max\left\{k(A^*)\log^8(n), \log^2(1/\delta) \right\}}{p^4},
\end{equation*}
where $k(A^*)$ is the minimum number of rectangular blocks for which $A^*$ is constant, then with probability at least $1 - \delta$, it holds that 
\begin{align*}
    \norm{\widehat A - A^*}_F \leq c_2 &\frac{J^3(A^*)}{p^2}\Bigg(\frac{n^{3/4}\sqrt{k(A^*)}\log^4(n)}{T}\\ &+\frac{\sqrt{k(A^*)}\log^4(n)+\sqrt{\log(1/\delta)}}{\sqrt{T}}\Bigg).
\end{align*}
\end{corollary}
\begin{proof}[Proof of Corollary \ref{cor:isotone-stat}]
Substitute the bound on $\gamma_1\left(\calC_{A^*}, \norm{.}_2\right)$ from Lemma \ref{lem:isotonic-gamma1-bound}, and on $\gamma_2\left(\calC_{A^*}, \norm{.}_F \right)$ from Lemma \ref{lem:gamma2-isotone-bound}, into Theorem \ref{thm:main-thm} with $B = A^*$. 

To bound $\gamma_2\left(\calC_{A^*}, \norm{.}_F \right)$, we first relate it to the statistical dimension of $\calT_{\calK, A^*}$ through the Gaussian width, by  \cite[Prop. 10.2]{tropp2014} and Talagrand's majorizing measures theorem. Then, the statistical dimension is bounded using equation (4.4) of \cite{chatterjee2018}.
To bound $\gamma_1\left(\calC_A^*, \norm{.}_2\right)$, we use the usual metric entropy bounds, splitting at a threshold $r$ representing the fine and coarse scales. The statistical dimension bound is used for the coarse scale. At the fine scale, we rely on the sharp entropy number estimate in \cite{hinrichs2017} to obtain a tighter metric entropy bound.
\end{proof}

The term $k(A^*)$ captures the structure of the bi-isotonic cone, and in general, has the range $1 \leq k(A^*) \leq n^2$. In the worst case $k(A^*) = n^2$, we recover the same rate as in Corollary \ref{cor:sub-stat} with $d = n^2$, which is expected since $A^*$ has no additional structure to exploit. The interesting regime is $k(A^*) = o(n^2)$. For example, when $k(A^*) = n\log(n)$, we obtain 
\begin{align*}
    \norm{A^* - A}_F &\lesssim \frac{J^3(A^*)}{p^2}\bigg\{\frac{\sqrt{\log(1/\delta)} + \sqrt{n}\log^{9/2}(n)}{\sqrt{T}} \\
    &+ \frac{n^{5/4}\log^{9/2}(n)}{T}  \bigg\},
\end{align*}
as soon as 
\begin{align*}
    T &\gtrsim \frac{J^4(A^*)}{p^4}\left\{n\log^9(n) + \log^2(1/\delta) \right\} \\
    &+ \frac{J^2(A^*)}{p^2}\left\{n^{5/4}\log^{9/4}(n) + \log(1/\delta) \right\}.
\end{align*}
An interesting phase transition  occurs above at $T \asymp n^{9/16}$. For $T \ll n^{9/16}$, we have $\gamma_2(\calC_{A^*}, \norm{.}_F) \ll \gamma_1(\calC_{A^*}, \norm{.}_2)$, and the converse holds when $T \gg n^{9/16}$.

\begin{corollary}\label{cor:isotone-pgd}
Suppose $\gamma = \frac{1}{2p^2 T J^2(A^*) }$.
There exist constants $c, C_1, C_2 > 1$ such that for any $\delta \in (0,1)$, if
    \begin{align}
        T &\geq C_1\frac{J^4(A^*)}{p^4}\left\{k(A^*)\log^8(n) + \log^2(1/\delta) \right\} \nonumber
        \\&+ C_1 \frac{J^2(A^*)}{p^2}\left\{n^{3/4}\sqrt{k(A^*)}\log^4(n) + \log(1/\delta)\right\},
    \end{align}
then with probability at least $1 - \delta$, it holds for all $k = 0,1,\dots$ that
\begin{equation*}
    \norm{A^{(k)} - A^*}_F \leq \rho^k \norm{A^{(0)} - A^*}_F + \frac{\varepsilon}{1 - \rho},
\end{equation*}
where $\rho \leq 1-\frac{1}{cJ^2(A^*)}$ and
\begin{align*}
    \varepsilon = C_2\frac{J(A^*)}{p^2}\Bigg(\frac{n^{3/4}\sqrt{k(A^*)}\log^4(n)}{T} \\ + \frac{\sqrt{k(A^*)}\log^4(n)+\sqrt{\log(1/\delta)}}{\sqrt{T}} \Bigg).
\end{align*}

\end{corollary}

\begin{proof}[Proof of Corollary \ref{cor:isotone-pgd}]
Apply bounds on $\gamma_{1}, \gamma_2$ from Lemmas \ref{lem:isotonic-gamma1-bound} and \ref{lem:gamma2-isotone-bound} into Theorem \ref{thm:pgd-thm}, and use $B = A^*$.
\end{proof}

\begin{remark}
For a fixed level $\delta$, the condition for $\norm{\widehat A - A^*}_F = o(1)$ with $p = o(1)$ is
\begin{equation}\label{eq:p-floor-isotone}
J^{3/2}(A^*)k^{1/4}(A^*)\log^2(n)\max\left\{\frac{n^{3/8}}{\sqrt{T}},\frac{1}{T^{1/4}} \right\} = o(p).
\end{equation}
Moreover, if $\norm{A^{(0)} - A^*}_F = O(1)$, $p = o(1)$, and \eqref{eq:p-floor-isotone} holds, then with $k = k_T \gtrsim \log(1/\varepsilon_T)$ iterations,
\begin{equation*}
    \norm{A^{k_T} - A^*}_F = O(\varepsilon_T) = o(1).
\end{equation*}

\end{remark}

\begin{lemma}\label{lem:gamma2-isotone-bound}
    Consider the class of bi-isotonic matrices. There exists an constant $C > 0$ such that
    \begin{equation*}
        \gamma_2(\calC_{A^*}, \norm{.}_F) \leq C\sqrt{k(A^*)}\log^4(n).
    \end{equation*}
\end{lemma}

\begin{lemma}\label{lem:isotonic-gamma1-bound}
    For the class of bi-isotonic matrices, there exists an absolute constant $C > 0$ such that
    \begin{equation*}
         \gamma_1\left(\calC_{A^*}, \norm{.}_2 \right) \leq Cn^{3/4}\sqrt{k(A^*)}\log^4(n).
    \end{equation*}
\end{lemma}

\begin{remark}
    The projection step for bi-isotonic matrices can be performed by Dykstra’s projection algorithm \citep{dykstra1985}, with a computational complexity of $O(n^2 k)$, where $k$ is the number of iterations.
\end{remark}

%
\section{Experiments}
The numerical properties of the PGD algorithm are explored for the subspace and bi-isotonic examples, with results for Lipschitz regression deferred to Appendix \ref{app-sec:lip_reg}. Since we are interested in the regime where the  objective function is non-convex,  
we fix $n = 100$. All experiments were run on a 2025 Macbook Air with a 10 core M4 chip and 16GB of RAM.

\subsection{Subspace}
We consider $p \in \{0.3, 0.5, 0.9\}$, $T \in \{50, 250, 500, 1000\}$, $d \in \{5, 15, 30\}$, resulting in $36$ configurations, each with $500$ Monte Carlo replications. 
To generate linear subspaces for a fixed $d$, we first simulate a matrix $H \in \mathbb{R}^{n^2\times d}$, with i.i.d standard Gaussian  entries $H_{ij}$. A QR decomposition was then performed to obtain an orthonormal matrix $Q$, before reshaping each column into $d$ orthonormal matrices $B_j \in \mathbb{R}^{n\times n}, j \in \{1, \dots, d\}$. The coefficients $\theta_j \in \mathbb{R}^d$ associated to each $B_j$ were simulated from a standard gaussian distribution with unit $\ell_2$ norm. The transition matrices were formed by $A^* = \frac{0.8}{\norm{\sum_{j=1}^d \theta_j B_j}_2} \sum_{j=1}^d \theta_j B_j$. Seeds were set so that $A^*$'s are fixed for each $d$.

For a linear subspace $\calK$, the orthogonal projection onto $\calK$
is available in closed form.
The learning rate was chosen as in Theorem \ref{thm:pgd-thm}, where the $J(A^*)$ term was computed numerically with a threshold $K = 50$. The maximum iteration was set to 500, with the stopping rule given by $A^{(k+1}) - A^{(k)} < \zeta$, for tolerance $\zeta = 10^{-8}$. In the following, we denote the PGD output by $\hat A$. 

To get a sense of the empirical exponents for $T$ and $p$, we perform a log-log regression, of the form $\log(\Delta_{rel}) = C - \alpha\log(T) - \beta\log(p)$, and is plotted in Figure \ref{fig:subspace-exponent}. The left panel plots $T$ against $\Delta_{rel}$, separately for each value of $p$, where the shaded bands represent the 5th to 95th quantile across the Monte carlo runs. The estimated exponent for $T$ is around 0.51, close to the theoretical rate of $T^{-1/2}$.
The right panel plots $\sqrt{T}\Delta_{rel}$ against $p$. Since the values of  $\sqrt{T}\Delta_{rel}$ were very close across the different values of $T$, we summarize them by the pooled median as seen by the blue line. The empirical exponent is around $1.1$, smaller than the theoretical bound of $p^{-2}$.

\begin{figure}
    \centering
    \resizebox{0.9\linewidth}{!}{
    \input{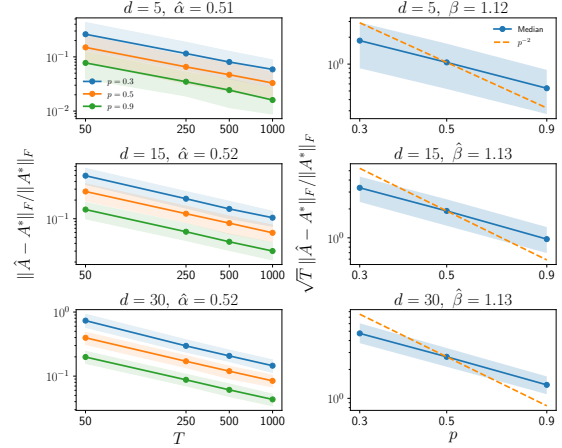}
    }
    \caption{Log-log plots displaying the empirical rates of\\ convergence (when $\calK$ is a subspace).}
    \label{fig:subspace-exponent}
\end{figure}

\subsection{Bi-isotonic regression}
To simulate a blockwise constant bi-isotonic matrix, we construct a disjoint partition $[n] \times [n] = \cup_{a=1}^{k_r}\cup_{b=1}^{k_c} I_a \times J_b$, so $A^*$ is constant on each rectangle $I_a \times J_b$. To obtain this, a sub-matrix $\theta$ with $k_r \times k_c$ distinct entries is generated from a standard gaussian distribution, before imposing monotonicity by solving a constrained optimization problem. Finally, we set  $A^*_{ij} = \theta_{\ell m}$, whenever $i \in I_{\ell}, j \in J_m$, and rescaling to obtain a spectral norm at most $0.8$.

We consider all combinations of $p \in \{0.3, 0.5, 0.9\}$, $T \in \{50, 100, 200\}$, and $(k_r, k_c) \in \{(2, 2), (5, 5) \}$, with 200 replications in each configuration. The projection step was performed using Dykstra's projection algorithm, with the individual row and column projections computed using the \texttt{isotonic\_regression} solver from \texttt{scipy}. The tolerance for the projection step was set to $10^{-3}$ with 100 maximum iterations, whereas the tolerance for the stopping rule of the pgd algorithm was set to $\zeta = 5\times 10^{-3}$, with 300 maximum iterations. Plots displaying the relative error can be seen in Figure \ref{fig:bi-isotonic}. The line plots represent the median errors over the replications, with the shaded bands representing the 25th to the 75th quantile. A narrower range of quantiles was considered for the bands, since for small $T$ and $p$, the outliers were extremely large as expected, since there is insufficient data relative to the complexity of the class for reliable estimation.

\begin{figure}
    \centering
    \resizebox{0.9\linewidth}{!}{
    \input{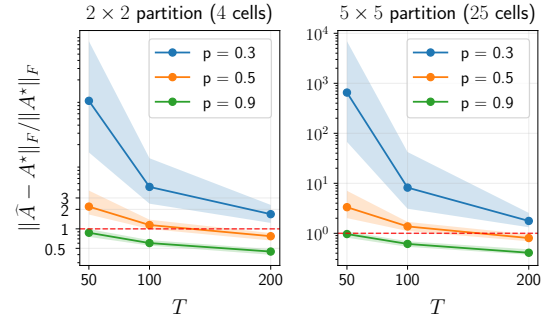}
    }
    \caption{Relative Frobenius errors for bi-isotonic matrices.}
    \label{fig:bi-isotonic}
\end{figure}

\newpage 

\subsection*{AI use statement}

We have not used generative AI tools for any task pertaining to this paper, and the rest of the required disclosure tasks are not applicable to this work.

\subsection*{Acknowledgements}
This work was supported by a Nanyang Associate Professorship (NAP) grant
from NTU Singapore.

\bibliographystyle{apalike}
\bibliography{references}

\clearpage
\appendix
\thispagestyle{empty}
\onecolumn
\makeatletter
\addtocontents{toc}{\string\setcounter{tocdepth}{3}} 
\setcounter{tocdepth}{3}

\let\addcontentsline\oldaddcontentsline 
\makeatother

\begin{center}
    {\large \bf Table of Contents (Appendix)}
\end{center}
\startcontents[sections]
\printcontents[sections]{}{1}{}

\section{Technical tools}\label{sec:prelim-tools}
\paragraph{Common terminology} We use the $x,\eta$, and $\Gamma$ to refer to the following. Let $x = \vect(X)$ and let $\eta = \vect\left(\matrixenc{\eta_1 & \eta_2 & \dots & \eta_{T-1} }\right)$. Let $\Gamma \in \matR^{n(T-1) \times n(T-1)}$ be a matrix defined as,
\[
\Gamma = 
\matrixenc{
I & 0 & 0 & \dots & 0 \\
A^* & I & 0 & \dots & 0 \\
{(A^*)}^2 & A^* & I & \dots & 0 \\
\vdots & \vdots & \vdots & \dots & \vdots \\
{(A^*)}^{T-2} & {(A^*)}^{T-3} & {(A^*)}^{T-4} & \dots & I  \\ 
}.
\]
Note that $x = \Gamma \eta$. Furthermore, it can be shown that $\norm{\Gamma}_2 \le J(A^*)$ \citep{Jedra20}.
\paragraph{Properties of $\gammafn{\alpha}$} Recall the definition of Talagrand's $\gammafn{\alpha}$ functionals (Definition~\ref{def:talgrand-gamma}). They have the following useful properties. 
\begin{enumerate}
\item \label{itm:prop-1} For any two metrics $d_1,d_2$ such that $d_1 \le c d_2$ (where $c>0$), it holds that $\gammafn{\alpha}(\calS,d_1) \le c  \gammafn{\alpha}(\calS,d_2)$. 

\item \label{itm:prop-2} \citep[Theorem 1.3.6]{Talagrand2005-gf} For $\widehat{\calS} \subseteq \calS$, $\gammafn{\alpha}(\widehat{\calS},d) \le c_{\alpha} \gammafn{\alpha}(\calS,d)$ where $c_{\alpha}>0$ is a constant depending only on $\alpha$.

\item \label{itm:prop-3} \citep[Theorem 1.3.6]{Talagrand2005-gf} If the map $f:(\calV,d_1) \rightarrow (\calU,d_2)$ is onto, and $d_2 \left(f(x),f(y)\right) \le M d_1\left(x,y\right)$ for all $x,y \in \calV$, then $\gammafn{\alpha}(\calU,d_2) \le c_{\alpha} M \gammafn{\alpha}(\calV,d_1)$ where $c_{\alpha}>0$ is a constant depending only on $\alpha$.
\end{enumerate}
Properties \ref{itm:prop-2} and \ref{itm:prop-3} were stated in \cite[Theorem 1.3.6]{Talagrand2005-gf} (with $c_{\alpha} = 1$) for an alternative definition of the $\gamma_{\alpha}$ functional \cite[Definition 2.2.19]{talagrand2014upper} which is equivalent to Definition \ref{def:tal_func} up to a constant depending only on $\alpha$; see also \cite[Section 2.3]{talagrand2014upper}. 

%
It is useful to note that $\gamma_\alpha(\calS,d)$ can be upper bounded using the metric-entropy of $(\calS, d)$. Denoting $\calN(\calS, d, \epsilon)$ to be the covering number of $\calS$ w.r.t $d$, for any $\epsilon > 0$, it is well-known that for some constant $c_\alpha > 0$ (depending only on $\alpha$)
\begin{align*}
    \gamma_{\alpha}(\calS,d) \leq c_{\alpha} \int_{0}^{\text{diam}(\calS)} \log^{1/\alpha} \calN(\calS, d, \epsilon) d\epsilon. 
\end{align*}
This can be shown using \cite[Corollary 2.3.2]{talagrand2014upper} and by adapting the arguments appearing after the proof of \cite[Lemma 2.2.11]{talagrand2014upper} to general $\alpha > 0$.

\paragraph{Concentration bounds.} Given a set of matrices $\calM$, let,
\[
d_2(\calM) \coloneq \sup_{M \in \calM} \norm{M}_2 \text{ and } d_F(\calM) \coloneq \sup_{M \in \calM} \norm{M}_F.
\]
The following theorem from \cite{krahmer14} provides a concentration bound for the suprema of second order
subgaussian chaos processes involving positive semidefinite matrices.

\begin{theorem}[Theorem 3.1 of \cite{krahmer14}]\label{thm:psd-sup}
Let $\calM$ be a set of matrices and $\nu$ be a vector with independent, zero-mean, $L$-subgaussian entries with variance 1. Let,
\begin{align*}
F &\coloneq \gammafn{2} \left(\calM,\norm{\cdot}_2\right) \left[\gammafn{2} \left(\calM,\norm{\cdot}_2\right)+d_F(\calM)\right]+d_F(\calM)d_2(\calM),\\
V &\coloneq d_2(\calM)\left[\gammafn{2} \left(\calM,\norm{\cdot}_2\right)+d_F(\calM)\right], \text{ and, } U \coloneq d^2_2(\calM)
\end{align*}
There exist constants $c_1,c_2>0$ depending only on $L$ such that for any $t > 0$,
\begin{align*}
\prob \left(\sup_{M \in \calM} \absenc{\norm{M\nu}_2^2-\expec \left[\norm{M\nu}_2^2\right]} \ge c_1 F+t\right) 
\le 2\exp \left(-c_2\min\set{\frac{t^2}{V},\frac{t}{U}}\right).
\end{align*}
\end{theorem}

We will also use the following theorem derived from a set of results of \cite{dirksen15} for bounding the suprema of general second order subgaussian chaos processes, where the matrices are not necessarily positive semidefinite. We refer the reader to Theorem 4 of \cite{tyagi2024learning} for the theorem in its form stated here.
\begin{theorem}[\cite{dirksen15}]\label{thm:non-psd-sup}
Let $\calM$ be a set of matrices and $\nu$ be a vector with independent, zero-mean, $1$-subgaussian entries. For $M\in \calM$, let $C_M(\nu) = \nu^\top M \nu - \expec \left[\nu^\top M \nu\right]$. Then,
\begin{enumerate}
\item There exists a universal constant $c>0$ such that for any $l \ge 1$,
\begin{align*}
\left(\expec \sup_{M \in \calM} \absenc{C_M(\nu)}^l\right)^{1/l} \le
 c \bigg(\gammafn{1} \left(\calM,\norm{\cdot}_2\right)+\gammafn{2} \left(\calM,\norm{\cdot}_F\right)+\sqrt{l}d_F(\calM)+ld_2(\calM)\bigg).
\end{align*}
\item There exists a universal constant $c'>0$ such that for any $u \ge 1$,
\begin{align*}
\prob \bigg(\sup_{M \in \calM} \absenc{C_M(\nu)} \ge L\bigg) \le e^{-u}.
\end{align*}
where 
\begin{align*}
L=c'\bigg(&\gammafn{1} \left(\calM,\norm{\cdot}_2\right)+\gammafn{2} \left(\calM,\norm{\cdot}_F\right) 
+\sqrt{u}d_F(\calM)+ud_2(\calM)\bigg).
\end{align*}
\end{enumerate}
\end{theorem}
We also use the following theorem from \cite[Chapter 3]{Wainwright2019} in our analysis.
\begin{theorem}[\cite{Wainwright2019}] \label{thm:bounded-func-conc}
Consider $X = (X_1,X_2,\dots,X_n)$ independent random variables with $X_i \in [a,b]$ almost surely. Let $f:\matR^{n} \rightarrow \matR$ be convex and $L$-Lipschitz with respect to $\norm{\cdot}_2$ norm. Then $\forall t \ge 0$
\[\prob \left(|f(X)-\expec f(X)|\ge t\right) \le 2\exp \left(\frac{-t^2}{2L^2(b-a)^2}\right).\]
\end{theorem}


\section{Proof of Lemma~\ref{lem:unb-est}} \label{sec:unb-est}

\begin{proof}
Let $\expec_{\smpl}[\cdot]$ denote the expected value over $(\smpl_t)_{t=1}^T$, conditioned on $(x_t)_{t=1}^T$. Consider,
\begin{align*}
\norm{\widetilde{Z}-AZ}_F^2-(1-p)\norm{A \cdot \diag^{1/2}(ZZ^\top)}_F^2.
\end{align*}
We will use $g(A)=\norm{\widetilde{Z}-AZ}_F^2-(1-p)\norm{A \cdot \diag^{1/2}(ZZ^\top)}_F^2$ in the rest of the proof for ease of notation. Note that,
\begin{align*}
\norm{\widetilde{Z}-AZ}_F^2 = \norm{\widetilde{Z}}_F^2+\norm{AZ}_F^2-2\dotprod{\widetilde{Z}}{AZ}.
\end{align*}
Let us first consider $\expec_{\smpl}\left[\norm{AZ}_F^2\right]$ and $\expec_{\smpl} \left[\dotprod{\widetilde{Z}}{AZ}\right]$. We have,
\begin{align*}
\expec_{\smpl} \left[\norm{AZ}_F^2\right] &= \expec_{\smpl} \left[\Tr\left(ZZ^\top A^\top A\right)\right] = \Tr\left(\expec_{\smpl} \left[ZZ^\top\right] A^\top A\right)\\
&= \Tr\left(\expec_{\smpl} \left[\sum_{t=1}^{T-1} \smpl_t y_t y_t^\top {\smpl_t}^\top \right] A^\top A\right)\\
&= \Tr\left(\expec_{\smpl} \left[\sum_{t=1}^{T-1} P_t x_t x_t^\top P_t^\top \right] A^\top A\right).
\end{align*}
Since ${(\smpl_t)}_{ii}\sim B(p)$ for all $t \in [T]$,
\begin{equation*}
\expec_{\smpl} \left[\sum_{t=1}^{T-1} P_t x_t x_t^\top P_t^\top \right] =  \sum_{t=1}^{T-1} S \odot \left( x_t x_t^\top \right)
=  S \odot \sum_{t=1}^{T-1} \left( x_t x_t^\top \right) = S \odot (XX^\top),
\end{equation*}
where $S = p^2 \mathbf{1}\mathbf{1}^\top + p(1-p)I,$ where $\mathbf{1} = (1, \dots, 1)^\top \in \mathbb{R}^n$.

Similarly for $\dotprod{\widetilde{Z}}{AZ}$,
\begin{align*}
\expec_{\smpl} \left[\dotprod{\widetilde{Z}}{AZ}\right] 
&= \Tr\left(\expec_{\smpl} \left[\sum_{t=1}^{T-1} P_t x_t x_{t+1}^\top P_{t+1}^\top \right] A\right).
\end{align*}

Since $P_t$ and $P_{t+1}$ are independent,
\begin{equation*}
\expec_{\smpl} \left[\sum_{t=1}^{T-1} P_t x_t x_{t+1}^\top P_{t+1}^\top \right] =  \sum_{t=1}^{T-1} p^2 \left( x_t x_{t+1}^\top \right) = p^2 X\widetilde{X}^\top.
\end{equation*}
For simplicity, let $\mu = \expec_{\smpl}\left[\norm{\widetilde{Z}}_F^2\right]$. Note that $\mu$ is independent of $A$. Therefore,
\begin{align*}
\expec_{\smpl} \left[\norm{\widetilde{Z}-AZ}_F^2\right] = \mu+\Tr \left(\left\{S \odot XX^\top\right\} A^\top A\right)-2p^2\Tr \left(X\widetilde{X}^\top A\right).
\end{align*}

Next consider, $\expec_{\smpl} \left[\norm{A \cdot \diag^{1/2}\left(ZZ^\top\right)}_F^2\right]$. Note that,
\begin{align*}
\expec_{\smpl} \left[\norm{A \cdot \diag^{1/2}\left(ZZ^\top\right)}_F^2\right] & =\expec_{\smpl} \left[\Tr \left(\diag^{1/2}\left(ZZ^\top\right) A^\top A \diag^{1/2}\left(ZZ^\top\right)\right)\right]\\
&=\expec_{\smpl} \left[\Tr \left(\diag\left(ZZ^\top\right) A^\top A \right)\right]\\
&=\Tr \left(\expec_{\smpl}\left[\diag\left(ZZ^\top\right)\right] A^\top A \right)\\
&=p\Tr\left(\diag\left(XX^\top\right) A^\top A \right).
\end{align*}
Hence we obtain
\begin{align*}
\expec_{\smpl} \left[g(A)\right] &= \expec_{\smpl} \left[\norm{\widetilde{Z}-AZ}_F^2\right]-(1-p)\expec_{\smpl} \left[\norm{A \cdot \diag^{1/2}(ZZ^\top)}_F^2\right]\\
&=\mu+\Tr \left(S \odot XX^\top A^\top A\right)-2p^2\Tr \left(X\widetilde{X}^\top A\right)-(1-p)p\Tr\left(\diag\left(XX^\top\right) A^\top A \right)\\
&=\mu+p^2 \Tr \left(XX^\top A^\top A\right)-2p^2\Tr \left(X\widetilde{X}^\top A\right)\\
&=\mu + p^2 \left(\norm{AX}_F^2-2\dotprod{\widetilde{X}}{AX}+\norm{\widetilde{X}}_F^2\right)-p^2\norm{\widetilde{X}}_F^2\\
&=\mu'+p^2 \norm{\widetilde{X}-AX}_F^2,
\end{align*}
where $\mu'=\mu-p^2\norm{\widetilde{X}}_F^2$ (which does not depend on $A$). Therefore, we can see that,
\begin{align*}
\argmin{A \in \calK}  \set{\expec_{(x_t)_{t=1}^{T}}\expec_{\smpl} \left[g(A)\right]} &= \argmin{A \in \calK} \set{\expec_{(x_t)_{t=1}^{T}}\mu'+p^2 \expec_{(x_t)_{t=1}^{T}}\norm{\widetilde{X}-AX}_F^2}\\
&=\argmin{A \in \calK}\set{\expec_{(x_t)_{t=1}^{T}} \norm{\widetilde{X}-AX}_F^2}\\
&=\argmin{A \in \calK}\set{\expec_{(x_t)_{t=1}^{T}} \sum_{t=1}^{T-1} \norm{x_{t+1}-Ax_{t}}_2^2}.
\end{align*}
\end{proof}

\section{Proof of Theorem~\ref{thm:main-thm}}
\label{sec:tangent}

Let us recall the matrices $\widetilde{Z}$ and $Z$ and the objective $g(A) \coloneq \norm{\widetilde{Z}-AZ}_F^2-(1-p)\norm{A  \diag^{1/2}(ZZ^\top)}_F^2$ defined in Section~\ref{sec:prob-est}. For any $B \in \calK$ and $\widehat{A} \in \calK$ which minimizes $g(A)$, we first observe the following relationship between $\widehat{A}$ and $B$.

\begin{lemma}\label{lem:B-bound}
For any $B \in \calK$, the solution $\widehat{A} \in \calK$ to (\ref{eq:unb-est}) satisfies,
\begin{align}\label{eq:B-bound}
\norm{(B-\widehat{A})Z}_F^2-(1-p)\norm{(\widehat{A}-B)D}_F^2&\le 2\dotprod{\widetilde{Z}-A^*Z}{(\widehat{A}-B)Z} +2\norm{\left(A^*-B\right)Z}_F\norm{(\widehat{A}-B)Z}_F\\
&+2(1-p)\dotprod{A^*  D}{(\widehat{A}-B) D}+2(1-p)\norm{(A^*-B)D}_F \norm{(\widehat{A}-B)D}_F \nonumber.
\end{align}
where $D=\diag^{1/2}(ZZ^\top)$.
\end{lemma}
\begin{proof}[Proof Sketch]
The proof of Lemma~\ref{lem:B-bound} follows from considering the optimality of $\widehat{A}$ and the feasibility of $B$, which implies $g(\widehat{A}) \le g(B)$, and then rewriting and reordering some terms. The details are available in Appendix~\ref{appsec:main-inequal}.
\end{proof}

Given Lemma~\ref{lem:B-bound}, our goal is to control the left and right hand sides of the equation~\ref{eq:B-bound}. 
We first consider the left-hand-side of equation~\ref{eq:B-bound}, i.e. $\norm{(B-\widehat{A})Z}_F^2-(1-p)\norm{(\widehat{A}-B)D}_F^2$. 
\begin{lemma}\label{lem:LHS-bound}
There exist constants $c_1, C > 1$ such that for any $\delta \in (0,1)$ it holds with probability at least $1-\delta$,
\begin{align*}
\norm{(B-\widehat{A})Z}_F^2-(1-p)\norm{(\widehat{A}-B)D}_F^2 \ge &\Bigg[p^2 T -c_1 \bigg(J^2(A^*)\bar{L}_1+ p(1-p)J^2(A^*)\bar{L}_1\\ 
&+pJ(A^*)\sqrt{T} \left(\gammafn{2}(\conesetB,\norm{\cdot}_F)+\log (1/\delta)\right)\bigg) \Bigg] \norm{\widehat{A}-B}_F^2
\end{align*}
where 
\begin{align*}
\bar{L}_1 &= C \left(\gammafn{2}^2(\conesetB,\norm{\cdot}_2) +\gammafn{2}(\conesetB,\norm{\cdot}_2) \sqrt{T}+\sqrt{T} \log \left(1/\delta\right)\right),
\end{align*}
and $D=\diag^{1/2}(ZZ^\top)$.
\end{lemma}
The complete proof is available in Appendix~\ref{appsec:RE_proof_tangent}. 
For the right-hand-side of \eqref{eq:B-bound}, we establish the following upper bound.
\begin{lemma}\label{lem:RHS-bound}
There exist constants $c_1, C > 1$ such that for any $\delta \in (0,1)$ it holds with probability at least $1-\delta$ that,
\begin{align*}
&2\dotprod{\widetilde{Z}-A^*Z}{(\widehat{A}-B)Z}+2\norm{\left(A^*-B\right)Z}_F\norm{(\widehat{A}-B)Z}_F \\
&+2(1-p)\dotprod{A^*  D}{(\widehat{A}-B)D}+2(1-p)\norm{(A^*-B)D}_F \norm{(\widehat{A}-B)D}_F\\
& \le  c_1 \bigg(J^3(A^*)\norm{\widehat{A}-B}_F \bigg(\gammafn{1}(\conesetB,\norm{\cdot}_2)+\gammafn{2}(\conesetB,\norm{\cdot}_F)\sqrt{T}+\sqrt{T\log \left(1/\delta\right)}+\log (1/\delta)\bigg)\\
&+J^2(A^*) \norm{\widehat{A}-B}_F\norm{A^*-B}_F \bigg(p T+\bar{L}_1 +p(1-p)^2 \bar{L}_1+p\sqrt{T} \bigg(\gammafn{2}(\conesetB,\norm{\cdot}_F)+\sqrt{\log (1/\delta)}\bigg)\bigg)\bigg).
\end{align*}
where $D=\diag^{1/2}(ZZ^\top)$, and,
\begin{align*}
\bar{L}_1 &= C \left(\gammafn{2}^2(\conesetB,\norm{\cdot}_2) +\gammafn{2}(\conesetB,\norm{\cdot}_2) \sqrt{T}+\sqrt{T} \log \left(1/\delta\right)\right).
\end{align*}
\end{lemma}
The complete proof is available in Appendix~\ref{appsec:rhs_proof_tangent-appendix}. Given Lemmas ~\ref{lem:LHS-bound} and ~\ref{lem:RHS-bound}, we can now complete the proof of Theorem~\ref{thm:main-thm}. Denote %
\[
v=J^4(A^*)\max \set{\gammafn{2}^2(\conesetB,\norm{\cdot}_F),\log^2 \left(1/\delta\right)}
\]
and let $T \ge \frac{c}{p^4} v$ for a sufficiently large constant $c$. Then, we can see that,
\begin{align*}
pJ&(A^*) \sqrt{T} \left(\gammafn{2}(\conesetB,\norm{\cdot}_F)+\log (1/\delta)\right) \le  p \sqrt{vT} \le \frac{1}{c} p^3 T, \; \text{and}\\
\bar{L}_1 &= C \left(\gammafn{2}^2(\conesetB,\norm{\cdot}_2) +\gammafn{2}(\conesetB,\norm{\cdot}_2) \sqrt{T}+\sqrt{T} \log \left(1/\delta\right)\right) \\
&\le C \left(\gammafn{2}^2(\conesetB,\norm{\cdot}_F) +\gammafn{2}(\conesetB,\norm{\cdot}_F) \sqrt{T}+\sqrt{T} \log \left(1/\delta\right)\right)\\
&\le  C(v+2\sqrt{vT}) \le \frac{3C}{c}p^2 T.
\end{align*}
Therefore for sufficiently large $c$ we have,
\begin{align*}
&p^2 T -c_1\bigg(\bar{L}_1+ p(1-p)\bar{L}_1 +pJ(A^*)\sqrt{T} \bigg(\gammafn{2}(\conesetB,\norm{\cdot}_F)+\log (1/\delta)\bigg)\bigg)\ge \frac{p^2 T}{2}.
\end{align*}
Hence, from Lemma~\ref{lem:LHS-bound}, we obtain,
\begin{align*}
&\norm{(B-\widehat{A})Z}_F^2-(1-p)\norm{(\widehat{A}-B)\diag^{1/2}(ZZ^\top)}_F^2 \ge \frac{p^2 T}{2} \norm{\widehat{A}-B}_F^2.
\end{align*}
Now combining this with Lemma~\ref{lem:RHS-bound}, and using Lemma~\ref{lem:B-bound}, we have,
\begin{align}\label{eq:stat-bound-interim}
\frac{p^2 T}{2} \norm{\widehat{A}-B}_F^2 \le &   c_1 \bigg(J^3(A^*) \norm{\widehat{A}-B}_F \bigg(\gammafn{1}(\conesetB,\norm{\cdot}_2)+\gammafn{2}(\conesetB,\norm{\cdot}_F)\sqrt{T}+\sqrt{T\log \left(1/\delta\right)}+\log (1/\delta)\bigg) \nonumber\\
&+J^2(A^*) \norm{\widehat{A}-B}_F\norm{A^*-B}_F \bigg(p T+\bar{L}_1 +p(1-p)^2 \bar{L}_1+p\sqrt{T} \bigg(\gammafn{2}(\conesetB,\norm{\cdot}_F)+\sqrt{\log (1/\delta)}\bigg)\bigg)\bigg).
\end{align}
where 
\begin{align*}
\bar{L}_1 &= C \left(\gammafn{2}^2(\conesetB,2) +\gammafn{2}(\conesetB,2) \sqrt{T}+\sqrt{T} \log \left(1/\delta\right)\right)
\end{align*}
for constants $C,c_1,c_2$.

Note that $T \ge \frac{c J^4(A^*)\max \set{\gammafn{2}^2(\conesetB,\norm{\cdot}_F),\log^2 \left(1/\delta\right) }}{p^2}$ (for sufficiently large $c$), implies 
\begin{align*}
    pT &\ge C \left(\gammafn{2}^2(\conesetB,2) +\gammafn{2}(\conesetB,2) \sqrt{T}+\sqrt{T} \log \left(1/\delta\right)\right) = \bar{L}_1, \\
    \text{and } \ pT &\ge p\sqrt{T} \bigg(\gammafn{2}(\conesetB,\norm{\cdot}_F)+\sqrt{\log (1/\delta)}\bigg).
\end{align*}
Therefore \eqref{eq:stat-bound-interim} shows that there exists a constant $c_2 > 1$ such that,
\begin{align*}
\frac{p^2 T}{2} \norm{\widehat{A}-B}_F^2 \le   c_2 &\bigg(J^3(A^*) \norm{\widehat{A}-B}_F \bigg(\gammafn{1}(\conesetB,\norm{\cdot}_2)+\gammafn{2}(\conesetB,\norm{\cdot}_F)\sqrt{T}+\sqrt{T\log \left(1/\delta\right)}+\log (1/\delta)\bigg) \\
&+p J^2(A^*) \norm{\widehat{A}-B}_F\norm{A^*-B}_F T\bigg).
\end{align*}
Dividing both sides by $\frac{p^2 T}{2} \norm{\widehat{A}-B}_F$ we get the desired result.

\section{Proof of Lemma~\ref{lem:B-bound}}\label{appsec:main-inequal}
\begin{proof}
Consider the optimal solution $\widehat{A} \in \calK$ that minimizes $g(A)$ in \eqref{eq:unb-est}, and any $B \in \calK$. Since $B$ is  feasible, hence
\begin{equation}\label{eq:feas-opt-B}
\norm{\widetilde{Z}-\widehat{A}Z}_F^2-(1-p)\norm{\widehat{A}  D}_F^2 \le \norm{\widetilde{Z}-BZ}_F^2-(1-p)\norm{B  D}_F^2.
\end{equation}

Intercalating by $B$, the LHS of \eqref{eq:feas-opt-B} can be written as
\begin{align*}
&\norm{\widetilde{Z}-\widehat{A}Z}_F^2-(1-p)\norm{\widehat{A}  D}_F^2 = \norm{\widetilde{Z}-BZ+(B-\widehat{A})Z}_F^2-(1-p)\norm{B  D+(\widehat A-B)D}_F^2 \\
&= \left( \norm{\widetilde{Z}-BZ}_F^2-(1-p)\norm{B  D}_F^2\right)+\norm{(B-\widehat{A})Z}_F^2+2\dotprod{\widetilde{Z}-BZ}{(B-\widehat{A})Z}\\
&-(1-p)\norm{(\widehat{A}-B)D}_F^2-2(1-p)\dotprod{B  D}{(\widehat{A}-B)D}.
\end{align*}
Substituting the above into ~\eqref{eq:feas-opt-B}, and further intercalating by $A^*$, we have
\begin{align}\label{eq:prelim-AB-ineq}
&\norm{(B-\widehat{A})Z}_F^2-(1-p)\norm{(\widehat{A}-B)D}_F^2 \le 2\dotprod{\widetilde{Z}-BZ}{(\widehat{A}-B)Z}+2(1-p)\dotprod{B  D}{(\widehat{A}-B)D}\\
&\le 2\dotprod{\widetilde{Z}-A^*Z}{(\widehat{A}-B)Z}+2(1-p)\dotprod{A^*  D}{(\widehat{A}-B)D}\nonumber\\
&+2\dotprod{\left(A^*-B\right)Z}{(\widehat{A}-B)Z}-2(1-p)\dotprod{\left(A^*-B \right) D}{(\widehat{A}-B)D}\nonumber.
\end{align}

From the Cauchy-Schwartz inequality, we have
\begin{align*}
\dotprod{\left(A^*-B\right)Z}{(\widehat{A}-B)Z} &\le \norm{\left(A^*-B\right)Z}_F\norm{(\widehat{A}-B)Z}_F, \\
-\dotprod{\left(A^*-B \right) D}{(\widehat{A}-B)D} 
&\le \norm{(A^*-B)D}_F \norm{(\widehat{A}-B)D}_F.
\end{align*}
Combining this with ~\eqref{eq:prelim-AB-ineq} yields
\begin{align*}
&\norm{(B-\widehat{A})Z}_F^2-(1-p)\norm{(\widehat{A}-B)D}_F^2 \le 2\dotprod{\widetilde{Z}-A^*Z}{(\widehat{A}-B)Z}+2(1-p)\dotprod{A^*  D}{(\widehat{A}-B)D}\\
&+2(1-p)\norm{(A^*-B)D}_F \norm{(\widehat{A}-B)D}_F +2\norm{\left(A^*-B\right)Z}_F\norm{(\widehat{A}-B)Z}_F.
\end{align*}
\end{proof}

\section{Proof of Lemma~\ref{lem:LHS-bound}} \label{appsec:RE_proof_tangent}

\begin{proof}
In order to prove Lemma~\ref{lem:LHS-bound}, we start by observing the following. 

\begin{claim}\label{clm:quad-frob-norm}
Let $\breve{P} = \matrixenc{P_1 & 0 & \dots & 0\\ 0 & P_2 & \dots & 0\\\vdots & \vdots & \dots & \vdots\\ 0 & 0 & \dots & P_{T-1}}$, and $\optA = \frac{\widehat{A}-B}{\norm{\widehat{A}-B}_F}$. Then,
\begin{equation*}
    \norm{\left(\widehat{A}-B\right)Z}_F^2= \norm{\widehat{A}-B}_F^2  \eta^\top \Gamma^\top \left(\breve{P} \left(I_{T-1} \otimes \left(\optA^\top \optA\right)\right)\breve{P}\right)\Gamma \eta,
\end{equation*}
and 
\begin{equation*}
    \norm{\left(\widehat{A}-B\right)\diag^{1/2} \left(ZZ^\top\right)}_F^2=\norm{\widehat{A}-B}_F^2  \eta^\top \Gamma^\top \left(\breve{P} \left(I_{T-1} \otimes \diag\left(\optA^\top \optA\right)\right)\right)\Gamma \eta.
\end{equation*}
\end{claim}
The proof of Claim~\ref{clm:quad-frob-norm} is deferred to Appendix~\ref{appsec:clm-quad-frob-norm}. Letting $\optA = \frac{\widehat{A}-B}{\norm{\widehat{A}-B}_F}$, Claim~\ref{clm:quad-frob-norm} implies, 
\begin{align*}
&\norm{\optA Z}_F^2= \eta^\top \Gamma^\top \left(\breve{P} \left(I_{T-1} \otimes \left(\optA^\top \optA\right)\right)\breve{P}\right)\Gamma \eta = \norm{\left(\left(I_{T-1} \otimes \optA\right)\breve{P}\right)\Gamma \eta}_2^2,\\ 
&\norm{\optA\diag^{1/2} \left(ZZ^\top\right)}_F^2=\eta^\top \Gamma^\top \left(\breve{P} \left(I_{T-1} \otimes \diag\left(\optA^\top \optA\right)\right) \breve{P}\right)\Gamma \eta = \norm{\left(I_{T-1} \otimes \diag^{1/2}\left(\optA^\top \optA\right)\right) \breve{P}\Gamma \eta}_2^2.
\end{align*}
Therefore, we can see that,
\begin{align*}
&\norm{(\widehat{A}-B)Z}_F^2-(1-p)\norm{(\widehat{A}-B)\diag^{1/2} \left(ZZ^\top\right)}_F^2 \\
& = \norm{\widehat{A}-B}_F^2 \bigg(\norm{\optA Z}_F^2-(1-p)\norm{\optA\diag^{1/2} \left(ZZ^\top\right)}_F^2\bigg)\\
&= \norm{\widehat{A}-B}_F^2 \bigg(\norm{\left(I_{T-1} \otimes \optA\right) \breve{P}\Gamma \eta}_2^2-(1-p)\norm{\left(I_{T-1} \otimes \diag^{1/2}\left(\optA^\top \optA\right)\right) \breve{P}\Gamma \eta}_2^2\bigg).
\end{align*}
A lower bound on this can be achieved by first establishing a lower bound on $\norm{\left(I_{T-1} \otimes \optA \right) \breve{P}\Gamma \eta}_2^2$, and then an upper bound on $\norm{\left(I_{T-1} \otimes \diag^{1/2}\left(\optA^\top \optA\right)\right) \breve{P}\Gamma \eta}_2^2$. In the standard full-observation setting, this can be achieved through a careful application of Theorem~\ref{thm:psd-sup} on the quadratics formed in Claim~\ref{clm:quad-frob-norm}, since the entries of $\eta$ are mean $0$, variance $1$ and subgaussian. However, in our setting, a direct application of Theorem~\ref{thm:psd-sup} does not work as we have randomness over $(P_t)_{t=1}^T$ as well as $\eta$. 

\paragraph{Deriving bounds conditional on $(P_t)_{t=1}^T$.} We start by stating two key lemmas that enable us to bound the $\norm{AZ}_F^2-(1-p)\norm{A\diag^{1/2} \left(ZZ^\top\right)}_F^2$ conditioned on $(P_t)_{t=1}^T$.

\begin{lemma}\label{lem:M1-expec}
Let $\calM_1 = \set{\left(I_{T-1} \otimes A\right)\breve{P}\Gamma: A \in \conesetB}$. There exist constants $\bar{c}_1,\bar{C}_2 > 1$, such that for any $\delta \in (0,1)$, conditional on $(P_t)_{t=1}^T$, it holds with probability at least $1-\delta$ that  
\[
\forall \ M_1\in \calM_1: \ 
\norm{M_1\eta}_2^2 \in \left( \norm{M_1}_F^2  - \bar{c}_1 \left(\bar{F}_1+\bar{V}_1\log (1/\delta)\right),\norm{M_1}_F^2  + \bar{c}_1 \left(\bar{F}_1+\bar{V}_1\log (1/\delta)\right)\right),
\]
where, 
\begin{align*}
\bar{F}_1&=\bar{C}_2^2 J^2(A^*) \left(\gammafn{2}^2(\conesetB,\norm{\cdot}_2) +\gammafn{2}(\conesetB,\norm{\cdot}_2) \sqrt{T}+\sqrt{T}\right), \;\text{and},\\
\bar{V}_1&=\bar{C}_2 J^2(A^*)\left(\gammafn{2}(\conesetB,\norm{\cdot}_2) +\sqrt{T}\right).
\end{align*}
\end{lemma}
\begin{lemma}\label{lem:M2-expec}
Let $\calM_2 = \set{\left(I_{T-1} \otimes \diag^{1/2}\left(A^\top A\right)\right)\breve{P}\Gamma: A \in \conesetB}$. There exist constants $\bar{c}_1,\bar{C}_2 > 1$, such that for any $\delta \in (0,1)$, conditional on $(P_t)_{t=1}^T$, it holds with probability at least $1-\delta$ that 
\[
\forall \ M_2 \in \calM_2: \ (1-p)\norm{M_2\eta}_2^2 \le (1-p)\norm{M_2}_F^2+\bar{c}_1(1-p) \left(\bar{F}_1 + \bar{V}_1\log (1/\delta) \right),
\]
where, 
\begin{align*}
\bar{F}_1&=\bar{C}_2^2 J^2(A^*) \left(\gammafn{2}^2(\conesetB,\norm{\cdot}_2) +\gammafn{2}(\conesetB,\norm{\cdot}_2) \sqrt{T}+\sqrt{T}\right), \; \text{and},\\
\bar{V}_1&=\bar{C}_2 J^2(A^*)\left(\gammafn{2}(\conesetB,\norm{\cdot}_2) +\sqrt{T}\right).
\end{align*}
\end{lemma}
The proofs of Lemma~\ref{lem:M1-expec} and Lemma~\ref{lem:M2-expec} follow from an application of Theorem~\ref{thm:psd-sup} on appropriately defined sets of matrices. The complete proofs are available in Appendix~\ref{appsec:lem-M1-expec} and Appendix~\ref{appsec:lem-M2-expec} respectively.

For any $\delta \in (0,1)$, denote the events
\begin{align*}
\eventp{1a} &\coloneq \set{\forall M_1 \in \calM_1: \ \norm{M_1\eta}_2^2 \in \left( \norm{M_1}_F^2  - \bar{c}_1 \left(\bar{F}_1+\bar{V}_1\log (1/\delta)\right),\norm{M_1}_F^2  + \bar{c}_1 \left(\bar{F}_1+\bar{V}_1\log (1/\delta)\right)\right)}, \; \text{and,} \\
\eventp{2a} &\coloneq \set{\forall M_2 \in \calM_2: \ 
(1-p)\norm{M_2\eta}_2^2 \le (1-p)\norm{M_2}_F^2+\bar{c}_1(1-p) \left(\bar{F}_1 + \bar{V}_1\log (1/\delta) \right)
}.
\end{align*} 
Lemmas~\ref{lem:M1-expec} and ~\ref{lem:M2-expec} give us, 
$$\prob \left(\eventp{1a} \vert (P_t)_{t=1}^T \right) \ge 1-\delta \text{ and } \prob \left(\eventp{2a} \vert (P_t)_{t=1}^T\right) \ge 1-\delta.$$ 
Let $\optA = \frac{\widehat{A}-B}{\norm{\widehat{A}-B}_F} \in \conesetB$ and let $\bar{M}_1=(I_{T-1} \otimes \optA)\breve{P} \Gamma \in \calM_1$ and $\bar{M}_2 = \left(I_{T-1} \otimes \diag^{1/2} (\optA^\top \optA)\right)\breve{P}\Gamma \in \calM_2$. 
Then on $\eventp{1a} \cap \eventp{2a}$ we have 
\begin{align*}
\norm{\bar{M}_1\eta}_{2}^2 - (1-p)\norm{\bar{M}_2\eta}_{2}^2 &\ge \norm{\left(I_{T-1} \otimes \optA\right)\breve{P}\Gamma}_F^2- \bar{c}_1\left(\bar{F}_1+\bar{V}_1\log (1/\delta)\right)\\
&-(1-p)\norm{\left(I_{T-1} \otimes \diag^{1/2}(\optA^\top \optA)\right)\breve{P}\Gamma}_F^2
- (1-p)\bar{c}_1\left(\bar{F}_1+\bar{V}_1\log (1/\delta)\right)
\end{align*}
and letting $\overline{P}=\breve{P}-pI$, we obtain
\begin{align*}
\norm{\bar{M}_1\eta}_{2}^2 - (1-p)\norm{\bar{M}_2\eta}_{2}^2\ge &\norm{\left(I_{T-1} \otimes \optA\right)\overline{P}\Gamma+p\left(I_{T-1} \otimes \optA\right)\Gamma}_F^2- c\left(\bar{F}_1+\bar{V}_1\log (1/\delta)\right)\\
&-(1-p)\norm{\left(I_{T-1} \otimes \diag^{1/2}(\optA^\top \optA)\right)\overline{P}\Gamma+p\left(I_{T-1} \otimes \diag^{1/2}(\optA^\top \optA)\right)\Gamma}_F^2,
\end{align*}
for some universal constant $c$. Expanding this gives us,
\begin{align}\label{eq:cond-bound}
&\norm{\bar{M}_1\eta}_{2}^2 - (1-p)\norm{\bar{M}_2\eta}_{2}^2  \nonumber \\
&\ge \norm{\left(I_{T-1} \otimes \optA\right)\overline{P}\Gamma}_F^2+ p^2 \norm{\left(I_{T-1} \otimes \optA\right)\Gamma}_F^2+2p\dotprod{\left(I_{T-1} \otimes \optA\right)\overline{P}\Gamma}{\left(I_{T-1} \otimes \optA\right)\Gamma}  \nonumber\\
&-(1-p)\norm{\left(I_{T-1} \otimes \diag^{1/2}\left(\optA^\top \optA\right)\right)\overline{P}\Gamma}_F^2-p^2(1-p)\norm{\left(I_{T-1} \otimes \diag^{1/2}\left(\optA^\top \optA\right)\right)\Gamma}_F^2\nonumber\\
&-2p(1-p)\dotprod{\left(I_{T-1} \otimes \diag^{1/2}\left(\optA^\top \optA\right)\right)\overline{P}\Gamma}{\left(I_{T-1} \otimes \diag^{1/2}\left(\optA^\top \optA\right)\right)\Gamma}-c \left(\bar{F}_1+\bar{V}_1\log (1/\delta)\right).
\end{align}

\paragraph{Deriving bounds uniformly over all $A \in \conesetB$.} We now need to lower-bound the terms in the RHS of \eqref{eq:cond-bound}, uniformly over $A \in \conesetB$. We start by defining a set of intermediate structures. Consider $\Gamma \Gamma^\top \in \matR^{n(T-1) \times n(T-1)}$. Let $G_1,G_2,\dots,G_{T-1} \in \matR^{n \times n}$ be a set of matrices where $G_t$ is the $t$-th block matrix on the diagonal of $\Gamma \Gamma^\top$, i.e.,
\[
G_t = \matrixenc{
(\Gamma \Gamma^\top)_{(t-1)n+1,(t-1)n+1} & (\Gamma \Gamma^\top)_{(t-1)n+1,(t-1)n+2} & \dots & (\Gamma \Gamma^\top)_{(t-1)n+1,tn}\\
(\Gamma \Gamma^\top)_{(t-1)n+2,(t-1)n+1} & (\Gamma \Gamma^\top)_{(t-1)n+2,(t-1)n+2} & \dots & (\Gamma \Gamma^\top)_{(t-1)n+2,tn}\\
\vdots & \vdots & \dots & \vdots\\
(\Gamma \Gamma^\top)_{tn,(t-1)n+1} & (\Gamma \Gamma^\top)_{tn,(t-1)n+2} & \dots & (\Gamma \Gamma^\top)_{tn,tn}\\
}.
\]
Let $G = \blkdiag(\matrixenc{ G_1 & G_2 & \dots & G_{T-1} })$. Also for all $t \in [T-1]$, let $q_t= \matrixenc{\left(\overline{P}_t\right)_{11} \\ \left(\overline{P}_t\right)_{22} \\ \vdots \\ \left(\overline{P}_t\right)_{nn}} \in \matR^n$. Then for any $A \in \conesetB$,
\begin{align*}
\norm{\left(I_{T-1} \otimes A\right)\overline{P}\Gamma}_F^2 &= \dotprod{\overline{P}\left(I_{T-1} \otimes \left(A^\top A\right)\right)\overline{P}}{\Gamma \Gamma^\top}\\
&=\sum_{t=1}^{T-1} \dotprod{G_t}{\overline{P}_t A^\top A \overline{P}_t}\\
&=\sum_{t=1}^{T-1} q_t^\top \left(G_t \odot A^\top A\right)q_t \\
& = q^\top \left(G \odot \left(I_{T-1} \otimes \left(A^\top A\right)\right)\right) q,
\end{align*}
where $q = \vect\left(\matrixenc{q_1 & q_2 & \dots & q_{T-1}}\right) \in \matR^{n(T-1)}$. Note that since $G_t$ is positive semi-definite, by Schur Product Theorem, $\left(G_t \odot A^\top A\right)$ is positive semi-definite for all $t \in [T-1]$ and therefore $\left(G \odot \left(I_{T-1} \otimes \left(A^\top A\right)\right)\right)$ is positive semi-definite. A similar argument shows that,
\begin{align*}
\norm{\left(I_{T-1} \otimes \diag^{1/2}\left(A^\top A\right)\right)\overline{P}\Gamma}_F^2 &= \dotprod{\overline{P}\left(I_{T-1} \otimes \diag\left(A^\top A\right)\right)\overline{P}}{\Gamma \Gamma^\top}\\
&=\sum_{t=1}^{T-1} \dotprod{G_t}{\overline{P}_t \diag\left(A^\top A\right) \overline{P}_t}\\
&=\sum_{t=1}^{T-1} q_t^\top \left(G_t \odot \diag\left(A^\top A\right)\right)q_t \\
& = q^\top \left(G \odot \left(I_{T-1} \otimes \diag\left(A^\top A\right)\right)\right) q,
\end{align*}
where as before, $\left(G \odot \left(I_{T-1} \otimes \diag\left(A^\top A\right)\right)\right)$ is a positive semi-definite matrix. Furthermore, we have,
\begin{align*}
\dotprod{\left(I_{T-1} \otimes A\right)\overline{P}\Gamma}{\left(I_{T-1} \otimes A\right)\Gamma}
&=\dotprod{\overline{P}}{\left(I_{T-1} \otimes \left(A^\top A\right)\right)\Gamma \Gamma^\top},
\end{align*}
and,
\begin{align*}
\dotprod{\left(I_{T-1} \otimes \diag^{1/2}\left(A^\top A\right)\right)\overline{P}\Gamma}{\left(I_{T-1} \otimes \diag^{1/2}\left(A^\top A\right)\right)\Gamma}
&=\dotprod{\overline{P}}{\left(I_{T-1} \otimes \diag\left(A^\top A\right)\right)\Gamma \Gamma^\top}.
\end{align*}
Therefore,
\begin{align*}
&\norm{AZ}_{F}^2 - (1-p)\norm{A\diag^{1/2}(ZZ^\top)}_{F}^2 \\
&\ge q^\top \left(G \odot \left(I_{T-1} \otimes \left(A^\top A\right)\right)\right) q-(1-p)q^\top \left(G \odot \left(I_{T-1} \otimes \diag \left(A^\top A\right)\right)\right) q\\
&+ p^2 \norm{\left(I_{T-1} \otimes A\right)\Gamma}_F^2-p^2(1-p)\norm{\left(I_{T-1} \otimes \diag^{1/2}\left(A^\top A\right)\right)\Gamma}_F^2\\
&+2p\dotprod{\overline{P}}{\left(I_{T-1} \otimes \left(A^\top A-(1-p)\diag\left(A^\top A\right)\right)\right)\Gamma \Gamma^\top}-c_1\bar{F}_1-\bar{V}_1\sqrt{U}.
\end{align*}

Now we need to uniformly bound,
\begin{enumerate}
\item $q^\top \left(G \odot \left(I_{T-1} \otimes \left(A^\top A\right)\right)\right) q$,
\item $(1-p)q^\top \left(G \odot \left(I_{T-1} \otimes \diag \left(A^\top A\right)\right)\right) q$, and,
\item $2p\dotprod{\overline{P}}{\left(I_{T-1} \otimes \left(A^\top A-(1-p)\diag\left(A^\top A\right)\right)\right)\Gamma \Gamma^\top}$,
\end{enumerate}
over all $A \in \conesetB$. We establish the following three lemmas which allows us to control each of the given three terms.

\begin{lemma}\label{lem:M3-expec}
There exist  constants $\bar{c}_2,\bar{C}_2 > 1$, such that for any $\delta \in (0,1)$, with probability at least $1-\delta$, it holds for all $A \in \conesetB$ 
\[
q^\top \left(G \odot \left(I_{T-1} \otimes (A^\top A)\right)\right) q \in p(1-p) \left(\norm{\left(G \odot \left(I_{T-1} \otimes (A^\top A)\right)\right)^{1/2}}_F^2\pm \bar{c}_2\left(\bar{F}_1+\bar{V}_1\log (1/\delta)\right)\right),
\]
where
\begin{equation*}
\bar{F}_1 = \bar{C}_2^2 J^2(A^*) \left( \gammafn{2}^2(\conesetB,\norm{\cdot}_2) +\gammafn{2}(\conesetB,\norm{\cdot}_2) \sqrt{T}+\sqrt{T}\right),
\bar{V}_1 = \bar{C}_2 J^2(A^*)\left(\gammafn{2}(\conesetB,\norm{\cdot}_2) +\sqrt{T}\right).
\end{equation*}
\end{lemma}
\begin{lemma}\label{lem:M4-expec}
There exist  constants $\bar{c}_2,\bar{C}_2 > 1$, such that for any $\delta \in (0,1)$, with probability at least $1-\delta$, it holds for all $A \in \conesetB$ that
\[
q^\top \left(G \odot \left(I_{T-1} \otimes \diag(A^\top A)\right)\right) q \le p(1-p) \left(\norm{\left(G \odot \left(I_{T-1} \otimes \diag(A^\top A)\right)\right)^{1/2}}_F^2+\bar{c}_2\left(\bar{F}_1+\bar{V}_1\log (1/\delta)\right)\right) 
\]
where
\begin{equation*}
\bar{F}_1 = \bar{C}_2^2 J^2(A^*) \left(\gammafn{2}^2(\conesetB,\norm{\cdot}_2) +\gammafn{2}(\conesetB,\norm{\cdot}_2) \sqrt{T}+\sqrt{T}\right),
\bar{V}_1 = \bar{C}_2 J^2(A^*)\left(\gammafn{2}(\conesetB,\norm{\cdot}_2) +\sqrt{T}\right).
\end{equation*}
\end{lemma}
\begin{lemma}\label{lem:SR-expec} There exist constants $\widetilde{c}_1,\widetilde{c}_2 > 1$ such that for $\delta \in (0,1)$ the following statements hold.
\begin{enumerate}
\item With probability at least $1-\delta$, for all $A \in \conesetB$,
\begin{equation*}
2p\dotprod{\overline{P}}{\left(I_{T-1} \otimes \left(A^\top A-(1-p)\diag\left(A^\top A\right)\right)\right)\Gamma \Gamma^\top} 
\ge -\widetilde{c}_1pJ(A^*)\sqrt{T} \left(\gammafn{2} \left(\conesetB,\norm{\cdot}_F\right)+\sqrt{\log (1/\delta)}\right).
\end{equation*}

\item With probability at least $1-\delta$, for all $A \in \conesetB$,
\[
2p\dotprod{\overline{P}}{\left(I_{T-1} \otimes \left(A^\top A\right)\right)\Gamma \Gamma^\top} \le \widetilde{c}_2pJ(A^*)\sqrt{T} \left(\gammafn{2} \left(\conesetB,\norm{\cdot}_F\right)+\sqrt{\log (1/\delta)}\right),
\]

\item With probability at least  $1-\delta$, for all $A \in \conesetB$,
\begin{equation*}
2p(1-p)\dotprod{\overline{P}}{\left(I_{T-1} \otimes \diag \left(A^\top A\right)\right)\Gamma \Gamma^\top} \le \widetilde{c}_2p(1-p)J(A^*)\sqrt{T} \left(\gammafn{2} \left(\conesetB,\norm{\cdot}_F\right)+\sqrt{\log (1/\delta)}\right).
\end{equation*}
\end{enumerate}
\end{lemma}

Proofs of Lemma~\ref{lem:M3-expec} (deferred to Appendix~\ref{appsec:lem-M3-expec}) and Lemma~\ref{lem:M4-expec} (deferred to Appendix~\ref{appsec:lem-M4-expec}) follow from applying Theorem~\ref{thm:psd-sup} on appropriately defined sets of matrices. The proof of Lemma~\ref{lem:SR-expec} (deferred to Appendix~\ref{appsec:lem-SR-expec}) follows from applying  Theorem~\ref{thm:bounded-func-conc}. 

Let us now denote for any $A \in \conesetB$,
\begin{align*}
L_1 &= \left(\norm{\left(G \odot \left(I_{T-1} \otimes (A^\top A)\right)\right)^{1/2}}_F^2\pm \bar{c}_2\left(\bar{F}_1+\bar{V}_1\log (1/\delta)\right)\right),\\ 
L_2 &=\left(\norm{\left(G \odot \left(I_{T-1} \otimes \diag(A^\top A)\right)\right)^{1/2}}_F^2+\bar{c}_2\left(\bar{F}_1+\bar{V}_1\log (1/\delta)\right)\right) \text{ and,}\\
L_3&=\left(\gammafn{2} \left(\conesetB,\norm{\cdot}_F\right)+\sqrt{\log (1/\delta)}\right), 
\end{align*}
and for any $\delta \in (0,1)$ the events
\begin{align*}
&\eventp{1b} \coloneq \set{\forall  A \in \conesetB: \ q^\top \left(G \odot \left(I_{T-1} \otimes (A^\top A)\right)\right) q \in p(1-p) L_1},\\
&\eventp{2b} \coloneq \set{\forall  A \in \conesetB: \ q^\top \left(G \odot \left(I_{T-1} \otimes \diag(A^\top A)\right)\right) q \le p(1-p) L_2},\\
&\eventp{3b} \coloneq \set{\forall  A \in \conesetB: \ 2p\dotprod{\overline{P}}{\left(I_{T-1} \otimes \left(A^\top A-(1-p)\diag\left(A^\top A\right)\right)\right)\Gamma \Gamma^\top} \ge -\widetilde{c}_1pJ(A^*)\sqrt{T}L_3}. 
\end{align*}
Here $\widetilde{c}_1, \bar{c}_2 > 1$ are constants. As an implication of Lemmas~\ref{lem:M3-expec}, ~\ref{lem:M4-expec}, and ~\ref{lem:SR-expec} we get $\prob \left(\eventp{1b}\right) \ge 1-\delta$, $\prob \left(\eventp{2b}\right) \ge 1-\delta$, and, $\prob \left(\eventp{3b}\right) \ge 1-\delta$. 

\paragraph{Putting it all together.} Consider the event $\eventp{f} = \eventp{1a} \cap \eventp{2a} \cap\eventp{1b}\cap\eventp{2b}\cap\eventp{3b}$. We can see that under $\eventp{f}$  with probability at least $1-5\delta$,
\begin{align*}
&\norm{\optA Z}_{F}^2 - (1-p)\norm{\optA\diag^{1/2}(ZZ^\top)}_{F}^2 \\
&\ge p(1-p)\Tr \left(G \odot \left(I_{T-1} \otimes \left(\optA^\top \optA\right)\right)\right)+ p^2 \norm{\left(I_{T-1} \otimes \optA\right)\Gamma}_F^2\\
&-p(1-p)^2\Tr \left(G \odot \left(I_{T-1} \otimes \diag\left(\optA^\top \optA\right)\right)\right)- p^2(1-p) \norm{\left(I_{T-1} \otimes \diag^{1/2}\left(\optA^\top \optA\right)\right)\Gamma}_F^2\\
&-\bar{c}_2 p(1-p)\left(\bar{F}_1+\bar{V}_1\log (1/\delta)\right)-c_2pJ(A^*)\sqrt{T} \left(\gammafn{2} \left(\conesetB,\norm{\cdot}_F\right)+\log (1/\delta)\right)-\bar{c}_1\left(\bar{F}_1+\bar{V}_1\log (1/\delta)\right)\\
&=p^2 \norm{\left(I_{T-1} \otimes \optA\right)\Gamma}_F^2\\
&-\bar{c}_2 p(1-p)\left(\bar{F}_1+\bar{V}_1\log (1/\delta)\right)-c_3 pJ(A^*)\sqrt{T} \left(\gammafn{2} \left(\conesetB,\norm{\cdot}_F\right)+\sqrt{\log (1/\delta)}\right)-\bar{c}_1\left(\bar{F}_1+\bar{V}_1\log (1/\delta)\right)\\
&\ge p^2 T-\bar{c}_2p(1-p)\left(\bar{F}_1+\bar{V}_1\log (1/\delta)\right)-c_3 pJ(A^*)\sqrt{T} \left(\gammafn{2} \left(\conesetB,\norm{\cdot}_F\right)+\sqrt{\log (1/\delta)}\right)-\bar{c}_1\left(\bar{F}_1+\bar{V}_1\log (1/\delta)\right)\\
&\ge p^2 T -\bar{c}_1 J^2(A^*) \bar{L}_1- \bar{c}_2 p(1-p) J^2(A^*)\bar{L}_1 
-c_3 pJ(A^*)\sqrt{T} \left(\gammafn{2} \left(\conesetB,\norm{\cdot}_F\right)+\log (1/\delta)\right)
\end{align*}
for universal constants $\bar{c}_1,\bar{c}_2,c_3,$ and, $\bar{C}$ where 
\begin{align*}
\bar{L}_1 &= \bar{C}^2 \left(\gammafn{2}^2(\conesetB,\norm{\cdot}_2) +\gammafn{2}(\conesetB,\norm{\cdot}_2) \sqrt{T}+\sqrt{T} \log \left(1/\delta\right)\right).
\end{align*}
Multiplying both sides by $\norm{\widehat{A}-B}_F^2$ and suitably rescaling $\delta$, we get the desired result.
\end{proof}

\subsection{Proof of Claim~\ref{clm:quad-frob-norm}}\label{appsec:clm-quad-frob-norm}
\begin{proof}
First consider $\norm{\left(\widehat{A}-B\right)Z}_F^2$. 
Note that $\vect(Z) = \breve{P}x = \breve{P}\Gamma \eta$. Therefore,
\begin{align*}
\norm{\left(\widehat{A}-B\right)Z}_F^2 &=\norm{\widehat{A}-B}_F^2\norm{\optA Z}_F^2\\ 
& = \norm{\widehat{A}-B}_F^2\Tr \left(Z^\top \optA^\top \optA Z\right)\\
& = \norm{\widehat{A}-B}_F^2 \sum_{t=1}^{T-1} (Z_{:,t})^\top \optA^\top \optA Z_{:,t}\\
& = \norm{\widehat{A}-B}_F^2 \vect(Z)^\top \left(I_{T-1} \otimes \left(\optA^\top \optA\right)\right)\vect(Z)\\
&=\norm{\widehat{A}-B}_F^2  \eta^\top \Gamma^\top \breve{P} \left(I_{T-1} \otimes \left(\optA^\top \optA\right)\right)\breve{P}\Gamma \eta.
\end{align*}
where $Z_{:,t}$ is the $t$-th column of $Z$.

Similarly for $\norm{\left(\widehat{A}-B\right)\diag^{1/2} \left(ZZ^\top\right)}_F^2$, we get,
\begin{align*}
\norm{\left(\widehat{A}-B\right)\diag^{1/2} \left(ZZ^\top\right)}_F^2 &=\norm{\widehat{A}-B}_F^2\norm{\optA\diag^{1/2} \left(ZZ^\top\right)}_F^2\\ 
& = \norm{\widehat{A}-B}_F^2\dotprod{\optA^\top \optA}{\diag\left(ZZ^\top\right)}\\
& = \norm{\widehat{A}-B}_F^2\dotprod{\diag\left(\optA^\top \optA\right)}{\diag\left(ZZ^\top\right)}\\
& = \norm{\widehat{A}-B}_F^2\sum_{i=1}^n \left(\diag\left(\optA^\top \optA\right)\right)_{ii} \left({\diag\left(ZZ^\top\right)}\right)_{ii}\\
& = \norm{\widehat{A}-B}_F^2\sum_{i=1}^n \left(\diag\left(\optA^\top \optA\right)\right)_{ii} \sum_{t=1}^{T-1} Z_{it}^2\\
& = \norm{\widehat{A}-B}_F^2\sum_{t=1}^{T-1} \sum_{i=1}^n \left(\diag\left(\optA^\top \optA\right)\right)_{ii}  Z_{it}^2\\
& = \norm{\widehat{A}-B}_F^2\sum_{t=1}^{T-1} \left(Z_{:,t}\right)^\top \diag\left(\optA^\top \optA\right) Z_{:,t}\\
& = \norm{\widehat{A}-B}_F^2 \vect{(Z)}^\top \left(I_{T-1} \otimes \diag\left(\optA^\top \optA\right)\right) \vect{(Z)}\\
&=\norm{\widehat{A}-B}_F^2 \eta^\top \Gamma^\top \breve{P} \left(I_{T-1} \otimes \diag\left(\optA^\top \optA\right)\right) \breve{P} \Gamma \eta.
\end{align*}
\end{proof}
\subsection{Proof of Lemma~\ref{lem:M1-expec}}\label{appsec:lem-M1-expec}
\begin{proof}
Applying Theorem~\ref{thm:psd-sup} gives us
\begin{align*}
\prob \left(\sup_{M_1 \in \calM_1} \absenc{\norm{M_1\eta}_2^2-\expec\left[ \norm{M_1\eta}_2^2 \Big\vert (P_t)_{t=1}^T \right]} \ge c_1F_1+t \  \Big\vert (P_t)_{t=1}^T\right) \le 2\exp\left({-c_2 \min\set{\frac{t^2}{V_1^2}, \frac{t}{U_1}}}\right),
\end{align*}
where 
\begin{align*}
F_1 &= \gammafn{2}(\calM_1,\norm{\cdot}_2) \left(\gammafn{2}(\calM_1,\norm{\cdot}_2)+d_{F}(\calM_1)\right)+d_F(\calM_1)d_2(\calM_1),\\
V_1 &= d_2(\calM_1)\left(\gammafn{2}(\calM_1,\norm{\cdot}_2)+d_{F}(\calM_1)\right), \text{ and } U_1 = d_2(\calM_1)^2.
\end{align*}
Also for $M_1 = \left(I_{T-1} \otimes A\right)\breve{P}\Gamma \in \calM_1$, we have
\begin{align*}
\expec\left[\norm{M_1\eta}_2^2 \ \big\vert (P_t)_{t=1}^T \right] &= \norm{M_1}_F^2 \quad \text{(since entries of $\eta$ are of zero mean and unit variance.)} \\
& =\norm{\left(I_{T-1} \otimes A\right)\breve{P}\Gamma}_F^2. 
\end{align*}
The terms $F_1,V_1$ and $U_1$ are dependent on $\calM_1$ instead of our tangent cone $\conesetB$. In order to relate this to $\conesetB$ we first show that the mapping  $(\conesetB,\norm{\cdot}_2)  \mapsto (\calM_1, \norm{\cdot}_2)$ is Lipschitz. 
%
%
\begin{claim}\label{clm:lip-map}
The map $\left(\conesetB, \norm{\cdot}_2 \right) \mapsto \left(\calM_1,\norm{\cdot}_2\right)$ is $J(A^*)$-Lipschitz a.s.
\end{claim}
Since the map $\left(\conesetB, \norm{\cdot}_2 \right) \mapsto \left(\calM_1,\norm{\cdot}_2\right)$ is also onto, hence  Property~\ref{itm:prop-3}  from Appendix~\ref{sec:prelim-tools} implies,
\begin{align*}
\gammafn{2}(\calM_1,\norm{\cdot}_2) \le \bar{C}_2 J(A^*) \gammafn{2}(\conesetB,\norm{\cdot}_2),
\end{align*}
for a constant $\bar{C}_2 \ge 1$. Also, it holds a.s that
\begin{equation*}
d_2(\calM_1) = \sup_{A \in \conesetB} \norm{\left(I_{T-1} \otimes A\right)\breve{P}\Gamma}_2  \le \sup_{A \in \conesetB} \norm{\breve{P}}_2\norm{\Gamma}_2\norm{A}_2 
 \le J(A^*),
\end{equation*}
and,
\begin{align*}
d_F(\calM_1) &= \sup_{A \in \conesetB} \norm{\left(I_{T-1} \otimes A\right)\breve{P}\Gamma}_F \le \sup_{A \in \conesetB}\norm{\breve{P}}_2\norm{\Gamma}_2\norm{\left(I_{T-1} \otimes A\right)}_F \\
&\le\sup_{A \in \conesetB} J(A^*) \sqrt{T} \norm{A}_F \le J(A^*) \sqrt{T}.
\end{align*}

Let $t = V_1 \sqrt{u}$ for $u \ge 1$. Since $V_1 \ge U_1$ this means,
\[
\prob \left(\sup_{M_1 \in \calM_1} \absenc{\norm{M_1\eta}_2^2-\expec\left[ \norm{M_1 \eta}_2^2 \Big\vert (P_t)_{t=1}^T\right]} \ge c_1F_1+V_1\sqrt{u} \Big\vert (P_t)_{t=1}^T \right) \le 2\exp (-c_2\sqrt{u}).
\]
Note that,
\begin{align*}
F_1 &= \gammafn{2}(\calM_1,\norm{\cdot}_2) \left(\gammafn{2}(\calM_1,\norm{\cdot}_2)+d_{F}(\calM_1)\right)+d_F(\calM_1)d_2(\calM_1)\\
&\le \bar{C}_2 J(A^*) \gammafn{2}(\conesetB,\norm{\cdot}_2) \left(\bar{C}_2 J(A^*) \gammafn{2}(\conesetB,\norm{\cdot}_2)+J(A^*) \sqrt{T}\right)+J^2(A^*) \sqrt{T}\\
&\le \bar{C}_2^2 J^2(A^*) \left(\gammafn{2}^2(\conesetB,\norm{\cdot}_2) +\gammafn{2}(\conesetB,\norm{\cdot}_2) \sqrt{T}+\sqrt{T}\right),
\end{align*}
and,
\begin{align*}
V_1 &= d_2(\calM_1)\left(\gammafn{2}(\calM_1,\norm{\cdot}_2)+d_{F}(\calM_1)\right)\\
&\le J(A^*) \left(\bar{C}_2 J(A^*)\gammafn{2}(\conesetB,\norm{\cdot}_2) +J(A^*)\sqrt{T}\right)\\
&\le \bar{C}_2 J^2(A^*)\left(\gammafn{2}(\conesetB,\norm{\cdot}_2) +\sqrt{T}\right).
\end{align*}
Denoting the respective upper bounds by 
\begin{align}\label{eq:f1-v1-bar}
\bar{F}_1&=\bar{C}_2^2 J^2(A^*) \left(\gammafn{2}^2(\conesetB,\norm{\cdot}_2) +\gammafn{2}(\conesetB,\norm{\cdot}_2) \sqrt{T}+\sqrt{T}\right) \; \text{and},\\
\bar{V}_1&=\bar{C}_2 J^2(A^*)\left(\gammafn{2}(\conesetB,\norm{\cdot}_2) +\sqrt{T}\right),
\end{align}
we obtain,
\[
\prob \left(\sup_{M_1 \in \calM_1} \absenc{\norm{M_1\eta}_2^2-\expec\left[\norm{M_1 \eta}_2^2 \Big\vert (P_t)_{t=1}^T \right]} \ge c_1\bar{F}_1+\bar{V}_1\sqrt{u} \Big\vert (P_t)_{t=1}^T \right) \le 2\exp (-c_2\sqrt{u}).
\]

Setting $u = c\log^2 (1/\delta)$ for a sufficiently large $c$ (such that $\sqrt{c}c_2 \ge 1$), we get that conditional on $(P_t)_{t=1}^T$, it holds with probability at least $1-\delta$, 
\begin{align*}
\forall M_1 \in \calM_1: \ \norm{M_1\eta}_2^2 \in \left( \norm{M_1}_F^2  \pm \bar{c}_1 \left(\bar{F}_1+\bar{V}_1\log (1/\delta)\right)\right)
\end{align*}
where $\bar{c}_1$ depends on $c_1, c_2$.
\end{proof}

\begin{proof}[Proof of Claim~\ref{clm:lip-map}]
Consider $A_1, A_2 \in \conesetB$ and the corresponding terms $\left(I_{T-1} \otimes A_1\right)\breve{P}\Gamma \in \calM_1$ and, $\left(I_{T-1} \otimes A_2\right)\breve{P}\Gamma \in \calM_1$. Note that,
\begin{align*}
\norm{\left(I_{T-1} \otimes A_1\right)\breve{P}\Gamma-\left(I_{T-1} \otimes A_2\right)\breve{P}\Gamma}_2 & = \norm{\left(I_{T-1} \otimes \left(A_1-A_2\right)\right)\breve{P}\Gamma}_2\\
&\le \norm{\breve{P}}_2 \norm{\Gamma}_2 \norm{A_1-A_2}_2\\
&\le J(A^*)\norm{A_1-A_2}_2,
\end{align*}
and therefore, the map $\left(\conesetB, \norm{\cdot}_2 \right) \mapsto \left(\calM_1,\norm{\cdot}_2\right)$ is $J(A^*)$-Lipschitz.
\end{proof}

\subsection{Proof of Lemma~\ref{lem:M2-expec}}\label{appsec:lem-M2-expec}

\begin{proof}
Applying Theorem~\ref{thm:psd-sup} gives,
\begin{align*}
\prob \left(\sup_{M_2 \in \calM_2} \absenc{\norm{M_2\eta}_2^2-\expec\left[ \norm{M_2\eta}_2^2 \Big\vert (P_t)_{t=1}^T\right]} \ge c_1F_2+t \Big\vert (P_t)_{t=1}^T \right) \le 2\exp\left({-c_2 \min\set{\frac{t^2}{V_2^2}, \frac{t}{U_2}}}\right),
\end{align*}
where
\begin{align*}
F_2 &= \gammafn{2}(\calM_2,\norm{\cdot}_2) \left(\gammafn{2}(\calM_2,\norm{\cdot}_2)+d_{F}(\calM_2)\right)+d_F(\calM_2)d_2(\calM_2),\\ 
V_2 &= d_2(\calM_2)\left(\gammafn{2}(\calM_2,\norm{\cdot}_2)+d_{F}(\calM_2)\right), \text{ and } U_2 = d_2(\calM_2)^2.
\end{align*}

For $M_2 = \left(I_{T-1} \otimes \diag^{1/2}\left(A^\top A\right)\right)\breve{P}\Gamma \in \calM_2$, note that
\begin{align*}
\expec \left[\norm{M_2\eta}_2^2 \Big\vert (P_t)_{t=1}^T \right] &= \norm{M_2}_F^2 \quad \text{(Since entries of $\eta$ are of zero mean and unit variance.)} \\
& =\norm{\left(I_{T-1} \otimes \diag^{1/2}\left(A^\top A\right)\right)\breve{P}\Gamma}_F^2.
\end{align*}
The factors $F_2,V_2$ and $U_2$ depend on $\calM_2$ instead of  $\conesetB$. To relate this to $\conesetB$ we proceed similarly as before.
%
%
\begin{claim}\label{clm:lip-map-M2}
The map $\left(\conesetB, \norm{\cdot}_2 \right) \mapsto \left(\calM_2,\norm{\cdot}_2\right)$ is $J(A^*)$-Lipschitz a.s.
\end{claim}
Since the map from $\left(\conesetB, \norm{\cdot}_2 \right)$ to $\left(\calM_2,\norm{\cdot}_2\right)$ is also onto, hence Property~\ref{itm:prop-3} from Appendix~\ref{sec:prelim-tools} implies,
\[
\gammafn{2}(\calM_2,\norm{\cdot}_2) \le \bar{C}_2 J(A^*) \gammafn{2}(\conesetB,\norm{\cdot}_2),
\]
for a constant $\bar{C}_2 \ge 1$. Also, it holds a.s that
\begin{align*}
d_2(\calM_2) &= \sup_{A \in \conesetB} \norm{\left(I_{T-1} \otimes \diag^{1/2}\left(A^\top A\right)\right)\breve{P}\Gamma}_2 \\
& \le \norm{\breve{P}}_2\norm{\Gamma}_2\sup_{A \in \conesetB}\norm{\diag^{1/2}\left(A^\top A\right)}_2 \\
& \le J(A^*),
\end{align*}
and,
\begin{align*}
d_F(\calM_2) &= \sup_{A \in \conesetB} \norm{\left(I_{T-1} \otimes \diag^{1/2}\left(A^\top A\right)\right)\breve{P}\Gamma}_F \\
& \le \norm{\breve{P}}_2\norm{\Gamma}_2\sup_{A \in \conesetB}\norm{\left(I_{T-1} \otimes \diag^{1/2}\left(A^\top A\right)\right)}_F \\
& \le J(A^*) \sqrt{T} \sup_{A \in \conesetB} \norm{\diag^{1/2}\left(A^\top A\right)}_F \\
&\le J(A^*) \sqrt{T}.
\end{align*}

Let $t = V_2 \sqrt{u}$ for $u \ge 1$. Since $V_2 \ge U_2$ this means,
\[
\prob \left(\sup_{M_2 \in \calM_2} \absenc{\norm{M_2\eta}_2^2-\expec \left[ \norm{M_2 \eta}_2^2 \Big\vert (P_t)_{t=1}^T\right]} \ge c_1F_2+V_2\sqrt{u} \Big\vert (P_t)_{t=1}^T\right) \le 2\exp (-c_2\sqrt{u}).
\]
Note that $F_2, V_2$ can be respectively bounded as
\begin{align*}
F_2 &= \gammafn{2}(\calM_2,\norm{\cdot}_2) \left(\gammafn{2}(\calM_2,\norm{\cdot}_2)+d_{F}(\calM_2)\right)+d_F(\calM_2)d_2(\calM_2)\\
&\le \bar{C}_2 J(A^*) \gammafn{2}(\conesetB,\norm{\cdot}_2) \left(\bar{C}_2 J(A^*) \gammafn{2}(\conesetB,\norm{\cdot}_2)+J(A^*) \sqrt{T}\right)+J^2(A^*) \sqrt{T}\\
&\le \bar{C}_2^2 J^2(A^*) \left(\gammafn{2}^2(\conesetB,\norm{\cdot}_2) +\gammafn{2}(\conesetB,\norm{\cdot}_2) \sqrt{T}+\sqrt{T}\right)\\
&=\bar{F}_1,
\end{align*}
and,
\begin{align*}
V_2 &= d_2(\calM_2)\left(\gammafn{2}(\calM_2,\norm{\cdot}_2)+d_{F}(\calM_2)\right)\\
&\le J(A^*) \left(\bar{C}_2J(A^*)\gammafn{2}(\conesetB,\norm{\cdot}_2) +J(A^*)\sqrt{T}\right)\\
&\le \bar{C}_2J^2(A^*)\left(\gammafn{2}(\conesetB,\norm{\cdot}_2) +\sqrt{T}\right)\\
&=\bar{V}_1.
\end{align*}
Then we obtain,
\[
\prob \left(\sup_{M_2 \in \calM_2} \absenc{\norm{M_2\eta}_2^2-\expec \left[ \norm{M_2 \eta}_2^2 \Big\vert (P_t)_{t=1}^T\right]} \ge c_1\bar{F}_1+\bar{V}_1\sqrt{u} \Big\vert (P_t)_{t=1}^T\right) \le 2\exp (-c_2\sqrt{u}).
\]
Setting $u = c\log^2 (2/\delta)$ (for sufficiently large $c$) we get that, conditional on $(P_t)_{t=1}^T$, it holds with probability at least $1-\delta$, 
\begin{align*}
\forall M_2 \in \calM_2: \ \norm{M_2\eta}_2^2 \le \norm{M_2}_F^2 + \bar{c}_1 \left(\bar{F}_1+\bar{V}_1\log (1/\delta)\right),
\end{align*}
where $\bar{c}_1$ depends on $c_1, c_2$.
\end{proof}

\begin{proof}[Proof of Claim~\ref{clm:lip-map-M2}]
Consider $A_1, A_2 \in \conesetB$ and the corresponding $\left(I_{T-1} \otimes \diag^{1/2}\left(A_1^\top A_1\right)\right)\breve{P}\Gamma \in \calM_2$ and $\left(I_{T-1} \otimes \diag^{1/2}\left(A_2^\top A_2\right)\right)\breve{P}\Gamma \in \calM_2$. Note that,
\begin{align*}
&\norm{\left(I_{T-1} \otimes \diag^{1/2}\left(A_1^\top A_1\right)\right)\breve{P}\Gamma-\left(I_{T-1} \otimes \diag^{1/2}\left(A_2^\top A_2\right)\right)\breve{P}\Gamma}_2\\ 
& = \norm{\left(I_{T-1} \otimes \left(\diag^{1/2}\left(A_1^\top A_1\right)-\diag^{1/2}\left(A_2^\top A_2\right)\right)\right)\breve{P}\Gamma}_2\\
&\le \norm{\breve{P}}_2 \norm{\Gamma}_2 \norm{\diag^{1/2}\left(A_1^\top A_1\right)-\diag^{1/2}\left(A_2^\top A_2\right)}_2\\
&\le J(A^*)\norm{A_1-A_2}_2 \quad \text{a.s},
\end{align*}
and therefore, the map $\left(\conesetB, \norm{\cdot}_2 \right) \mapsto \left(\calM_1,\norm{\cdot}_2\right)$ is $J(A^*)$-Lipschitz a.s.

Note that the last inequality comes from the fact that  $\diag^{1/2}\left(A_1^\top A_1\right)-\diag^{1/2}\left(A_2^\top A_2\right)$ is diagonal and the $i$-th diagonal entry is $\norm{(A_1)_{:,i}}_2-\norm{(A_2)_{:,i}}_2$ (where $(A_1)_{:,i},(A_2)_{:,i}$ are the $i$-th columns of $A_1,A_2$),
\begin{align*}
\norm{\diag^{1/2}\left(A_1^\top A_1\right)-\diag^{1/2}\left(A_2^\top A_2\right)}_2 &\le \max_{i \in [n]} \absenc{\norm{(A_1)_{:,i}}_2-\norm{(A_2)_{:,i}}_2}\\
&\le \max_{i \in [n]} \norm{(A_1)_{:,i}-(A_2)_{:,i}}_2 \le \norm{A_1-A_2}_2.
\end{align*}
\end{proof}

\subsection{Proof of Lemma~\ref{lem:M3-expec}}\label{appsec:lem-M3-expec}

\begin{proof}
Given our goal of controlling $q^\top \left(G \odot \left(I_{T-1} \otimes (A^\top A)\right)\right) q$ uniformly over all $A \in \conesetB$, in order to apply Theorem~\ref{thm:psd-sup}, we need to first derive a matrix $\widetilde{M}_3$ such that $\widetilde{M}_3^\top \widetilde{M}_3 = G \odot \left(I_{T-1} \otimes (A^\top A)\right)$. However, since our matrix is a Hadamard product, the definition of $\widetilde{M}_3$ is not directly apparent. We start by first deriving $\widetilde{M}_3$ such that $\widetilde{M}_3^\top \widetilde{M}_3 = G \odot \left(I_{T-1} \otimes (A^\top A)\right)$. To do this we recall Theorem 1 of \cite{VISICK200045}. 

\begin{remark}[Theorem 1 of \cite{VISICK200045}]\label{rem:thm-had-kron}
Let $A,B \in \matR^{n \times m}$. Then, 
\[
A \odot B = Q_n^\top \bigg(A \otimes B\bigg) Q_m
\]
where $Q_n^\top = \matrixenc{\left(e_1^{(n)}{e_1^{(n)}}^\top\right) & \left(e_2^{(n)}{e_2^{(n)}}^\top\right) & \dots & \left(e_n^{(n)}{e_n^{(n)}}^\top\right)}$ ($e_i^{(n)}$ is the $i$-th standard basis vector in $\matR^n$).
\end{remark}
Using Remark~\ref{rem:thm-had-kron}, since $G$ is p.s.d, we can see that,
\begin{align*}
G \odot \left(I_{T-1} \otimes (A^\top A)\right) & = Q_{n(T-1)}^\top \left(G \otimes \left(I_{T-1} \otimes (A^\top A)\right)\right) Q_{n(T-1)}\\
& = Q_{n(T-1)}^\top \left((G^{1/2} G^{1/2}) \otimes \left(\left(I_{T-1} \otimes A^\top\right)\left(I_{T-1} \otimes A\right)\right)\right) Q_{n(T-1)}\\
& = Q_{n(T-1)}^\top \left(\left(G^{1/2} \otimes \left(I_{T-1} \otimes A^\top\right)\right)\left(G^{1/2} \otimes \left(I_{T-1} \otimes A\right)\right)\right) Q_{n(T-1)}\\
& = \widetilde{M}_3^\top \widetilde{M}_3 ,
\end{align*}
where $\widetilde{M}_3=\left(G^{1/2} \otimes \left(I_{T-1} \otimes A\right)\right) Q_{n(T-1)}$.

Let $\widetilde{q} = \frac{1}{\sqrt{p(1-p)}}q$. It is easy to see that the variance of the entries of $\widetilde{q}$ is $1$. Let 
\begin{align*}
\calM_3 &= \set{\left(G^{1/2} \otimes \left(I_{T-1} \otimes A\right)\right) Q_{n(T-1)}: A \in \conesetB}. \end{align*}
Applying  Theorem~\ref{thm:psd-sup} gives us,
\[
\prob \left(\sup_{M_3 \in \calM_3} \absenc{\norm{M_3\widetilde{q}}_2^2-\expec \norm{M_3\widetilde{q}}_2^2} \ge c_1F_3+t\right) \le 2\exp \left(-c_2 \min \set{\frac{t^2}{V_3},\frac{t}{U_3}}\right),
\]
where,
\begin{align*}
F_3 &= \gammafn{2}(\calM_3,\norm{\cdot}_2) \left(\gammafn{2}(\calM_3,\norm{\cdot}_2)+d_{F}(\calM_3)\right)+d_F(\calM_3)d_2(\calM_3),\\
V_3 &= d_2(\calM_3)\left(\gammafn{2}(\calM_3,\norm{\cdot}_2)+d_{F}(\calM_3)\right), \; \text{and}\; U_3 = d_2(\calM_3)^2.
\end{align*}

Note that, the factors $F_3,V_3$ and $U_3$ are dependent on $\calM_3$ instead of  $\conesetB$.  Proceeding as in the proof of Claim \ref{clm:lip-map-M2}, we have the following.
%
%
\begin{claim}\label{clm:lip-map-M3}
The map $\left(\conesetB, \norm{\cdot}_2 \right) \mapsto \left(\calM_3,\norm{\cdot}_2\right)$ is $J(A^*)$-Lipschitz.
\end{claim}
Since the map from $\left(\conesetB, \norm{\cdot}_2 \right)$ to $\left(\calM_3,\norm{\cdot}_2\right)$ is Lipschitz and onto, Property \ref{itm:prop-3} from Appendix~\ref{sec:prelim-tools} implies,
\[
\gammafn{2}(\calM_3,\norm{\cdot}_2) \le CJ(A^*) \gammafn{2}(\conesetB,\norm{\cdot}_2),
\]
for some constant $C \ge 1$. Using the definitions of $d_2(\cdot), d_F(\cdot)$ from Section~\ref{sec:prelim-tools} gives us,
\begin{align*}
d_2(\calM_3) &= \sup_{A \in \conesetB} \norm{\left(G^{1/2} \otimes \left(I_{T-1} \otimes A\right)\right) Q_{n(T-1)}}_2 \\
&\le \sup_{A \in \conesetB} \norm{\left(G^{1/2} \otimes \left(I_{T-1} \otimes A\right)\right)}_2 \\
& \le \norm{G}_2^{1/2} \sup_{A \in \conesetB}\norm{A}_2 \\
& \le J(A^*),
\end{align*}
and,
\begin{align*}
d_F(\calM_3) &= \left(\sup_{A \in \conesetB} \norm{\left(G^{1/2} \otimes \left(I_{T-1} \otimes A\right)\right) Q_{n(T-1)}}_F^2\right)^{1/2} \\
&= \left(\sup_{A \in \conesetB}\Tr \left(Q_{n(T-1)}^\top \left(\left(G^{1/2} \otimes \left(I_{T-1} \otimes A^\top\right)\right)\left(G^{1/2} \otimes \left(I_{T-1} \otimes A\right)\right)\right) Q_{n(T-1)}\right)\right)^{1/2} \\
&= \left(\sup_{A \in \conesetB}\Tr \left(Q_{n(T-1)}^\top \left(G \otimes \left(I_{T-1} \otimes (A^\top A)\right)\right) Q_{n(T-1)}\right)\right)^{1/2} \\
&= \left(\sup_{A \in \conesetB}\Tr \left(G \odot \left(I_{T-1} \otimes \left(A^\top A\right)\right)\right)\right)^{1/2} \\
&= \left(\sup_{A \in \conesetB}\sum_{t=1}^{T-1}\Tr \left(G_t \odot A^\top A\right)\right)^{1/2} \\
&= \left(\sup_{A \in \conesetB}\sum_{t=1}^{T-1} \sum_{i=1}^n (G_t)_{ii} (A^\top A)_{ii}\right)^{1/2} \\
& \le \left(\sup_{A \in \conesetB}\sum_{t=1}^{T-1} \max_{i} \absenc{(G_t)_{ii}}\sum_{i=1}^n \absenc{(A^\top A)_{ii} }\right)^{1/2}  \\
&\le \left(\sum_{t=1}^{T-1} J^2(A^*)\right)^{1/2} \le J(A^*) \sqrt{T}.
\end{align*}

Let $t = V_3 \sqrt{u}$ for $u \ge 1$. Since $V_3 \ge U_3$ this means,
\[
\prob \left(\sup_{M_3 \in \calM_3} \absenc{\norm{M_3\widetilde{q}}_2^2-\expec \norm{M_3 \widetilde{q}}_2^2 } \ge c_1F_3+V_3\sqrt{u}\right) \le 2\exp (-c_2\sqrt{u}).
\]
Note that,
\begin{align*}
F_3 &= \gammafn{2}(\calM_3,\norm{\cdot}_2) \left(\gammafn{2}(\calM_3,\norm{\cdot}_2)+d_{F}(\calM_3)\right)+d_F(\calM_3)d_2(\calM_3)\\
&\le \bar{C}_2J(A^*) \gammafn{2}(\conesetB,\norm{\cdot}_2) \left(\bar{C}_2J(A^*) \gammafn{2}(\conesetB,\norm{\cdot}_2)+J(A^*) \sqrt{T}\right)+J^2(A^*) \sqrt{T}\\
&\le \bar{C}_2^2 J^2(A^*) \left(\gammafn{2}^2(\conesetB,\norm{\cdot}_2) +\gammafn{2}(\conesetB,\norm{\cdot}_2) \sqrt{T}+\sqrt{T}\right)\\
&=\bar{F}_1,
\end{align*}
and,
\begin{align*}
V_3 &= d_2(\calM_3)\left(\gammafn{2}(\calM_3,\norm{\cdot}_2)+d_{F}(\calM_3)\right)\\
&\le J(A^*) \left(\bar{C}_2J(A^*)\gammafn{2}(\conesetB,\norm{\cdot}_2) +J(A^*)\sqrt{T}\right)\\
&\le \bar{C}_2 J^2(A^*)\left(\gammafn{2}(\conesetB,\norm{\cdot}_2) +\sqrt{T}\right)\\
&=\bar{V}_1.
\end{align*}
Then,
\[
\prob \left(\sup_{M_3 \in \calM_3} \absenc{\norm{M_3\widetilde{q}}_2^2-\expec \norm{M_3\widetilde{q}}_2^2 } \ge c_1\bar{F}_1+\bar{V}_1\sqrt{u}\right) \le 2\exp (-c_2\sqrt{u}).
\]

Letting $u = c\log^2 (2/\delta)$ (for sufficiently large $c$) we have with probability at least $1-\delta$, that 
$$\forall M_3 \in \calM_3: \  \norm{M_3\widetilde{q}}_2^2\in \left(\norm{M_3}_F^2 \pm \bar{c}_2 \left(\bar{F}_1+\bar{V}_1 \log (1/\delta)\right)\right),$$
which completes the proof.
\end{proof}
\begin{proof}[Proof of Claim~\ref{clm:lip-map-M3}]
Consider $A_1, A_2 \in \conesetB$ and the corresponding $\left(G^{1/2} \otimes \left(I_{T-1} \otimes A_1\right)\right) Q_{n(T-1)} \in \calM_3$ and $\left(G^{1/2} \otimes \left(I_{T-1} \otimes A_2\right)\right) Q_{n(T-1)} \in \calM_3$. Note that,
\begin{align*}
& \norm{\left(G^{1/2} \otimes \left(I_{T-1} \otimes A_1\right)\right) Q_{n(T-1)}-\left(G^{1/2} \otimes \left(I_{T-1} \otimes A_2\right)\right) Q_{n(T-1)}}_2\\ 
& \le \norm{\left(G^{1/2} \otimes \left(I_{T-1} \otimes A_1\right)\right)-\left(G^{1/2} \otimes \left(I_{T-1} \otimes A_2\right)\right)}_2\\
& = \norm{\left(G^{1/2} \otimes \left(I_{T-1} \otimes (A_1-A_2)\right)\right)}_2 \le \norm{G}_2^{1/2} \norm{A_1-A_2}_2.
\end{align*}
Since $\norm{G}_2 = \max_{t} \norm{G_t}_2 \le \norm{\Gamma \Gamma^\top}_2 \le J^2(A^*)$, we get,
\begin{align*}
\norm{\left(G^{1/2} \otimes \left(I_{T-1} \otimes A_1\right)\right) Q_{n(T-1)}-\left(G^{1/2} \otimes \left(I_{T-1} \otimes A_2\right)\right) Q_{n(T-1)}}_2 \le J(A^*) \norm{A_1-A_2}_2.
\end{align*}
Therefore, the map $\left(\conesetB, \norm{\cdot}_2 \right) \mapsto \left(\calM_3,\norm{\cdot}_2\right)$ is $J(A^*)$-Lipschitz.
\end{proof}

\subsection{Proof of Lemma~\ref{lem:M4-expec}}\label{appsec:lem-M4-expec}
\begin{proof}
Let $\widetilde{q} = \frac{1}{\sqrt{p(1-p)}}q$. It is easy to see that the variance of the entries of $\widetilde{q}$ is $1$. Let 
\begin{align*}
\calM_4 &= \set{\left(G \odot \left(I_{T-1} \otimes \diag \left(A^\top A\right)\right)\right)^{1/2}: A \in \conesetB}. 
\end{align*}

Applying the Theorem~\ref{thm:psd-sup} gives us
\[
\prob \left(\sup_{M_4 \in \calM_4} \absenc{\norm{M_4\widetilde{q}}_2^2-\expec \norm{M_4\widetilde{q}}_2^2} \ge c_1F_4+t\right) \le 2\exp \left(-c_2 \min \set{\frac{t^2}{V_4},\frac{t}{U_4}}\right),
\]
where,
\begin{align*}
F_4 &= \gammafn{2}(\calM_4,\norm{\cdot}_2) \left(\gammafn{2}(\calM_4,\norm{\cdot}_2)+d_{F}(\calM_4)\right)+d_F(\calM_4)d_2(\calM_4),\\
V_4 &= d_2(\calM_4)\left(\gammafn{2}(\calM_4,\norm{\cdot}_2)+d_{F}(\calM_4)\right),\; \text{and}, \; U_4 = d_2(\calM_4)^2.
\end{align*}

Proceeding as in the proof of Claim \ref{clm:lip-map}, we obtain the following claim.
%
%
\begin{claim}\label{clm:lip-map-M4}
The map $\left(\conesetB, \norm{\cdot}_2 \right) \mapsto \left(\calM_4,\norm{\cdot}_2\right)$ is $J(A^*)$-Lipschitz.
\end{claim}
Since the map from $\left(\conesetB, \norm{\cdot}_2 \right)$ to $\left(\calM_4,\norm{\cdot}_2\right)$ is Lipschitz and onto, Property \ref{itm:prop-3} from Appendix~\ref{sec:prelim-tools} implies,
\[
\gammafn{2}(\calM_4,\norm{\cdot}_2) \le CJ(A^*) \gammafn{2}(\conesetB,\norm{\cdot}_2),
\]
for some constant $C \ge 1$. Using the definitions of $d_2(\cdot), d_F(\cdot)$ from Section~\ref{sec:prelim-tools} gives us,
\begin{align*}
d_2(\calM_4) &= \sup_{A \in \conesetB} \norm{\left(G \odot \left(I_{T-1} \otimes \diag \left(A^\top A\right)\right)\right)^{1/2}}_2 \\
&= \sup_{A \in \conesetB} \norm{\left(\diag(G) \odot \left(I_{T-1} \otimes \diag\left(A^\top A\right)\right)\right)^{1/2}}_2 \\
&= \sup_{A \in \conesetB} \norm{\diag^{1/2}(G)  \left(I_{T-1} \otimes \diag^{1/2}\left(A^\top A\right)\right)}_2 \\
& \le \sup_{A \in \conesetB}\norm{\diag^{1/2}(G)}_2 \norm{\diag^{1/2}(A^\top A)}\\
&=\sup_{A \in \conesetB}\norm{\diag(G)}_2^{1/2} \norm{\diag(A^\top A)}^{1/2} \le \sup_{A \in \conesetB}\norm{G}_2^{1/2} \norm{A^\top A}_2^{1/2} \\
&\le \sup_{A \in \conesetB}J(A^*) \norm{A}_2 \le J(A^*) 
\end{align*}
and,
\begin{align*}
d_F(\calM_4) &= \left(\sup_{A \in \conesetB} \norm{\left(G \odot \left(I_{T-1} \otimes \diag \left(A^\top A\right)\right)\right)^{1/2}}_F^2\right)^{1/2} \\
&= \left(\sup_{A \in \conesetB}\Tr \left(G \odot \left(I_{T-1} \otimes \diag \left(A^\top A\right)\right)\right)\right)^{1/2} \\
&= \left(\sup_{A \in \conesetB}\sum_{t=1}^{T-1}\Tr \left(G_t \odot \diag \left(A^\top A\right)\right)\right)^{1/2} \\
&= \left(\sup_{A \in \conesetB}\sum_{t=1}^{T-1} \sum_{i=1}^n (G_t)_{ii} (A^\top A)_{ii}\right)^{1/2} \\
& \le \left(\sup_{A \in \conesetB}\sum_{t=1}^{T-1} \max_{i} \absenc{(G_t)_{ii}}\sum_{i=1}^n \absenc{(A^\top A)_{ii} }\right)^{1/2}  \\
&\le \left(\sum_{t=1}^{T-1} J^2(A^*)\right)^{1/2} \le J(A^*) \sqrt{T}.
\end{align*}

Let $t = V_4 \sqrt{u}$ for $u \ge 1$. Since $V_4 \ge U_4$ this means,
\[
\prob \left(\sup_{M_4 \in \calM_4} \absenc{\norm{M_4\widetilde{q}}_2^2-\expec \norm{M_4 \widetilde{q}}_2^2 } \ge c_1F_4+V_4\sqrt{u}\right) \le 2\exp (-c_2\sqrt{u}).
\]
Note that,
\begin{align*}
F_4 &= \gammafn{2}(\calM_4,\norm{\cdot}_2) \left(\gammafn{2}(\calM_4,\norm{\cdot}_2)+d_{F}(\calM_4)\right)+d_F(\calM_4)d_2(\calM_4)\\
&\le \bar{C}_2J(A^*) \gammafn{2}(\conesetB,\norm{\cdot}_2) \left(\bar{C}_2 J(A^*) \gammafn{2}(\conesetB,\norm{\cdot}_2)+J(A^*) \sqrt{T}\right)+J^2(A^*) \sqrt{T}\\
&\le \bar{C}_2^2 J^2(A^*) \left(\gammafn{2}^2(\conesetB,\norm{\cdot}_2) +\gammafn{2}(\conesetB,\norm{\cdot}_2) \sqrt{T}+\sqrt{T}\right)\\
&=\bar{F}_1,
\end{align*}
and,
\begin{align*}
V_4 &= d_2(\calM_4)\left(\gammafn{2}(\calM_4,\norm{\cdot}_2)+d_{F}(\calM_4)\right)\\
&\le J(A^*) \left(\bar{C}_2 J(A^*)\gammafn{2}(\conesetB,\norm{\cdot}_2) +J(A^*)\sqrt{T}\right)\\
&\le \bar{C}_2 J^2(A^*)\left(\gammafn{2}(\conesetB,\norm{\cdot}_2) +\sqrt{T}\right)\\
&=\bar{V}_1.
\end{align*}
Then we have,
\[
\prob \left(\sup_{M_4 \in \calM_4} \absenc{\norm{M_4\widetilde{q}}_2^2-\expec \norm{M_4\widetilde{q}}_2^2 } \ge c_1\bar{F}_1+\bar{V}_1\sqrt{u}\right) \le 2\exp (-c_2\sqrt{u}).
\]
Letting $u = c\log^2 (2/\delta)$ (for sufficiently large $c$) we have with probability at least $1-\delta$, that 
$$\forall M_4 \in \calM_4: \ \norm{M_4\widetilde{q}}_2^2\le \norm{M_4}_F^2+\bar{c}_2 \left(\bar{F}_1+\bar{V}_1 \log (1/\delta)\right).$$ 
This completes the proof.
\end{proof}

\begin{proof}[Proof of Claim~\ref{clm:lip-map-M4}]
Consider $A_1, A_2 \in \conesetB$ and the corresponding $\left(G \odot \left(I_{T-1} \otimes \diag(A_1^\top A_1)\right)\right)^{1/2} \in \calM_4$ and $\left(G \odot \left(I_{T-1} \otimes \diag(A_2^\top A_2)\right)\right)^{1/2} \in \calM_4$. Note that,
\begin{align*}
&\left(G \odot \left(I_{T-1} \otimes \diag\left(A_1^\top A_1\right)\right)\right)^{1/2}-\left(G \odot \left(I_{T-1} \otimes \diag\left(A_2^\top A_2\right)\right)\right)^{1/2}\\
&=\left(\diag(G) \odot \left(I_{T-1} \otimes \diag\left(A_1^\top A_1\right)\right)\right)^{1/2}-\left(\diag(G) \odot \left(I_{T-1} \otimes \diag\left(A_2^\top A_2\right)\right)\right)^{1/2}\\
&=\diag^{1/2}(G)\left(I_{T-1} \otimes \left(\diag^{1/2}\left(A_1^\top A_1\right)-\diag^{1/2}\left(A_2^\top A_2\right)\right)\right).
\end{align*}
Therefore,
\begin{align*}
& \norm{\left(G \odot \left(I_{T-1} \otimes \diag\left(A_1^\top A_1\right)\right)\right)^{1/2}-\left(G \odot \left(I_{T-1} \otimes \diag\left(A_2^\top A_2\right)\right)\right)^{1/2}}_2\\ 
& \le \norm{\diag^{1/2}(G)\left(I_{T-1} \otimes \left(\diag^{1/2}\left(A_1^\top A_1\right)-\diag^{1/2}\left(A_2^\top A_2\right)\right)\right)}_2\\
&\le \norm{\diag^{1/2}(G)}_2 \norm{\diag^{1/2}\left(A_1^\top A_1\right)-\diag^{1/2}\left(A_2^\top A_2\right)}_2.
\end{align*}
Note that,
\begin{align*}
\norm{\diag^{1/2}(G)}_2 &= \max_i \sqrt{G_{ii}} = \left(\max_i G_{ii}\right)^{1/2}\\
&= \left(\max_i \frac{e_i^\top G e_i}{e_i^\top e_i}\right)^{1/2} \; \text{(Where $e_i$ is the $i$-th standard basis vector)}\\
& \le \norm{G}_2^{1/2} \le J(A^*)
\end{align*}
and,
\begin{align*}
&\norm{\diag^{1/2}\left(A_1^\top A_1\right)-\diag^{1/2}\left(A_2^\top A_2\right)}_2 \\
&= \norm{\diag\left(\matrixenc{\norm{(A_1)_{:,1}}_2 & \norm{(A_1)_{:,2}}_2 & \dots & \norm{(A_1)_{:,n}}_2}\right)-\diag\left(\matrixenc{\norm{(A_2)_{:,1}}_2 & \norm{(A_2)_{:,2}}_2 & \dots & \norm{(A_2)_{:,n}}_2}\right)}_2\\
&=\max_{i} \absenc{\norm{(A_1)_{:,i}}_2-\norm{(A_2)_{:,i}}_2} \le \max_i \norm{(A_1)_{:,i}-(A_1)_{:,i}}_2 \le \norm{A_1-A_2}_2
\end{align*}
where,  $(A_1)_{:,i}, (A_2)_{:,i}$ are the $i$-th columns of $A_1, A_2$. Therefore,
\begin{align*}
\norm{\left(G \odot \left(I_{T-1} \otimes \diag\left(A_1^\top A_1\right)\right)\right)^{1/2}-\left(G \odot \left(I_{T-1} \otimes \diag\left(A_2^\top A_2\right)\right)\right)^{1/2}}_2 \le J(A^*) \norm{A_1-A_2}_2.
\end{align*}
and therefore, the map $\left(\conesetB, \norm{\cdot}_2 \right) \mapsto \left(\calM_3,\norm{\cdot}_2\right)$ is $J(A^*)$-Lipschitz.
\end{proof}
\subsection{Proof of Lemma~\ref{lem:SR-expec}}\label{appsec:lem-SR-expec}
\begin{proof}
Consider $2p\dotprod{\overline{P}}{\left(I_{T-1} \otimes \left(A^\top A\right)\right)\Gamma \Gamma^\top}$. Note that since $\overline{P}$ is a diagonal matrix, 
\[
2p\dotprod{\overline{P}}{\left(I_{T-1} \otimes \left(A^\top A\right)\right)\Gamma \Gamma^\top} = 2p\dotprod{\overline{P}}{\diag \left(\left(I_{T-1} \otimes \left(A^\top A\right)\right)\Gamma \Gamma^\top\right)}.
\]
Given $G_1,G_2,\dots,G_{T-1}$, 
\begin{equation*}
2p\dotprod{\overline{P}}{\diag \left(\left(I_{T-1} \otimes \left(A^\top A\right)\right)\Gamma \Gamma^\top\right)} =2p\sum_{t=1}^{T-1} \dotprod{P_t}{\diag \left(A^\top A G_t\right)} =2p\sum_{t=1}^{T-1} \dotprod{q_t}{r_t(A)},
\end{equation*}
where $r_t(A)= \matrixenc{\diag \left(A^\top A G_t\right)_{11} & \diag \left(A^\top A G_t\right)_{22} & \dots & \diag \left(A^\top A G_t\right)_{dd}}^\top$ for all $t \in [T-1]$. Denoting 
$$r(A) = \vect \left(\matrixenc{r_1(A) & r_2(A) & \dots & r_{T-1}(A)}\right),$$ we get,
\begin{equation*}
2p\dotprod{\overline{P}}{\left(I_{T-1} \otimes \left(A^\top A\right)\right)\Gamma \Gamma^\top} =2p\sum_{t=1}^{T-1} \dotprod{q_t}{r_t(A)} =2p \dotprod{q}{r(A)}.
\end{equation*}

Now consider $2p(1-p)\dotprod{\overline{P}}{\left(I_{T-1} \otimes \left(\diag\left(A^\top A\right)\right)\right)\Gamma \Gamma^\top}$. Note that,
\[
2p(1-p)\dotprod{\overline{P}}{\left(I_{T-1} \otimes \left(\diag\left(A^\top A\right)\right)\right)\Gamma \Gamma^\top} = 2p(1-p)\dotprod{\overline{P}}{\diag \left(\left(I_{T-1} \otimes \left(\diag\left(A^\top A\right)\right)\right)\Gamma \Gamma^\top\right)},
\]
which gives, 
\begin{align*}
2p(1-p)\dotprod{\overline{P}}{\diag \left(\left(I_{T-1} \otimes \left(\diag\left(A^\top A\right)\right)\right)\Gamma \Gamma^\top\right)} &=2p(1-p)\sum_{t=1}^{T-1} \dotprod{P_t}{\diag \left(\diag \left(A^\top A\right) G_t\right)}\\
&=2p(1-p)\sum_{t=1}^{T-1} \dotprod{q_t}{s_t(A)},
\end{align*}
where $s_t(A)= \matrixenc{\diag \left(\diag \left(A^\top A\right) G_t\right)_{11} & \dots & \diag \left(\diag \left(A^\top A\right) G_t\right)_{dd}}^\top$ for all $t \in [T-1]$. Letting $s(A) = \vect \left(\matrixenc{s_1(A) & s_2(A) & \dots & s_{T-1}(A)}\right)$, we get,
\begin{align*}
2p(1-p)\dotprod{\overline{P}}{\left(I_{T-1} \otimes \left(\diag\left(A^\top A\right)\right)\right)\Gamma \Gamma^\top} &=2p(1-p)\sum_{t=1}^{T-1} \dotprod{q_t}{s_t(A)}\\
&=2p(1-p) \dotprod{q}{s(A)}.
\end{align*}
Let $\calR =\set{r(A): A \in \conesetB}$ and $\calS =\set{s(A): A \in \conesetB}$. Then,
\begin{align*}
2p\dotprod{\overline{P}}{\left(I_{T-1} \otimes \left(A^\top A-(1-p)\diag\left(A^\top A\right)\right)\right)\Gamma \Gamma^\top} &= 2p \dotprod{q}{r(A)}-2p(1-p)\dotprod{q}{s(A)}\\
&\ge -2p\sup_{r \in \calR} \dotprod{-q}{r}-2p(1-p)\sup_{s\in \calS} \dotprod{q}{s}.
\end{align*}

We can now upper bound $\sup_{r \in \calR} \dotprod{-q}{r}$. Note that $-q \in \matR^{n(T-1)}$ has i.i.d, centered and bounded entries in $[-1,1]$ and therefore $1$-subgaussian. Note that $g_1(x) = \sup_{r \in \calR} \dotprod{x}{r}$ is convex and $J(A^*)\sqrt{T}$-Lipschitz (w.r.t $\norm{\cdot}_2$). Convexity is a direct result of linearity and Lipschitz bounds come from that fact that $g_1(x)-g_1(y) \le \norm{x-y}_2 \sup_{r \in \calR} \norm{r}_2$ (since for any $r \in \calR$, $\norm{r}_2^2 \le J^2(A^*)T$). 

In order to bound $\sup_{r\in \calR} \dotprod{-q}{r}$, we use the following remark from \cite{HDPbook}:
\begin{remark}[Remark 8.5.9 of \cite{HDPbook}]\label{rem:gaus-width} 
If $X$ is a $0$-mean subgaussian vector with $\norm{X}_{\psi_2}=K$, then, $\expec \sup_{t\in \calS} \dotprod{X}{t} \le c K w(\calS)$ for any bounded set $\calS \subset \matR^{n}$.
\end{remark}

Applying Remark~\ref{rem:gaus-width} on $\calR$ implies $\expec \sup_{r \in \calR} \dotprod{-q}{r} \le cw(\calR)$ and from Talagrand's majorizing measure theorem~\cite[Theorem 2.4.1]{talagrand2014upper}, 
\[
w(\calR) \asymp \gammafn{2} \left(\calR,\norm{\cdot}_2\right) = \gammafn{2}\left(\calR_{mat},\norm{\cdot}_F\right),
\]
where $\calR_{mat} = \set{\blkdiag\left(\matrixenc{\diag\left((A^\top A) G_1 \right) & \dots & \diag\left((A^\top A) G_{T-1} \right)}\right):A \in \conesetB}$. 

%
%
\begin{claim}\label{clm:lip-map-Rmat}
The map $\left(\conesetB, \norm{\cdot}_F\right) \mapsto \left(\calR_{mat},\norm{\cdot}_F\right)$ is $2\sqrt{T} J(A^*)$-Lipschitz.
\end{claim}
Since the map from $\left(\conesetB, \norm{\cdot}_F\right)$ to $\left(\calR_{mat},\norm{\cdot}_F\right)$ is Lipschitz and onto (it is easy to see this simply by definition), Property \ref{itm:prop-3} from Appendix~\ref{sec:prelim-tools} implies,
\[
\gammafn{2} \left(\calR_{mat},\norm{\cdot}_F\right) \le C\sqrt{T}J(A^*)\gammafn{2} \left(\conesetB,\norm{\cdot}_F\right).
\]
Furthermore, 
\[
\prob \left(\absenc{ \sup_{r \in \calR} \dotprod{-q}{r}-\expec \sup_{r \in \calR} \dotprod{-q}{r} } \ge t\right) \le 2\exp \left(-\frac{t^2}{8TJ^2(A^*)}\right).
\]
Note that $$\expec \sup_{r \in \calR} \dotprod{-q}{r} = w(\calR) = \gammafn{2} \left(\calR_{mat}, \norm{\cdot}_F\right) \le C\sqrt{T}J(A^*)\gammafn{2} \left(\conesetB,\norm{\cdot}_F\right).$$ Set $t = C\sqrt{T}J(A^*) \sqrt{\log (2/\delta)}$ and denote
\[
\eventp{1c} = \set{\sup_{r\in R} \dotprod{-q}{r} \le C\sqrt{T}J(A^*)\left(\gammafn{2} \left(\conesetB,\norm{\cdot}_F\right)+\sqrt{\log (2/\delta)}\right)}.
\]
Then,
\[
\prob \left(\eventp{1c}\right) \ge 1-\delta.
\]
We can see that by following the same argument, we have with probability at least $1-\delta$, $$\sup_{r\in R} \dotprod{q}{r} \le C\sqrt{T}J(A^*)\left(\gammafn{2} \left(\conesetB,\norm{\cdot}_F\right)+\sqrt{\log (2/\delta)}\right).$$
Therefore, with probability at least $1-\delta$, 
\[
\dotprod{\overline{P}}{\left(I_{T-1} \otimes \left(A^\top A\right)\right)\Gamma \Gamma^\top} \le C\sqrt{T}J(A^*)\left(\gammafn{2} \left(\conesetB,\norm{\cdot}_F\right)+\sqrt{\log (1/\delta)}\right).
\]

Next we can upper bound $\sup_{s \in \calS} \dotprod{q}{s}$. Note that $q \in \matR^{n(T-1)}$ has i.i.d, centered and bounded entries in $[-1,1]$ and therefore $1$-subgaussian. We can see that $g_2(x) = \sup_{s \in \calS} \dotprod{x}{s}$ is convex and also $J(A^*)\sqrt{T}$-Lipschitz. This can be verified following the same arguments as for $g_1(x)$. 

Applying Remark~\ref{rem:gaus-width} on $\calS$ implies $\expec \sup_{s \in \calS} \dotprod{q}{s} \le c w(\calS)$, hence from Talagrand's majorizing measure theorem~\cite[Theorem 2.4.1]{talagrand2014upper}, we obtain
\[
w(\calS) \asymp \gammafn{2} \left(\calS,\norm{\cdot}_2\right) = \gammafn{2}\left(\calS_{mat},\norm{\cdot}_F\right),
\]
where $\calS_{mat} = \set{\blkdiag\left(\matrixenc{\diag\left(\diag(A^\top A) G_1 \right) & \dots & \diag\left(\diag(A^\top A) G_{T-1} \right)}\right):A \in \conesetB}$. 
%
%
\begin{claim}\label{clm:lip-map-Smat}
The map $\left(\conesetB, \norm{\cdot}_F\right) \mapsto \left(\calS_{mat},\norm{\cdot}_F\right)$ is $2\sqrt{T} J(A^*)$-Lipschitz.
\end{claim}
Since the map from $\left(\conesetB, \norm{\cdot}_F\right)$ to $\left(\calS_{mat},\norm{\cdot}_F\right)$ is Lipschitz and onto (it is easy to see this simply by definition), Property \ref{itm:prop-3} from Appendix~\ref{sec:prelim-tools} implies,
\[
\gammafn{2} \left(\calS_{mat},\norm{\cdot}_F\right) \le C\sqrt{T}J(A^*)\gammafn{2} \left(\conesetB,\norm{\cdot}_F\right).
\]
Denote the event
\[
\eventp{2c} = \set{\sup_{s \in \calS} \dotprod{q}{s} \le C\sqrt{T}J(A^*)\left(\gammafn{2} \left(\conesetB,\norm{\cdot}_F\right)+\sqrt{\log (2/\delta)}\right)}.
\]
Then similar to before,
\[
\prob \left(\eventp{2c}\right) \ge 1-\delta.
\]

Putting events $\eventp{1c}, \eventp{2c}$ together, we get that w.p. at least $1-2\delta$, it holds for all $A \in \conesetB$,
\[
2p \dotprod{q}{r(A)}-2p(1-p)\dotprod{q}{s(A)} \ge -cpJ(A^*)\sqrt{T} \left(\gammafn{2} \left(\conesetB,\norm{\cdot}_F\right)+\sqrt{\log (1/\delta)}\right),
\]
for a suitably large constant $c$.
\end{proof}
\begin{proof}[Proof of Claim~\ref{clm:lip-map-Rmat}]
Consider $A_1, A_2 \in \conesetB$ and the corresponding $R_1, R_2 \in \calR_{mat}$. Note that,
\begin{align*}
\norm{R_1-R_2}_F^2 &= \sum_{t=1}^T \norm{\diag \left(A_1^\top A_1 G_t\right)-\diag \left(A_2^\top A_2 G_t\right)}_F^2\\
&\le \sum_{t=1}^T \norm{A_1^\top A_1 G_t-A_2^\top A_2 G_t}_F^2\\
&\le J^2(A^*) \sum_{t=1}^{T-1} \norm{A_1^\top A_1-A_2^\top A_2}_F^2 \\
&\le 4J^2(A^*) T \norm{A_1-A_2}_F^2.
\end{align*}
Therefore, the map $\left(\conesetB, \norm{\cdot}_F\right) \mapsto \left(\calR_{mat},\norm{\cdot}_F\right)$ is $2\sqrt{T} J(A^*)$-Lipschitz.
\end{proof}

\begin{proof}[Proof of Claim~\ref{clm:lip-map-Smat}]
Consider $A_1, A_2 \in \conesetB$ and the corresponding $S_1, S_2 \in \calS_{mat}$. Note that,
\begin{align*}
\norm{S_1-S_2}_F^2 &= \sum_{t=1}^T \norm{\diag \left(\diag(A_1^\top A_1) G_t\right)-\diag \left(\diag(A_2^\top A_2) G_t\right)}_F^2\\
&\le \sum_{t=1}^T \norm{\diag(A_1^\top A_1) G_t-\diag(A_2^\top A_2) G_t}_F^2\\
&\le J^2(A^*) \sum_{t=1}^{T-1} \norm{A_1^\top A_1-A_2^\top A_2}_F^2 \\
&\le 4J^2(A^*) T \norm{A_1-A_2}_F^2.
\end{align*}
Therefore, the map $\left(\conesetB, \norm{\cdot}_F\right) \mapsto \left(\calS_{mat},\norm{\cdot}_F\right)$ is $2\sqrt{T} J(A^*)$-Lipschitz.
\end{proof}

\section{Proof of Lemma~\ref{lem:RHS-bound}}\label{appsec:rhs_proof_tangent-appendix}

\begin{proof}
In order to derive the bounds in Lemma~\ref{lem:RHS-bound}, we need to bound the following three terms.
\begin{enumerate}
\item $2 \dotprod{\widetilde{Z}-A^*Z}{(\widehat{A}-B)Z}+2(1-p)\dotprod{A^*\diag^{1/2}(ZZ^\top)}{(\widehat{A}-B)\diag^{1/2}(ZZ^\top)}$ \label{itm:cross-prod-lem8},
\item $2\norm{(A^*-B)Z}_F\norm{(\widehat{A}-B)Z}_F$,
\item $2(1-p)\norm{(A^*-B)\diag^{1/2}(ZZ^\top)}_F\norm{(\widehat{A}-B)\diag^{1/2}(ZZ^\top)}_F$.
\end{enumerate}

Let $\optA = \frac{\widehat{A}-B}{\norm{\widehat{A}-B}_F}$ and $\widetilde{A}^* = \frac{A^*-B}{\norm{A^*-B}_F}$. Since $\widehat{A},A^* \in \calK$, we have $\optA,\widetilde{A}^* \in \conesetB$. Therefore
\begin{align*}
2\norm{(A^*-B)Z}_F\norm{(\widehat{A}-B)Z}_F &= 2 \norm{A^*-B}_F\norm{\widehat{A}-B}_F \norm{\widetilde{A}^* Z}_F\norm{\optA Z}_F\\ 
&\le 2 \norm{A^*-B}_F\norm{\widehat{A}-B}_F \sup_{A \in \conesetB} \norm{AZ}_F \sup_{A \in \conesetB} \norm{AZ}_F \\
&= 2 \norm{A^*-B}_F\norm{\widehat{A}-B}_F \sup_{A \in \conesetB} \norm{A Z}_F^2.
\end{align*}
Similarly,
\begin{align*}
&2(1-p)\norm{(A^*-B)\diag^{1/2}(ZZ^\top)}_F\norm{(\widehat{A}-B)\diag^{1/2}(ZZ^\top)}_F \\
&= 2(1-p) \norm{A^*-B}_F\norm{\widehat{A}-B}_F \norm{\widetilde{A}^*\diag^{1/2}(ZZ^\top)}_F\norm{\optA\diag^{1/2}(ZZ^\top)}_F\\ 
&\leq 2(1-p) \norm{A^*-B}_F\norm{\widehat{A}-B}_F \sup_{A \in \conesetB} \norm{A\diag^{1/2}(ZZ^\top)}_F^2.
\end{align*}
Recall Lemmas~\ref{lem:M1-expec} and ~\ref{lem:M2-expec} from the proof of Lemma~\ref{lem:LHS-bound} and the definitions of $\breve{P}$ and $\overline{P}$. 
By Lemmas~\ref{lem:M1-expec} and~\ref{lem:M2-expec}, we have for any $A \in \conesetB$, on the event $\eventp{1a}\cap \eventp{2a}$,
\begin{align*}
&\norm{AZ}_F^2 \le \norm{(I_{T-1} \otimes A) \breve{P} \Gamma}_F^2+\bar{c}_1\left(\bar{F}_1+\bar{V}_1\log (1/\delta)\right),\; \text{and},\\
&(1-p)\norm{A\diag^{1/2}(ZZ^\top)}_F^2 \le (1-p) \norm{(I_{T-1} \otimes \diag^{1/2}(A^\top A)) \breve{P} \Gamma}_F^2+\bar{c}_1(1-p)\left(\bar{F}_1+\bar{V}_1\log (1/\delta)\right),
\end{align*}
where $\bar{F}_1, \bar{V}_1$ are as in Lemmas ~\ref{lem:M1-expec} and ~\ref{lem:M2-expec}. Following the same arguments as in the proof of Lemma~\ref{lem:LHS-bound}, 
\begin{align*}
\norm{AZ}_{F}^2 &\le \norm{\left(I_{T-1} \otimes A\right)\overline{P}\Gamma}_F^2+2p\dotprod{\left(I_{T-1} \otimes A\right)\overline{P}\Gamma}{\left(I_{T-1} \otimes A\right)\Gamma} \\
&+ p^2 \norm{\left(I_{T-1} \otimes A\right)\Gamma}_F^2+\bar{c}_1 \left(\bar{F}_1+\bar{V}_1\log (1/\delta)\right),
\end{align*}
\quad \text{and} \quad
\begin{align*}
&(1-p)\norm{A \diag^{1/2}(ZZ^\top)}_{F}^2\\
&\le (1-p)\norm{\left(I_{T-1} \otimes \diag^{1/2}\left(A^\top A\right)\right)\overline{P}\Gamma}_F^2+p^2(1-p)\norm{\left(I_{T-1} \otimes \diag^{1/2}\left(A^\top A\right)\right)\Gamma}_F^2\\
&+2p(1-p)\dotprod{\left(I_{T-1} \otimes \diag^{1/2}\left(A^\top A\right)\right)\overline{P}\Gamma}{\left(I_{T-1} \otimes \diag^{1/2}\left(A^\top A\right)\right)\Gamma}+\bar{c}_1 \left(\bar{F}_1+\bar{V}_1\log (1/\delta)\right).
\end{align*}
Using Lemmas ~\ref{lem:M3-expec}, ~\ref{lem:M4-expec}, and ~\ref{lem:SR-expec}, under the event $\eventp{1b}\cap \eventp{2b} \cap \eventp{2c} \cap \eventp{3c}$ (see proofs of Lemmas~\ref{lem:LHS-bound} and~\ref{lem:SR-expec}), we have for all $A \in \conesetB$, 
\begin{align*}
\norm{AZ}_{F}^2 &\le p(1-p)\norm{\left(G \odot (I_{T-1} \otimes (A^\top A))\right)^{1/2}}_F^2+ p^2 \norm{\left(I_{T-1} \otimes A\right)\Gamma}_F^2\\
&+c_3 pJ(A^*)\sqrt{T} \left(\gammafn{2} \left(\conesetB,\norm{\cdot}_F\right)+\sqrt{\log (1/\delta)}\right) \\
&+\bar{c}_1 \left(\bar{F}_1+\bar{V}_1\log (1/\delta)\right)+p(1-p) \bar{c}_2 \left(\bar{F}_1+\bar{V}_1\log (1/\delta)\right),
\end{align*}
and,
\begin{align*}
&(1-p)\norm{A \diag^{1/2}(ZZ^\top)}_{F}^2\\
&\le p(1-p)^2 \norm{\left(G \odot \left(I_{T-1} \otimes \diag(A^\top A)\right)\right)^{1/2}}_F^2+p^2(1-p)\norm{\left(I_{T-1} \otimes \diag^{1/2}\left(A^\top A\right)\right)\Gamma}_F^2\\
&+c_3 p(1-p)J(A^*)\sqrt{T} \left(\gammafn{2} \left(\conesetB,\norm{\cdot}_F\right)+\sqrt{\log (1/\delta)}\right)\\
&+\bar{c}_1 \left(\bar{F}_1+\bar{V}_1\log (1/\delta)\right) +p(1-p)^2 \bar{c}_2\left(\bar{F}_1+\bar{V}_1\log (1/\delta)\right),
\end{align*}
where $\bar{c}_1,\bar{c}_2,$ and, $c_3$ are  constants.

A simple calculation shows that
\begin{align*}
&\norm{\left(G \odot (I_{T-1} \otimes (A^\top A))\right)^{1/2}}_F^2 \le J^2(A^*) T \; \tag{see $d_F(.)$ bounds in Lemma~\ref{lem:M3-expec}}, \\
&\norm{\left(G \odot \left(I_{T-1} \otimes \diag(A^\top A)\right)\right)^{1/2}}_F^2 \le J^2(A^*) T \; \tag{see $d_F(.)$ bounds in Lemma~\ref{lem:M4-expec}},\\ &\norm{\left(I_{T-1} \otimes A\right)\Gamma}_F^2 \le J^2(A^*) T, \quad \text{and}, \\
&\norm{\left(I_{T-1} \otimes \diag^{1/2}\left(A^\top A\right)\right)\Gamma}_F^2 \le J^2(A^*)T.
\end{align*}
Therefore, we have
\begin{align}\label{eq:norm1-unif-bound-PGD}
\norm{AZ}_{F}^2 &\le pJ^2(A^*)T+\bar{c}_1 \left(\bar{F}_1+\bar{V}_1\log (1/\delta)\right)+p(1-p) \bar{c}_2 \left(\bar{F}_1+\bar{V}_1\log (1/\delta)\right)\\
&+c_3 pJ(A^*)\sqrt{T} \left(\gammafn{2} \left(\conesetB,\norm{\cdot}_F\right)+\sqrt{\log (1/\delta)}\right),
\end{align}
and
\begin{align}\label{eq:norm2-unif-bound-PGD}
&(1-p)\norm{A \diag^{1/2}(ZZ^\top)}_{F}^2\\
&\le p(1-p) J^2(A^*) T+\bar{c}_1 \left(\bar{F}_1+\bar{V}_1\log (1/\delta)\right) +p(1-p)^2 \bar{c}_2\left(\bar{F}_1+\bar{V}_1\log (1/\delta)\right)\\
&+c_3 p(1-p)J(A^*)\sqrt{T} \left(\gammafn{2} \left(\conesetB,\norm{\cdot}_F\right)+\sqrt{\log (1/\delta)}\right).
\end{align}

It remains to bound $2 \dotprod{\widetilde{Z}-A^*Z}{(\widehat{A}-B)Z}+2(1-p)\dotprod{A^*\diag^{1/2}(ZZ^\top)}{(\widehat{A}-B)\diag^{1/2}(ZZ^\top)}$. To do this, we first show that it can be written as quadratics of centered 1-subgaussian random variables, as seen in Claims~\ref{clm:quad-dotprod-direct} and ~\ref{clm:quad-dotprod-diag}. 

\begin{claim}\label{clm:quad-dotprod-direct} Let $\bar{\eta} = \vect(\matrixenc{\eta_1 & \eta_2 & \dots & \eta_T})$, and 
\[
\breve{P} = \matrixenc{P_1 & 0 & \dots & 0\\ 0 & P_2 & \dots & 0\\\vdots & \vdots & \dots & \vdots\\ 0 & 0 & \dots & P_{T-1}}, 
\widetilde{P} = \matrixenc{P_2 & 0 & \dots & 0\\ 0 & P_3 & \dots & 0\\\vdots & \vdots & \dots & \vdots\\ 0 & 0 & \dots & P_{T}}, \text{ and } \optA = \frac{\widehat{A}-B}{\norm{\widehat{A}-B}_F}.
\] Then, it holds that
\begin{align*}
\dotprod{\widetilde{Z}-A^*Z}{\left(\widehat{A}-B\right)Z} & \\
= \norm{\widehat{A}-B}_F\bar{\eta}^\top \bigg(&   \matrixenc{\Gamma^\top \left(\left((I_{T-1} \otimes {A^*}^\top) \widetilde{P}-\breve{P} (I_{T-1} \otimes {A^*}^\top)\right) (I_{T-1} \otimes \optA) \breve{P} \right) \Gamma & 0\\ 0 & 0}\\
&+\matrixenc{0 & 0\\
\widetilde{P}(I_{T-1} \otimes \optA) \breve{P} \Gamma & 0}\bigg) \bar{\eta}.
\end{align*}
\end{claim}

\begin{claim}\label{clm:quad-dotprod-diag}
Let 
\begin{equation*}
    \breve{P} = \matrixenc{P_1 & 0 & \dots & 0\\ 0 & P_2 & \dots & 0\\\vdots & \vdots & \dots & \vdots\\ 0 & 0 & \dots & P_{T-1}}, \quad \text{and } \optA = \frac{\widehat{A}-B}{\norm{\widehat{A}-B}_F}.
\end{equation*}
Then, it holds that
\begin{equation*}
\dotprod{A^* \diag^{1/2}(ZZ^\top)}{(\widehat{A}-B)\diag^{1/2}(ZZ^\top)} = \norm{\widehat{A}-B}_F 
\bar{\eta}^\top \matrixenc{\Gamma^\top \breve{P} \left(I_{T-1} \otimes \diag\left(\optA^\top A^*\right)\right) \Gamma & 0 \\ 0 & 0} \bar{\eta},
\end{equation*}
\end{claim}
The proofs of Claims~\ref{clm:quad-dotprod-direct} and ~\ref{clm:quad-dotprod-diag} are deferred to Appendixes~\ref{appsec:proof-clm-quad-dotprod-direct} and ~\ref{appsec:proof-clm-quad-dotprod-diag}.
 
We are now ready to bound the inner product terms with Claims~\ref{clm:quad-dotprod-direct} and ~\ref{clm:quad-dotprod-diag}. For any $A \in \conesetB$, let 

\begin{align*}
R_1(A) &= \matrixenc{\Gamma^\top \left(\left((I_{T-1} \otimes {A^*}^\top) \widetilde{P}-\breve{P} (I_{T-1} \otimes {A^*}^\top)\right) (I_{T-1} \otimes A) \breve{P} \right) \Gamma & 0\\ 0 & 0},\\
R_2(A) &= \matrixenc{0 & 0\\
\widetilde{P}(I_{T-1} \otimes A) \breve{P} \Gamma & 0},\text{ and}\\
R_3(A) &=(1-p) \matrixenc{\Gamma^\top \breve{P} \left(I_{T-1} \otimes \diag\left(A^\top A^*\right)\right) \Gamma & 0 \\ 0 & 0}.
\end{align*} 
Then, we have
\begin{align*}
&\dotprod{\widetilde{Z}-A^*Z}{AZ}+(1-p)\dotprod{A^* \diag^{1/2}(ZZ^\top)}{A\diag^{1/2}(ZZ^\top)}\\
&=\bar{\eta}^\top \matrixenc{\Gamma^\top \left(\left((I_{T-1} \otimes {A^*}^\top) \widetilde{P}-\breve{P} (I_{T-1} \otimes {A^*}^\top)\right) (I_{T-1} \otimes A) \breve{P} \right) \Gamma & 0\\ 0 & 0}\bar{\eta}\\
&+\eta^\top \bigg( \matrixenc{0 & 0\\
\widetilde{P}(I_{T-1} \otimes A) \breve{P} \Gamma & 0}+(1-p)\matrixenc{\Gamma^\top \breve{P} \left(I_{T-1} \otimes \diag\left(A^\top A^*\right)\right) \Gamma & 0 \\ 0 & 0}\bigg) \bar{\eta}\\
&=\bar{\eta}^\top \bigg(R_1(A)+R_2(A)+R_3(A)\bigg) \bar{\eta},
\end{align*}
where $R_1(A), R_2(A)$, and $R_3(A)$ are random matrices depending on the random variables $(P_t)_{t=1}^T$. Conditional on $(P_t)_{t=1}^T$, let 
$$\calR = \set{R_1(A)+R_2(A)+R_3(A): \  A \in \conesetB}.$$ 
Then, we can show the following.
\begin{lemma}\label{lem:dotprod-fixedP}
Let $\calR = \set{R_1(A)+R_2(A)+R_3(A): \ A \in \conesetB}$. There exist constants $c,\bar{c}_1,\bar{c}_2$ such that for any $u >0$, conditional on $(P_t)_{t=1}^T$, with probability at least $1-e^{-u}$,
\begin{align*}
&\sup_{R \in \calR} \absenc{\bar{\eta}^\top R \bar{\eta}} -\sup_{R \in \calR} \absenc{\Tr \left( R\right)}\\
&\le c\left((2-p)J^2(A^*)\norm{A^*}_2+J(A^*)\right)\bigg(\bar{c}_1 \gammafn{1}(\conesetB,\norm{\cdot}_2)+\bar{c}_2 \gammafn{2}(\conesetB,\norm{\cdot}_F)\sqrt{T}+\sqrt{u} \sqrt{T}+u\bigg).
\end{align*}
\end{lemma}
The proof of Lemma~\ref{lem:dotprod-fixedP} follows from observing that the entries of $\eta$ are zero mean and $L$-subgaussian, and  applying Theorem~\ref{thm:non-psd-sup} (proof deferred to Appendix~\ref{appsec:proof-lem-dotprod-fixedP}). 

For a suitably large constant $c> 1$, let 
\begin{equation*}
    L_1=c\left((2-p)J^2(A^*)\norm{A^*}_2+J(A^*)\right) \bigg(\bar{c}_1 \gammafn{1}(\conesetB,\norm{\cdot}_2)+\bar{c}_2 \gammafn{2}(\conesetB,\norm{\cdot}_F)\sqrt{T}+\sqrt{\log (1/\delta)} \sqrt{T}+\log (1/\delta)\bigg),
\end{equation*}
and denote the event
\begin{equation*}
    \eventp{1d} = \set{\forall R \in \calR: \ \absenc{\bar{\eta}^\top R \bar{\eta}} \le \sup_{R \in \calR} \absenc{\Tr \left( R\right)}+L_1}.
\end{equation*}
Setting $u=\log(1/\delta)$ in Lemma~\ref{lem:dotprod-fixedP} we can see that
$\prob\left(\eventp{1d}\vert (P_t)_{t=1}^T\right) \ge 1-\delta$.

Now, consider $\sup_{R \in \calR} \absenc{\Tr(R)} = \sup_{A \in \conesetB} \absenc{\Tr \left(R_1(A)+R_2(A)+R_3(A)\right)}$. Since $\Tr(R_2(A))=0$ for all $A \in \conesetB$, we have $\sup_{R \in \calR} \absenc{\Tr(R)} = \sup_{A \in \conesetB} \absenc{\Tr \left(R_1(A)\right)+\Tr \left(R_3(A)\right)}$. For any $A \in \conesetB$, 
\begin{align*}
\Tr \left(R_1(A)\right)+\Tr \left(R_3(A)\right) & = \Tr \left( \Gamma^\top \left(\left((I_{T-1} \otimes {A^*}^\top) \widetilde{P}-\breve{P} (I_{T-1} \otimes {A^*}^\top)\right) (I_{T-1} \otimes A) \breve{P} \right) \Gamma\right)\\
&+(1-p)\Tr \left(\Gamma^\top \breve{P} \left(I_{T-1} \otimes \diag\left(A^\top A^*\right)\right) \Gamma\right)\\
& = \Tr \left( \Gamma^\top \left((I_{T-1} \otimes {A^*}^\top) \widetilde{P} (I_{T-1} \otimes A) \breve{P} \right) \Gamma\right)\\
& - \Tr \left( \Gamma^\top \left(\breve{P} \left(I_{T-1} \otimes \left({A^*}^\top A\right)\right) \breve{P} \right) \Gamma\right)\\
&+(1-p)\Tr \left(\Gamma^\top \breve{P} \left(I_{T-1} \otimes \diag\left(A^\top A^*\right)\right) \breve{P} \Gamma\right).
\end{align*}
Let $\overline{P}_1 = \breve{P}-pI$ and $\overline{P}_2 = \widetilde{P}-pI$. Then replacing $\breve{P}, \widetilde{P}$ with $\overline{P}_1+pI,\overline{P}_2+pI$ respectively and expanding the terms gives us,
\begin{align*}
\Tr \left(R_1(A)\right)+\Tr \left(R_3(A)\right) & = \Tr \left( \Gamma^\top \left((I_{T-1} \otimes {A^*}^\top) \overline{P}_2 (I_{T-1} \otimes A) \overline{P}_1 \right) \Gamma\right) - \Tr \left( \Gamma^\top \left(\overline{P}_1 \left(I_{T-1} \otimes \left({A^*}^\top A\right)\right) \overline{P}_1 \right) \Gamma\right)\\
&+(1-p)\Tr \left(\Gamma^\top \overline{P}_1 \left(I_{T-1} \otimes \diag\left(A^\top A^*\right)\right)\overline{P}_1 \Gamma\right)+p^2(1-p)\Tr \left(\Gamma^\top \left(I_{T-1} \otimes \diag\left(A^\top A^*\right)\right)  \Gamma\right)\\
& + p\Tr \left( \Gamma^\top \left((I_{T-1} \otimes {A^*}^\top) \overline{P}_2 (I_{T-1} \otimes A) \right) \Gamma\right)+2p(1-p)\Tr \left(\Gamma^\top \overline{P}_1 \left(I_{T-1} \otimes \diag\left(A^\top A^*\right)\right)  \Gamma\right)\\
& - p\Tr \left( \Gamma^\top \left(\overline{P}_1 \left(I_{T-1} \otimes \left({A^*}^\top A\right)\right)  \right) \Gamma\right).
\end{align*}

Therefore, for any $A \in \conesetB$, we have
\begin{align*}
\absenc{\Tr \left(R_1(A)\right)+\Tr \left(R_3(A)\right)} & \le \bigg |\Tr \left( \Gamma^\top \left((I_{T-1} \otimes {A^*}^\top) \overline{P}_2 (I_{T-1} \otimes A) \overline{P}_1 \right) \Gamma\right) - \Tr \left( \Gamma^\top \left(\overline{P}_1 \left(I_{T-1} \otimes \left({A^*}^\top A\right)\right) \overline{P}_1 \right) \Gamma\right)\\
&+(1-p)\Tr \left(\Gamma^\top \overline{P}_1 \left(I_{T-1} \otimes \diag\left(A^\top A^*\right)\right)\overline{P}_1 \Gamma\right)+p^2(1-p)\Tr \left(\Gamma^\top \left(I_{T-1} \otimes \diag\left(A^\top A^*\right)\right)  \Gamma\right) \bigg | \\
& + \bigg| p\Tr \left( \Gamma^\top \left((I_{T-1} \otimes {A^*}^\top) \overline{P}_2 (I_{T-1} \otimes A) \right) \Gamma\right)+2p(1-p)\Tr \left(\Gamma^\top \overline{P}_1 \left(I_{T-1} \otimes \diag\left(A^\top A^*\right)\right)  \Gamma\right)\\
& - p\Tr \left( \Gamma^\top \left(\overline{P}_1 \left(I_{T-1} \otimes \left({A^*}^\top A\right)\right)  \right) \Gamma\right)\bigg|.
\end{align*}

Let us define 
\begin{align*}
f_1(A) &= \Tr \left( \Gamma^\top \left((I_{T-1} \otimes {A^*}^\top) \overline{P}_2 (I_{T-1} \otimes A) \overline{P}_1 \right) \Gamma\right) - \Tr \left( \Gamma^\top \left(\overline{P}_1 \left(I_{T-1} \otimes \left({A^*}^\top A\right)\right) \overline{P}_1 \right) \Gamma\right)\\
&+(1-p)\Tr \left(\Gamma^\top \overline{P}_1 \left(I_{T-1} \otimes \diag\left(A^\top A^*\right)\right)\overline{P}_1 \Gamma\right).
\end{align*}
Then, we can see that for any  $A \in \conesetB$,
\begin{align*}
\expec f_1(A) &= \expec \Tr \left( \Gamma^\top \left((I_{T-1} \otimes {A^*}^\top) \overline{P}_2 (I_{T-1} \otimes A) \overline{P}_1 \right) \Gamma\right) - \expec \Tr \left( \Gamma^\top \left(\overline{P}_1 \left(I_{T-1} \otimes \left({A^*}^\top A\right)\right) \overline{P}_1 \right) \Gamma\right)\\
&+(1-p)\expec \Tr \left(\Gamma^\top \overline{P}_1 \left(I_{T-1} \otimes \diag\left(A^\top A^*\right)\right)\overline{P}_1 \Gamma\right)\\
&=-p(1-p)\Tr \left(\Gamma^\top \left(I_{T-1} \otimes \diag\left(A^\top A^*\right)\right)  \Gamma\right)+p(1-p)^2\Tr \left(\Gamma^\top \left(I_{T-1} \otimes \diag\left(A^\top A^*\right)\right)  \Gamma\right)\\
&=-p^2(1-p)\Tr \left(\Gamma^\top \left(I_{T-1} \otimes \diag\left(A^\top A^*\right)\right)  \Gamma\right).
\end{align*}
which implies,
\begin{align*}
\absenc{f_1(A)-\expec f_1(A)} &= \bigg |\Tr \left( \Gamma^\top \left((I_{T-1} \otimes {A^*}^\top) \overline{P}_2 (I_{T-1} \otimes A) \overline{P}_1 \right) \Gamma\right) - \Tr \left( \Gamma^\top \left(\overline{P}_1 \left(I_{T-1} \otimes \left({A^*}^\top A\right)\right) \overline{P}_1 \right) \Gamma\right)\\
&+(1-p)\Tr \left(\Gamma^\top \overline{P}_1 \left(I_{T-1} \otimes \diag\left(A^\top A^*\right)\right)\overline{P}_1 \Gamma\right)+p^2(1-p)\Tr \left(\Gamma^\top \left(I_{T-1} \otimes \diag\left(A^\top A^*\right)\right)  \Gamma\right) \bigg |. 
\end{align*}

Recall the definitions of $q$ and $G$ from the proof of Lemma~\ref{lem:LHS-bound}. We can see that, following the parameterizations in the proof of Lemma~\ref{lem:LHS-bound} to show that,
\begin{align*}
\Tr \left( \Gamma^\top \left(\overline{P}_1 \left(I_{T-1} \otimes \left({A^*}^\top A\right)\right) \overline{P}_1 \right) \Gamma\right) &= \dotprod{\Gamma \Gamma^\top}{\overline{P}_1\left(I_{T-1} \otimes \left({A^*}^\top A\right)\right)\overline{P}_1}\\
&=\sum_{t=1}^{T-1} \dotprod{G_t}{(\overline{P}_1)_t {A^*}^\top A (\overline{P}_1)_t}\\
&=\sum_{t=1}^{T-1} q_t^\top \left(G_t \odot {A^*}^\top A\right)q_t \\
& = q^\top \left(G \odot \left(I_{T-1} \otimes \left({A^*}^\top A\right)\right)\right) q,
\end{align*}
and,
\begin{align*}
\Tr \left(\Gamma^\top \overline{P}_1 \left(I_{T-1} \otimes \diag\left(A^\top A^*\right)\right)\overline{P}_1 \Gamma\right) &= \dotprod{\Gamma \Gamma^\top}{\overline{P}_1\left(I_{T-1} \otimes \diag\left(A^\top A^*\right)\right)\overline{P}_1}\\
&=\sum_{t=1}^{T-1} \dotprod{G_t}{(\overline{P}_1)_t \diag\left(A^\top A^*\right) (\overline{P}_1)_t}\\
&=\sum_{t=1}^{T-1} q_t^\top \left(G_t \odot \diag\left(A^\top A^*\right)\right)q_t \\
& = q^\top \left(G \odot \left(I_{T-1} \otimes \diag\left(A^\top A^*\right)\right)\right) q.
\end{align*}
Let $\widetilde{q} = \vect\left(\matrixenc{q_2 & q_3 & \dots & q_T}\right)$. Then, we can also see that, 
\begin{align*}
\Tr \left( \Gamma^\top \left((I_{T-1} \otimes {A^*}^\top) \overline{P}_2 (I_{T-1} \otimes A) \overline{P}_1 \right) \Gamma\right) &= \dotprod{(I_{T-1} \otimes {A^*}) \Gamma \Gamma^\top}{\overline{P}_2 (I_{T-1} \otimes A) \overline{P}_1}\\
&=\sum_{t=1}^{T-1} \dotprod{A^* G_t}{(\overline{P}_2)_t A (\overline{P}_1)_t}\\
&=\sum_{t=1}^{T-1} q_{t+1}^\top \left((A^*G_t) \odot A\right)q_t \\
& = \widetilde{q}^\top \left(\left((I_{T-1} \otimes A^*)G\right) \odot \left(I_{T-1} \otimes A\right)\right) q.
\end{align*}

Now let $\bar{q} = \vect\left(\matrixenc{q_1 & q_2 & q_3 & \dots & q_T}\right)$, and,
\begin{align*}
&S_1(A) = \matrixenc{0_{n \times n(T-1)} & 0_{n \times n} \\ \left((I_{T-1} \otimes A^*)G\right) \odot \left(I_{T-1} \otimes A\right) & 0_{n(T-1) \times n} },\\
&S_2(A) = (1-p)\matrixenc{G \odot \left(I_{T-1} \otimes \diag\left(A^\top A^*\right)\right) & 0_{n(T-1) \times n} \\ 0_{n \times n(T-1)} & 0_{n \times n}},\\
&S_3(A) = -\matrixenc{G \odot \left(I_{T-1} \otimes \left({A^*}^\top A\right)\right) & 0_{n(T-1) \times n} \\ 0_{n \times n(T-1)} & 0_{n \times n}}.
\end{align*}
Thus, for any $A \in \conesetB$, $f_1(A) = \bar{q}^\top \left(S_1(A)+S_2(A)+S_3(A)\right)\bar{q}$. We can show the following Lemma.

\begin{lemma}\label{lem:trace-uniform-P}
Let $\calS = \set{S_1(A)+S_2(A)+S_3(A): A \in \conesetB}$. There exist constants $\bar{c}_1,\bar{c}_2$ such that for any $u >0$, with probability at least $1-e^{-u}$,
\begin{equation*}
\sup_{S \in \calS} \absenc{\bar{q}^\top S \bar{q}-\expec \bar{q}^\top S \bar{q}} \le \left((3-p)J^2(A^*)\norm{A^*}_2\right) \left(\bar{c}_1 \gammafn{1} (\conesetB, \norm{\cdot}_2)+\bar{c}_2\gammafn{2} (\conesetB, \norm{\cdot}_F)\sqrt{T}+u+\sqrt{u}\sqrt{T}\right).
\end{equation*}
\end{lemma}
The proof of Lemma~\ref{lem:trace-uniform-P} also follows from observing the entries of $\bar{q}$ are zero mean and 1-subgaussian, and applying Theorem~\ref{thm:non-psd-sup} (proof deferred to Appendix~\ref{appsec:proof-lem-trace-uniform-P}). Setting $u = \log(1/\delta)$ and denoting
\begin{align*}
L_2&=\left((3-p)J^2(A^*)\norm{A^*}_2\right) \bigg(\bar{c}_1 \gammafn{1} (\conesetB, \norm{\cdot}_2)+\bar{c}_2\gammafn{2} (\conesetB, \norm{\cdot}_F)\sqrt{T}+\log (1/\delta)+\sqrt{\log (1/\delta)}\sqrt{T}\bigg), \\
\eventp{2d} &= \set{\forall S \in \calS: \ \absenc{\bar{q}^\top S \bar{q}-\expec \bar{q}^\top S \bar{q}}\le L_2}.
\end{align*}
we see that Lemma~\ref{lem:trace-uniform-P} implies $\prob\left(\eventp{2d}\right) \ge 1-\delta$.

Now consider 
\begin{align*}
f_2(A) &= \bigg | p\Tr \left( \Gamma^\top \left((I_{T-1} \otimes {A^*}^\top) \overline{P}_2 (I_{T-1} \otimes A) \right) \Gamma\right)+2p(1-p)\Tr \left(\Gamma^\top \overline{P}_1 \left(I_{T-1} \otimes \diag\left(A^\top A^*\right)\right)  \Gamma\right)\\
&- p\Tr \left( \Gamma^\top \left(\overline{P}_1 \left(I_{T-1} \otimes \left({A^*}^\top A\right)\right)  \right) \Gamma\right) \bigg |.
\end{align*}
We can bound $f_2(A)$ uniformly over $A \in \conesetB$ as follows.
\begin{lemma}\label{lem:dotprod-unifP}
For any $\delta \in (0,1)$, with probability at least $1-\delta$, for any $A \in \conesetB$,
\begin{align*}
f_2(A) \le \ \widetilde{c}p(2-p)J^2(A^*)\norm{A^*}_2 \sqrt{T} \left( \gammafn{2}(\conesetB,\norm{\cdot}_F)+\sqrt{\log (1/\delta)}\right).
\end{align*}
for some constant $\widetilde{c} > 1$.
\end{lemma}
The proof of Lemma~\ref{lem:dotprod-unifP} follows from an application of Theorem~\ref{thm:bounded-func-conc}. A detailed proof is available in Appendix~\ref{appsec:proof-lem-dotprod-unifP}.

Denoting
\[
\eventp{3d} = \set{\forall A \in \conesetB: \ f_2(A) \le \widetilde{c}p(2-p)J^2(A^*)\norm{A^*}_2 \sqrt{T} \left( \gammafn{2}(\conesetB,\norm{\cdot}_F)+\sqrt{\log (1/\delta)}\right)},
\]
 we can see that from Lemma~\ref{lem:dotprod-unifP}, $\prob \left(\eventp{3d}\right) \ge 1-\delta$.

On the event $\eventp{1d} \cap \eventp{2d} \cap \eventp{3d}$, we have for any $R \in \calR$,
\begin{align*}
\bar{\eta}^\top R \bar{\eta} &\le c(2-p)J^2(A^*)\norm{A^*}_2 \bigg(\bar{c}_1 \gammafn{1}(\conesetB,\norm{\cdot}_2)+\bar{c}_2 \gammafn{2}(\conesetB,\norm{\cdot}_F)\sqrt{T}+\sqrt{\log (1/\delta)} \sqrt{T}+\log (1/\delta)\bigg)\\
&+cJ(A^*) \bigg(\bar{c}_1 \gammafn{1}(\conesetB,\norm{\cdot}_2)+\bar{c}_2 \gammafn{2}(\conesetB,\norm{\cdot}_F)\sqrt{T}+\sqrt{\log (1/\delta)} \sqrt{T}+\log (1/\delta)\bigg)\\
&+\left((3-p)J^2(A^*)\norm{A^*}_2\right) \bigg(\bar{c}_1 \gammafn{1} (\conesetB, \norm{\cdot}_2)+\bar{c}_2\gammafn{2} (\conesetB, \norm{\cdot}_F)\sqrt{T}+\log (1/\delta)+\sqrt{\log (1/\delta)}\sqrt{T}\bigg)\\
&+\widetilde{c}2p(2-p)J^2(A^*)\norm{A^*}_2 \sqrt{T} \left( \gammafn{2}(\conesetB,\norm{\cdot}_F)+\sqrt{\log (1/\delta)}\right)\\
&\le \widetilde{c}\left(J^2(A^*)\norm{A^*}_2+J(A^*)\right)  \left(\gammafn{1}(\conesetB,\norm{\cdot}_2)+\gammafn{2}(\conesetB,\norm{\cdot}_F)\sqrt{T}+\sqrt{\log (1/\delta)} \sqrt{T}+\log (1/\delta)\right).
\end{align*}
This implies,
\begin{align}
& \dotprod{\widetilde{Z}-A^*Z}{AZ}+(1-p)\dotprod{A^* \diag^{1/2}(ZZ^\top)}{A\diag^{1/2}(ZZ^\top)} \notag \\
& \le c_1\left(J^2(A^*)\norm{A^*}_2+J(A^*)\right) \bigg(\gammafn{1}(\conesetB,\norm{\cdot}_2)+\gammafn{2}(\conesetB,\norm{\cdot}_F)\sqrt{T}+\sqrt{T\log (1/\delta)}+\log (1/\delta)\bigg) \notag \\
& \le \widetilde{c}_1J^3(A^*) \bigg(\gammafn{1}(\conesetB,\norm{\cdot}_2)+\gammafn{2}(\conesetB,\norm{\cdot}_F)\sqrt{T}+\sqrt{T\log (1/\delta)}+\log (1/\delta)\bigg). \label{eq:dotprod-PGD-unif-bound}
\end{align}
for some constant $\widetilde{c}_1$ (where the last inequality uses the fact that $J(A^*)\ge 1+\norm{A^*}_2$).

Therefore with probability at least $1-\delta$,
\begin{align*}
&2\dotprod{\widetilde{Z}-A^*Z}{(\widehat{A}-B)Z} +2(1-p)\dotprod{A^* \diag^{1/2}(ZZ^\top)}{(\widehat{A}-B)\diag^{1/2}(ZZ^\top)}\\
&+2(1-p)\norm{(A^*-B)\diag^{1/2}(ZZ^\top)}_F \norm{(\widehat{A}-B)\diag^{1/2}(ZZ^\top)}_F+2\norm{\left(A^*-B\right)Z}_F\norm{(\widehat{A}-B)Z}_F \\
& \le  c_1 J^3(A^*) \norm{\widehat{A}-B}_F \bigg(\gammafn{1}(\conesetB,\norm{\cdot}_2)+\gammafn{2}(\conesetB,\norm{\cdot}_F)\sqrt{T}+\sqrt{T\log \left(1/\delta\right)}+\log (1/\delta)\bigg)\\
&+c_2J^2(A^*) \norm{\widehat{A}-B}_F\norm{A^*-B}_F \bigg(p T+\bar{L}_1 +p(1-p)^2 \bar{L}_1+p\sqrt{T} \left(\gammafn{2} \bigg(\conesetB,\norm{\cdot}_F\right)+\sqrt{\log (1/\delta)}\bigg)\bigg),
\end{align*}
where 
\begin{align*}
\bar{L}_1 &= C \left(\gammafn{2}^2(\conesetB,\norm{\cdot}_2) +\gammafn{2}(\conesetB,\norm{\cdot}_2) \sqrt{T}+\sqrt{T} \log \left(1/\delta\right)\right),
\end{align*}
and $C,c_1,c_2$ are constants.
\end{proof}

\subsection{Proof of Claim~\ref{clm:quad-dotprod-direct}}\label{appsec:proof-clm-quad-dotprod-direct}

\begin{proof}
Since $x_{t+1}=A^*x_t+\eta_{t+1}$, and, $y_t=\smpl_t x_t$, we can see that, 
\begin{align*}
\left(\widetilde{Z}-A^*Z\right)_{t}
&=\smpl_{t+1}y_{t+1}-A^*\smpl_{t} y_t \\&= \smpl_{t+1} x_{t+1}- \smpl_{t} x_{t}\\ 
&= \left(P_{t+1}A^*-A^*P_t\right)x_t+P_{t+1} \eta_{t+1}.
\end{align*}

Let $R=\matrixenc{\left(P_{2}A^*-A^*P_1\right)x_1 & \left(P_{3}A^*-A^*P_2\right)x_2&\dots&\left(P_{T}A^*-A^*P_{T-1}\right)x_{T-1}}$ and $E_{proj} = \matrixenc{ P_2 \eta_2 & P_3 \eta_3 & \dots & P_T \eta_T}$. Then we can see that,
$\widetilde{Z}-A^*Z=R+E_{proj}$. 

Note that $Z = \matrixenc{ P_1 x_1 & P_2 x_2 & \dots & P_{T-1}x_{T-1}}$. Therefore,
\begin{align*}
&\dotprod{\widetilde{Z}-A^*Z}{\left(\widehat{A}-B\right)Z} \\
&= \dotprod{R}{\left(\widehat{A}-B\right)Z}+\dotprod{E_{proj}}{\left(\widehat{A}-B\right)Z}\\
& =\Tr \left(\left(\widehat{A}-B\right)ZR^\top\right)+\Tr \left(\left(\widehat{A}-B\right)ZE_{proj}^\top\right)\\
&=\sum_{t=1}^{T-1} \Tr\left(\left(\widehat{A}-B\right)P_t x_t x_t^\top ({A^*}^\top P_{t+1}-P_t {A^*}^\top)\right)+\sum_{t=1}^{T-1} \Tr\left(\left(\widehat{A}-B\right)P_t x_t \eta_{t+1}^\top P_{t+1}\right)\\
&=\sum_{t=1}^{T-1} x_t^\top ({A^*}^\top P_{t+1}-P_t {A^*}^\top)\left(\widehat{A}-B\right)P_t x_t +\sum_{t=1}^{T-1} \eta_{t+1}^\top P_{t+1}\left(\widehat{A}-B\right)P_t x_t.
\end{align*}

We can use the definitions of the $x,\eta,\Gamma,\widetilde{P}$, and $\breve{P}$, to show that for $\sum_{t=1}^{T-1} x_t^\top ({A^*}^\top P_{t+1}-P_t {A^*}^\top)\left(\widehat{A}-B\right)P_t x_t$,
\begin{align*}
\sum_{t=1}^{T-1} x_t^\top ({A^*}^\top P_{t+1}-P_t {A^*}^\top)\left(\widehat{A}-B\right)P_t x_t &= \norm{\widehat{A}-B}_F \sum_{t=1}^{T-1} x_t^\top ({A^*}^\top P_{t+1}-P_t {A^*}^\top)\optA P_t x_t
\end{align*}
Note that,
\begin{align*}
&\sum_{t=1}^{T-1} x_t^\top ({A^*}^\top P_{t+1}-P_t {A^*}^\top)A P_t x_t\\
&= x^\top \matrixenc{
({A^*}^\top P_{2}-P_1 {A^*}^\top)\optA P_1 & 0 & \dots & 0\\
0 & ({A^*}^\top P_{3}-P_2 {A^*}^\top)\optA P_2 & \dots & 0 \\
\vdots & \vdots & \dots & \vdots\\
0 & 0 & \dots & ({A^*}^\top P_{T}-P_{T-1} {A^*}^\top)A P_{T-1}} x \\
&=x^\top \left(\left((I_{T-1} \otimes {A^*}^\top) \widetilde{P}-\breve{P} (I_{T-1} \otimes {A^*}^\top)\right) (I_{T-1} \otimes \optA) \breve{P} \right) x\\
&=\eta^\top \Gamma^\top \left(\left((I_{T-1} \otimes {A^*}^\top) \widetilde{P}-\breve{P} (I_{T-1} \otimes {A^*}^\top)\right) (I_{T-1} \otimes \optA) \breve{P} \right) \Gamma \eta
\end{align*}
and therefore,
\begin{align*}
&\sum_{t=1}^{T-1} x_t^\top ({A^*}^\top P_{t+1}-P_t {A^*}^\top)\left(\widehat{A}-B\right)P_t x_t \\
&= \norm{\widehat{A}-B}_F \eta^\top \Gamma^\top \left(\left((I_{T-1} \otimes {A^*}^\top) \widetilde{P}-\breve{P} (I_{T-1} \otimes {A^*}^\top)\right) (I_{T-1} \otimes \optA) \breve{P} \right) \Gamma \eta\\
&=\norm{\widehat{A}-B}_F\bar{\eta}^\top \matrixenc{\Gamma^\top \left(\left((I_{T-1} \otimes {A^*}^\top) \widetilde{P}-\breve{P} (I_{T-1} \otimes {A^*}^\top)\right) (I_{T-1} \otimes \optA) \breve{P} \right) \Gamma & 0\\ 0 & 0}\bar{\eta}
\end{align*}
Let $\widetilde{\eta} = \vect\left(\matrixenc{\eta_2 & \eta_3 & \dots & \eta_T}\right)$. Then,
\begin{align*}
\sum_{t=1}^{T-1} \eta_{t+1}^\top P_{t+1}\left(\widehat{A}-B\right)P_t x_t &= \norm{\widehat{A}-B}_F \sum_{t=1}^{T-1} \eta_{t+1}^\top P_{t+1}\optA P_t x_t= \norm{\widehat{A}-B}_F \widetilde{\eta}^\top \widetilde{P}(I_{T-1} \otimes \optA) \breve{P} x\\
&= \norm{\widehat{A}-B}_F \widetilde{\eta}^\top \widetilde{P}(I_{T-1} \otimes \optA) \breve{P} \Gamma \eta= \norm{\widehat{A}-B}_F \bar{\eta}^\top \matrixenc{0 & 0\\
\widetilde{P}(I_{T-1} \otimes \optA) \breve{P} \Gamma & 0} \bar{\eta}.
\end{align*}
\end{proof}

\subsection{Proof of Claim~\ref{clm:quad-dotprod-diag}}\label{appsec:proof-clm-quad-dotprod-diag}

\begin{proof}
Let $\optA=\frac{\widehat{A}-B}{\norm{\widehat{A}-B}_F}$. Note that,
\begin{align*}
\dotprod{A^*  \diag^{1/2}(ZZ^\top)}{(\widehat{A}-B)\diag^{1/2}(ZZ^\top)}
&=\Tr \left((\widehat{A}-B)^\top A^*  \diag(ZZ^\top)\right)\\
&=\norm{\widehat{A}-B}\Tr \left(A^\top A^*  \diag(ZZ^\top)\right)\\
&=\norm{\widehat{A}-B}\dotprod{{A^*}^\top \optA}{\diag(ZZ^\top)}\\ 
& = \norm{\widehat{A}-B}_F\dotprod{\diag\left({A^*}^\top \optA\right)}{\diag\left(ZZ^\top\right)}\\
& = \norm{\widehat{A}-B}_F\sum_{i=1}^n \left(\diag\left({A^*}^\top \optA\right)\right)_{ii} \left({\diag\left(ZZ^\top\right)}\right)_{ii}\\
& = \norm{\widehat{A}-B}_F\sum_{i=1}^n \left(\diag\left({A^*}^\top \optA\right)\right)_{ii} \sum_{t=1}^{T-1} Z_{it}^2\\
& = \norm{\widehat{A}-B}_F\sum_{t=1}^{T-1} \sum_{i=1}^n \left(\diag\left({A^*}^\top \optA\right)\right)_{ii}  Z_{it}^2\\
& = \norm{\widehat{A}-B}_F\sum_{t=1}^{T-1} \left(Z_{:,t}\right)^\top \diag\left({A^*}^\top \optA\right) Z_{:,t}\\
& = \norm{\widehat{A}-B}_F \vect{(Z)}^\top \left(I_{T-1} \otimes \diag\left({A^*}^\top \optA\right)\right) \vect{(Z)}\\
&=\norm{\widehat{A}-B}_F \eta^\top \Gamma^\top \breve{P} \left(I_{T-1} \otimes \diag\left({A^*}^\top \optA\right)\right) \breve{P} \Gamma \eta\\
&=\norm{\widehat{A}-B}_F  \bar{\eta}^\top \matrixenc{\Gamma^\top \breve{P} \left(I_{T-1} \otimes \diag\left({A^*}^\top \optA\right)\right) \breve{P} \Gamma & 0 \\ 0 & 0} \bar{\eta}.
\end{align*}
\end{proof}

\subsection{Proof of Lemma~\ref{lem:dotprod-fixedP}}\label{appsec:proof-lem-dotprod-fixedP}

\begin{proof} Since $\bar{\eta}$ is zero mean and $L$-subgaussian, we can see that, $\eta_{L} = \frac{\bar{\eta}}{L}$ is 1-subgaussian. Therefore, applying Theorem~\ref{thm:non-psd-sup} on $\eta_{L}^\top R \eta_{L}$ we get that,
\begin{align*}
\prob\bigg(\sup_{R \in \calR} \absenc{\eta_{L}^\top R \eta_{L} - \expec \left[\eta_{L}^\top R \eta_{L} \bigg \vert (P_t)_{t=1}^T \right]} \ge c' \left(\gammafn{1}(\calR,\norm{\cdot}_2)+\gammafn{2}(\calR,\norm{\cdot}_F)\right)+\sqrt{u} d_F(\calR)+ud_2(\calR) \ \bigg \vert (P_t)_{t=1}^T \bigg) \le e^{-u}.
\end{align*}
Note that since each entry of $\bar{\eta}$ has unit variance, $\expec \left[\eta_{L}^\top R \eta_{L} \bigg \vert (P_t)_{t=1}^T \right] = \frac{1}{L^2}\expec \left[\bar{\eta}^\top R \bar{\eta}\right] = \frac{1}{L^2}\Tr(R)$. Therefore, conditional on $(P_t)_{t=1}^T$, with probability at least $1-e^{-u}$,
\begin{align*}
\sup_{R \in \calR} \absenc{\eta_{L}^\top R \eta_{L} - \frac{1}{L^2}\Tr \left( R\right)} \le c' \left(\gammafn{1}(\calR,\norm{\cdot}_2)+\gammafn{2}(\calR,\norm{\cdot}_F)\right)+\sqrt{u} d_F(\calR)+ud_2(\calR).
\end{align*}
which implies
\begin{align*}
\sup_{R \in \calR} \absenc{\bar{\eta}^\top R \bar{\eta}-\Tr \left( R\right)} \le c\left(c' \left(\gammafn{1}(\calR,\norm{\cdot}_2)+\gammafn{2}(\calR,\norm{\cdot}_F)\right)+\sqrt{u} d_F(\calR)+ud_2(\calR)\right).
\end{align*}
where $c$ is a constant.

Similar to the proof of Lemma~\ref{lem:M1-expec} we next show that the maps $(\conesetB,\norm{\cdot}_2) \mapsto (\calR,\norm{\cdot}_2)$ and $(\conesetB,\norm{\cdot}_F) \mapsto (\calR,\norm{\cdot}_F)$ are Lipschitz. 

\begin{claim}\label{clm:lip-map-calR}
The following statements hold a.s.
\begin{enumerate}
\item The map $(\conesetB,\norm{\cdot}_2) \mapsto (\calR,\norm{\cdot}_2)$ is $\left((2-p)J^2(A^*)\norm{A^*}_2+J(A^*)\right)$-Lipschitz, and,
\item The map $(\conesetB,\norm{\cdot}_F) \mapsto (\calR,\norm{\cdot}_F)$ is $\left((2-p)J^2(A^*)\norm{A^*}_2+J(A^*)\right)\sqrt{T}$-Lipschitz.
\end{enumerate}
\end{claim}

Hence from Property~\ref{itm:prop-3} from Appendix~\ref{sec:prelim-tools}, we get,
\begin{align*}
&\gammafn{1}\left(\calR,\norm{\cdot}_2\right) \le c_1 \left((2-p)J^2(A^*)\norm{A^*}_2+J(A^*)\right) \gammafn{1}\left(\conesetB,\norm{\cdot}_2\right),\\
&\gammafn{2}\left(\calR,\norm{\cdot}_F\right) \le c_1 \left((2-p)J^2(A^*)\norm{A^*}_2+J(A^*)\right)\sqrt{T} \gammafn{2}\left(\conesetB,\norm{\cdot}_F\right).
\end{align*}
We can also see that,
\begin{align*}
d_2(\calR) &= \sup_{R\in \calR} \norm{R}_2 = \sup_{A \in \conesetB} \norm{R_1(A)+R_2(A)+R_3(A)}_2\\ 
&\le\sup_{A \in \conesetB} \left((2-p)J^2(A^*)\norm{A^*}_2+J(A^*)\right) \norm{A}_2  \quad \text{a.s} \\
& \le \left((2-p)J^2(A^*)\norm{A^*}_2+J(A^*)\right), \;\text{and},\\
d_F(\calR) &= \sup_{R\in \calR} \norm{R}_F = \sup_{A \in \conesetB} \norm{R_1(A)+R_2(A)+R_3(A)}_F\\
& \le \sup_{A \in \conesetB}\left((2-p)J^2(A^*)\norm{A^*}_2+J(A^*)\right)\sqrt{T} \norm{A}_F \quad \text{a.s}\\
& = \left((2-p)J^2(A^*)\norm{A^*}_2+J(A^*)\right)\sqrt{T}.
\end{align*}

Therefore, conditional on $(P_t)_{t=1}^T$, with probability at least $1-e^{-u}$,
\begin{align*}
&\sup_{R \in \calR} \absenc{\bar{\eta}^\top R \bar{\eta}} -\sup_{R \in \calR} \absenc{\Tr \left( R\right)} \le \sup_{R \in \calR} \absenc{\bar{\eta}^\top R \bar{\eta}-\Tr \left( R\right)}\\
&\le c\left((2-p)J^2(A^*)\norm{A^*}_2+J(A^*)\right) \bigg(\bar{c}_1 \gammafn{1}(\conesetB,\norm{\cdot}_2)+\bar{c}_2 \gammafn{2}(\conesetB,\norm{\cdot}_F)\sqrt{T}+\sqrt{u} \sqrt{T}+u\bigg).
\end{align*}
\end{proof}

\begin{proof}[Proof of Claim~\ref{clm:lip-map-calR}]
Let $A_1,A_2 \in \conesetB$, and $\widetilde{R}_1, \widetilde{R}_2 \in \calR$ be such that $\widetilde{R}_1 = R_1(A_1)+R_2(A_1)+R_3(A_1)$ and $\widetilde{R}_2 = R_1(A_2)+R_2(A_2)+R_3(A_2)$. Then,
\begin{align*}
\norm{\widetilde{R}_1-\widetilde{R}_2}_2 & = \norm{R_1(A_1)+R_2(A_1)+R_3(A_1)-R_1(A_2)-R_2(A_2)-R_3(A_2)}_2\\
& \le \norm{R_1(A_1)-R_1(A_2)}_2+\norm{R_2(A_1)-R_2(A_2)}_2+\norm{R_3(A_1)-R_3(A_2)}_2\\
& = \norm{\Gamma^\top \left(\left((I_{T-1} \otimes {A^*}^\top) \widetilde{P}-\breve{P} (I_{T-1} \otimes {A^*}^\top)\right) (I_{T-1} \otimes (A_1-A_2)) \breve{P} \right) \Gamma}_2 \\
&+\norm{\matrixenc{0 & 0\\
\widetilde{P}(I_{T-1} \otimes (A_1-A_2)) \breve{P} \Gamma & 0}}_2+\norm{(1-p)\Gamma^\top \breve{P} \left(I_{T-1} \otimes \diag\left((A_1-A_2)^\top A^*\right)\right) \Gamma}_2\\
& \le \norm{\Gamma}_2^2\norm{\left((I_{T-1} \otimes {A^*}^\top) \widetilde{P}-\breve{P} (I_{T-1} \otimes {A^*}^\top)\right)}_2 \norm{\breve{P}}_2 \norm{A_1-A_2}_2 \\
&+\norm{\Gamma}_2 \norm{\widetilde{P}}_2 \norm{\breve{P}}_2 \norm{A_1-A_2}_2+(1-p)\norm{\Gamma}_2^2\norm{\breve{P}}_2\norm{ \diag\left((A_1-A_2)^\top A^*\right)}_2\\
& = \norm{\Gamma}_2^2\norm{\left((I_{T-1} \otimes {A^*}^\top) \widetilde{P}-\breve{P} (I_{T-1} \otimes {A^*}^\top)\right)}_2 \norm{\breve{P}}_2 \norm{A_1-A_2}_2 \\
&+\norm{\Gamma}_2 \norm{\widetilde{P}}_2 \norm{\breve{P}}_2 \norm{A_1-A_2}_2+(1-p)\norm{\Gamma}_2^2\norm{\breve{P}}_2\max_{i \in [n]} {\left((A_1-A_2)_{:,i}\right)}^\top {(A^{*})}_{:,i}\\
& \le J^2(A^*) \norm{A^*}_2 \norm{A_1-A_2}_2 +J(A^*) \norm{A_1-A_2}_2+(1-p)J^2(A^*) \norm{\left((A_1-A_2)\right)}_2 \norm{{(A^{*})}}_2 \quad \text{a.s}\\
&=\left((2-p)J^2(A^*) \norm{\left((A_1-A_2)\right)}_2+J(A^*) \right) \norm{A_1-A_2}_2 
\end{align*}
where $(A_1-A_2)_{:,i}$ and ${(A^{*})}_{:,i}$ are the $i$-th columns of $A_1-A_2$ and $A^*$ respectively. Therefore, the map $(\conesetB,\norm{\cdot}_2) \mapsto (\calR,\norm{\cdot}_2)$ is $\left((2-p)J^2(A^*)\norm{A^*}_2+J(A^*)\right)$-Lipschitz a.s.

To show that  $(\conesetB,\norm{\cdot}_F) \mapsto (\calR,\norm{\cdot}_F)$ is Lipschitz, we can follow the same arguments as above.
\begin{align*}
\norm{\widetilde{R}_1-\widetilde{R}_2}_F & = \norm{R_1(A_1)+R_2(A_1)+R_3(A_1)-R_1(A_2)-R_2(A_2)-R_3(A_2)}_F\\
& \le \norm{R_1(A_1)-R_1(A_2)}_F+\norm{R_2(A_1)-R_2(A_2)}_F+\norm{R_3(A_1)-R_3(A_2)}_F\\
& = \norm{\Gamma^\top \left(\left((I_{T-1} \otimes {A^*}^\top) \widetilde{P}-\breve{P} (I_{T-1} \otimes {A^*}^\top)\right) (I_{T-1} \otimes (A_1-A_2)) \breve{P} \right) \Gamma}_F \\
&+\norm{\matrixenc{0 & 0\\
\widetilde{P}(I_{T-1} \otimes (A_1-A_2)) \breve{P} \Gamma & 0}}_F+\norm{(1-p)\Gamma^\top \breve{P} \left(I_{T-1} \otimes \diag\left((A_1-A_2)^\top A^*\right)\right) \Gamma}_F\\
& \le \norm{\Gamma}_2^2\norm{\left((I_{T-1} \otimes {A^*}^\top) \widetilde{P}-\breve{P} (I_{T-1} \otimes {A^*}^\top)\right)}_2 \norm{\breve{P}}_2 \sqrt{T}\norm{A_1-A_2}_F \\
&+\norm{\Gamma}_2 \norm{\widetilde{P}}_2 \norm{\breve{P}}_2 \sqrt{T}\norm{A_1-A_2}_F+(1-p)\norm{\Gamma}_2^2\norm{\breve{P}}_2\sqrt{T}\norm{ \diag\left((A_1-A_2)^\top A^*\right)}_F\\
& \le J^2(A^*)\norm{A^*}_2\sqrt{T} \norm{A_1-A_2}_F +J(A^*) \sqrt{T}\norm{A_1-A_2}_F\\
&+(1-p)J^2(A^*)\sqrt{T} \left(\sum_{i \in [n]} \left({\left((A_1-A_2)_{:,i}\right)}^\top {(A^{*})}_{:,i}\right)^2\right)^{1/2}\\
& \le J^2(A^*)\norm{A^*}_2\sqrt{T} \norm{A_1-A_2}_F +J(A^*) \sqrt{T}\norm{A_1-A_2}_F\\
&+(1-p)J^2(A^*)\sqrt{T} \left(\sum_{i \in [n]}\norm{\left((A_1-A_2)_{:,i}\right)}_2^2 \norm{{(A^{*})}_{:,i}}_2^2\right)^{1/2} \quad \text{a.s}\\
& \le J^2(A^*)\norm{A^*}_2\sqrt{T} \norm{A_1-A_2}_F +J(A^*) \sqrt{T}\norm{A_1-A_2}_F\\
&+(1-p)J^2(A^*)\sqrt{T} \norm{A^*}_2\left(\sum_{i \in [n]}\norm{\left((A_1-A_2)_{:,i}\right)}_2^2\right)^{1/2}\\
& \le \bigg( J^2(A^*) \norm{A^*}_2 +J(A^*) +(1-p)J^2(A^*)  \norm{{(A^{*})}}_2 \bigg) \sqrt{T} \norm{\left((A_1-A_2)\right)}_F \\
&=\left((2-p)J^2(A^*) \norm{A^*}_2+J(A^*) \right)\sqrt{T} \norm{A_1-A_2}_F.
\end{align*}
Therefore, the map $(\conesetB,\norm{\cdot}_F) \mapsto (\calR,\norm{\cdot}_F)$ is $\left((2-p)J^2(A^*)\norm{A^*}_2+J(A^*)\right)\sqrt{T}$-Lipschitz a.s.
\end{proof}

\subsection{Proof of Lemma~\ref{lem:trace-uniform-P}}\label{appsec:proof-lem-trace-uniform-P}

\begin{proof}
Let $\calS = \set{S_1(A)+S_2(A)+S_3(A): A \in \conesetB}$. Since each $\bar{q}_i \in [-p,1-p]$ for $i \in [nT]$ and centered, we can see that, entries of $\bar{q}$ are zero mean, 1-subgaussian. Therefore, we can apply Theorem~\ref{thm:non-psd-sup} on $\calS$ to show that for any $u>0$, with probability $1-e^{-u}$,
\begin{align*}
\sup_{S \in \calS} \absenc{\bar{q}^\top S \bar{q}-\expec \bar{q}^\top S \bar{q}} \le c' \left(\gammafn{1}(\calS,\norm{\cdot}_2)+\gammafn{2}(\calS,\norm{\cdot}_F)\right)+ud_2(\calS)+\sqrt{u}d_F(\calS).
\end{align*}
Similar to in the proof of Lemma~\ref{lem:M1-expec} we next show that $(\conesetB,\norm{\cdot}_2) \mapsto (\calS,\norm{\cdot}_2)$ and $(\conesetB,\norm{\cdot}_F) \mapsto (\calS,\norm{\cdot}_F)$ are Lipschitz.
\begin{claim}\label{clm:lip-map-calS}
The following statements hold.
\begin{enumerate}
\item The map $(\conesetB,\norm{\cdot}_2) \mapsto (\calS,\norm{\cdot}_2)$ is $(3-p)J^2(A^*)\norm{A^*}_2$-Lipschitz, and,
\item The map $(\conesetB,\norm{\cdot}_F) \mapsto (\calS,\norm{\cdot}_F)$ is $(3-p)J^2(A^*)\norm{A^*}_2\sqrt{T}$-Lipschitz.
\end{enumerate}
\end{claim}

Using Property~\ref{itm:prop-3} from Appendix~\ref{sec:prelim-tools}, we get,
\begin{align*}
&\gammafn{1}\left(\calS,\norm{\cdot}_2\right) \le \bar{c}_1 (3-p)J^2(A^*)\norm{A^*}_2 \gammafn{1}\left(\conesetB,\norm{\cdot}_2\right),\\
&\gammafn{2}\left(\calS,\norm{\cdot}_F\right) \le \bar{c}_2 (3-p)J^2(A^*)\norm{A^*}_2\sqrt{T} \gammafn{2}\left(\conesetB,\norm{\cdot}_F\right).
\end{align*}
We can also see that,
\begin{align*}
d_2(\calS) &= \sup_{S\in \calS} \norm{S}_2 
= \sup_{A \in \conesetB} \norm{S_1(A)+S_2(A)+S_3(A)}_2\\
& \le (3-p)J^2(A^*)\norm{A^*}_2  \sup_{A \in \conesetB} \norm{A}_2  \\
& \le (3-p)J^2(A^*)\norm{A^*}_2, \\
\text{ and } \ d_F(\calS) &= \sup_{S\in \calS} \norm{S}_F = \sup_{A \in \conesetB} \norm{S_1(A)+S_2(A)+S_3(A)}_F\\
& \le (3-p)J^2(A^*)\norm{A^*}_2\sqrt{T} \sup_{A \in \conesetB}   \norm{A}_F \\
& = (3-p)J^2(A^*)\norm{A^*}_2\sqrt{T}.
\end{align*}

Therefore, with probability at least $1-e^{-u}$,
\begin{align*}
&\sup_{S \in \calS} \absenc{\bar{q}^\top S \bar{q}-\expec \bar{q}^\top S \bar{q}} \\
&\le (3-p)J^2(A^*)\norm{A^*}_2 \bigg(\bar{c}_1 \gammafn{1}(\conesetB,\norm{\cdot}_2)+\bar{c}_2 \gammafn{2}(\conesetB,\norm{\cdot}_F)\sqrt{T}+\sqrt{u} \sqrt{T}+u\bigg).
\end{align*}
\end{proof}

\begin{proof}[Proof of Claim~\ref{clm:lip-map-calS}]

Let $A_1,A_2 \in \conesetB$, and $\widetilde{S}_1,\widetilde{S}_2 \in \calS$ be such that $\widetilde{S}_1 = S_1(A_1)+S_2(A_1)+S_3(A_1)$ and $\widetilde{S}_2 = S_1(A_2)+S_2(A_2)+S_3(A_2)$. We will start with the map, $(\conesetB,\norm{\cdot}_2) \mapsto (\calS,\norm{\cdot}_2)$. Note that,
\begin{align*}
\norm{\widetilde{S}_1-\widetilde{S}_2}_2 &= \norm{S_1(A_1)+S_2(A_1)+S_3(A_1)-S_1(A_2)-S_2(A_2)-S_3(A_2)}_2\\
&\le \norm{S_1(A_1)-S_1(A_2)}_2+\norm{S_2(A_1)-S_2(A_2)}_2+\norm{S_3(A_1)-S_3(A_2)}_2\\
&=\norm{\left(\left(I_{T-1} \otimes A^*\right)G\right) \odot \left(I_{T-1} \otimes (A_1-A_2)\right)}_2\\
&+(1-p)\norm{G \odot \left(I_{T-1} \otimes \diag\left((A_1-A_2)^\top A^*\right)\right)}_2+\norm{G \odot \left(I_{T-1} \otimes \left({A^*}^\top (A_1-A_2)\right)\right)}_2\\
& \le \norm{\left(I_{T-1} \otimes A^*\right)G}_2\norm{A_1-A_2}_2+(1-p)\norm{G}_2\norm{A^*}_2\norm{(A_1-A_2}_2+\norm{G}_2 \norm{A^*}_2\norm{A_1-A_2}_2\\
&\le (3-p)J^2(A^*) \norm{A^*}_2 \norm{(A_1-A_2}_2.
\end{align*}
Therefore, the map $(\conesetB,\norm{\cdot}_2) \mapsto (\calS,\norm{\cdot}_2)$ is $(3-p)J^2(A^*)\norm{A^*}_2$-Lipschitz.

Now consider the map, $(\conesetB,\norm{\cdot}_F) \mapsto (\calS,\norm{\cdot}_F)$. We can see that,
\begin{align*}
\norm{\widetilde{S}_1-\widetilde{S}_2}_F &= \norm{S_1(A_1)+S_2(A_1)+S_3(A_1)-S_1(A_2)-S_2(A_2)-S_3(A_2)}_F\\
&\le \norm{S_1(A_1)-S_1(A_2)}_F+\norm{S_2(A_1)-S_2(A_2)}_F+\norm{S_3(A_1)-S_3(A_2)}_F\\
&=\norm{\left(\left(I_{T-1} \otimes A^*\right)G\right) \odot \left(I_{T-1} \otimes (A_1-A_2)\right)}_F\\
&+(1-p)\norm{G \odot \left(I_{T-1} \otimes \diag\left((A_1-A_2)^\top A^*\right)\right)}_F+\norm{G \odot \left(I_{T-1} \otimes \left({A^*}^\top (A_1-A_2)\right)\right)}_F\\
& \le \left(\norm{\left(I_{T-1} \otimes A^*\right)G}_2+(1-p)\norm{G}_2\norm{A^*}_2+\norm{G}_2 \norm{A^*}_2\right)\sqrt{T}\norm{A_1-A_2}_F\\
&\le (3-p)J^2(A^*) \norm{A^*}_2 \sqrt{T}\norm{(A_1-A_2}_F.
\end{align*}
Therefore, the map $(\conesetB,\norm{\cdot}_F) \mapsto (\calS,\norm{\cdot}_F)$ is $(3-p)J^2(A^*)\norm{A^*}_2\sqrt{T}$-Lipschitz.
\end{proof}

\subsection{Proof of Lemma~\ref{lem:dotprod-unifP}}\label{appsec:proof-lem-dotprod-unifP}

\begin{proof}
Consider 
\begin{align*}
f_2(A) &= \bigg | p\Tr \left( \Gamma^\top \left((I_{T-1} \otimes {A^*}^\top) \overline{P}_2 (I_{T-1} \otimes A) \right) \Gamma\right)+2p(1-p)\Tr \left(\Gamma^\top \overline{P}_1 \left(I_{T-1} \otimes \diag\left(A^\top A^*\right)\right)  \Gamma\right)\\
&- p\Tr \left( \Gamma^\top \left(\overline{P}_1 \left(I_{T-1} \otimes \left({A^*}^\top A\right)\right)  \right) \Gamma\right) \bigg |.
\end{align*}
Note that, using the properties of trace, we can rearrange the terms to get,
\begin{align*}
f_2(A) &= \bigg | p\dotprod{(I_{T-1} \otimes {A^*}) \Gamma \Gamma^\top (I_{T-1} \otimes {A}^\top)}{\overline{P}_2}+2p(1-p)\dotprod{\Gamma \Gamma^\top \left(I_{T-1} \otimes \diag\left({A^*}^\top A\right)\right) }{\overline{P}_1}\\
&- p\dotprod{\Gamma \Gamma^\top \left(I_{T-1} \otimes \left({A}^\top A^*\right)\right)}{\overline{P}_1} \bigg |.
\end{align*}
For any $A \in \conesetB$, let,
\begin{align*}
H(A) &= p\matrixenc{0_{n \times n} & 0_{n \times n(T-1)} \\ 0_{n(T-1) \times n} & (I_{T-1} \otimes {A^*}) \Gamma \Gamma^\top (I_{T-1} \otimes {A}^\top)}-p\matrixenc{ \Gamma \Gamma^\top \left(I_{T-1} \otimes \left({A}^\top A^*\right)\right) & 0_{n(T-1) \times n}\\0_{n \times n(T-1)} & 0_{n \times n}}\\
&+2p(1-p)\matrixenc{\Gamma \Gamma^\top \left(I_{T-1} \otimes \diag\left({A^*}^\top A\right)\right)  & 0_{n(T-1) \times n}\\ 0_{n \times n(T-1)} & 0_{n \times n}}.
\end{align*}
Note that then for any $A \in \conesetB$, we have $f_2(A) = \dotprod{H(A)}{\overline{P}}$ where 
$$\overline{P} = \matrixenc{P_1 & 0 & 0 & \dots & 0 \\ 0 & P_2 & 0 & \dots & 0 \\ \vdots & \vdots & \vdots & \dots & \vdots \\ 0 & 0 & 0 & \dots & P_T}-pI.$$ 
Since $\overline{P}$ is a diagonal matrix, the only entries that affect the trace are the diagonal entries of $H(A)$. Following a simple calculation, we can show that the matrix consisting of the diagonal entries of $H(A)$ (denoted by $H_D(A)$) can be written as,
\begin{align*}
& H_D(A) \\
&= p\matrixenc{
0 & 0 & \dots & 0 \\ 
0 & \diag\left(A^* G_1 A^\top\right) & \dots & 0\\ 
\vdots & \vdots & \dots & \vdots\\
0 & 0 & \dots & \diag\left(A^* G_{T-1} A^\top\right)
}\\
&+2p(1-p)\matrixenc{\diag\left(G_1\diag\left( {A^{*}}^\top A\right)\right) & 0 & \dots & 0 & 0 \\
0 & \diag\left(G_2\diag\left( {A^{*}}^\top A\right)\right) & \dots & 0 & 0\\
\vdots & \vdots & \ddots & \vdots & \vdots \\
0 & 0 & \dots & \diag\left(G_{T-1}\diag\left( {A^{*}}^\top A\right)\right) & 0\\
0 & 0 & \dots & 0 & 0}\\
&-p\matrixenc{\diag\left(G_1 A^\top A^*\right) & 0 & \dots & 0 & 0 \\
0 & \diag\left(G_2 A^\top A^*\right) & \dots & 0 & 0\\
\vdots & \vdots & \ddots & \vdots & \vdots \\
0 & 0 & \dots & \diag\left(G_{T-1} A^\top A^*\right) & 0\\
0 & 0 & \dots & 0 & 0},
\end{align*}
and $\dotprod{H(A)}{\overline{P}} = \dotprod{H_D(A)}{\overline{P}}$. Let $h(A)$ be a vector such that $h(A)_i = H_D(A)_{ii}$. Then, we can also see that, $\dotprod{H(A)}{\overline{P}} =\dotprod{h(A)}{\bar{q}}$. Now let,
\begin{equation*}
\calH = \set{h(A): \ A \in \conesetB}, \quad \text{and} \quad
\calH_{mat} = \set{H_D(A): \ A \in \conesetB}.
\end{equation*}
Note that by definition, the map, $(\calH, \norm{\cdot}_2) \mapsto (\calH_{mat}, \norm{\cdot}_F)$ is a one-to-one and onto mapping. Therefore, we can see that $\gammafn{2}(\calH, \norm{\cdot}_2)=\gammafn{2}(\calH_{mat}, \norm{\cdot}_F)$. Also note that since each $q_i \in [-p,1-p]$ and centered, we have a vector $q$ with zero mean, 1-subgaussian entries. Applying the Remark~\ref{rem:gaus-width} on $\calH$ implies $\expec \sup_{h \in \calH} \dotprod{q}{h} \le c w(\calH)$ and from Talagrand's majorizing measure theorem~\cite[Theorem 2.4.1]{talagrand2014upper}, 
\[
w(\calH) \asymp \gammafn{2} \left(\calH,\norm{\cdot}_2\right) = \gammafn{2}\left(\calH_{mat},\norm{\cdot}_F\right).
\]
Furthermore, letting $g(x) = \sup_{h \in \calH} \dotprod{x}{h}$, we can see that, 
\begin{align*}
&\absenc{g(x_1)-g(x_2)} 
= \absenc{\sup_{h \in \calH} \dotprod{x_1}{h}-\sup_{h \in \calH} \dotprod{x_2}{h}} \\
& \le \sup_{h \in \calH} \norm{h}_2 \norm{x_1-x_2}_2\\
& \le \sup_{A \in \conesetB} \norm{h(A)}_2 \norm{x_1-x_2}_2\\
& \le \sup_{A \in \conesetB} \bigg(p\left(\sum_{t=1}^{T-1} \norm{\diag(A^*G_tA^\top)}_F^2\right)^{1/2}+2p(1-p)\left(\sum_{t=1}^{T-1} \norm{\diag(G_t\diag({A^*}^\top A)}_F^2\right)^{1/2}\\
&+p\left(\sum_{t=1}^{T-1} \norm{\diag(G_tA^\top A^*)}_F^2\right)^{1/2}\bigg) \norm{x_1-x_2}_2\\
& \le \sup_{A \in \conesetB}  \bigg(p \norm{G}_2 \norm{A^*}_2 \sqrt{T} \norm{A}_F+2p(1-p)\norm{G}_2 \norm{A^*}_2 \sqrt{T} \norm{A}_F\\
& +p\norm{G}_2 \norm{A^*}_2 \sqrt{T} \norm{A}_F\bigg) \norm{x_1-x_2}_2\\
&\le 2p(2-p)J^2(A^*)\norm{A^*}_2 \sqrt{T} \norm{x_1-x_2}_2.
\end{align*}
Therefore, $g_1(x)$ is $2p(2-p)J^2(A^*)\norm{A^*}_2 \sqrt{T}$-Lipschitz with respect to $\norm{\cdot}_2$. Therefore, applying the Theorem~\ref{thm:bounded-func-conc} on $g(q)$ gives us, for any $t>0$,
\begin{align*}
\prob \left(\absenc{\sup_{h \in H} \dotprod{q}{h}-\expec \sup_{h \in H} \dotprod{q}{h}} \ge t\right) \le 2\exp\left(-\frac{t^2}{8p^2(2-p)^2J^4(A^*)\norm{A^*}_2^2 T}\right).
\end{align*}
Letting $t = c 2p(2-p)J^2(A^*)\norm{A^*}_2 \sqrt{T} \sqrt{\log (1/\delta)}$ (for $\delta>0$), we have, 
\begin{align*}
\prob \left(\sup_{h \in H} \dotprod{q}{h} \ge \expec \sup_{h \in H} \dotprod{q}{h}+ c 2p(2-p)J^2(A^*)\norm{A^*}_2 \sqrt{T} \sqrt{\log (1/\delta)}\right) \le 2\exp\left(-c\log (1/\delta)\right).
\end{align*}
and therefore, with probability at least $1-\delta$, 
\begin{align*}
\sup_{h \in H} \dotprod{q}{h} &\le \expec \sup_{h \in H} \dotprod{q}{h}+ c p(2-p)J^2(A^*)\norm{A^*}_2 \sqrt{T} \sqrt{\log (1/\delta)}\\ 
&\le c\gammafn{2}(\calH_{mat},\norm{\cdot}_F)+ c p(2-p)J^2(A^*)\norm{A^*}_2 \sqrt{T} \sqrt{\log (1/\delta)}.
\end{align*}
Now consider the map $(\conesetB, \norm{\cdot}_F) \mapsto (\calH_{mat}, \norm{\cdot}_F)$. We can see that this map is onto by definitions of $\calH$ and $\calH_{mat}$. We next show that this map is Lipschitz as well.
\begin{claim}\label{clm:lip-map-calH}
The map $(\conesetB, \norm{\cdot}_F) \mapsto (\calH_{mat}, \norm{\cdot}_F)$ is $2p(2-p)J^2(A^*)\norm{A^*}_2 \sqrt{T}$-Lipschitz.
\end{claim}
Using Claim~\ref{clm:lip-map-calH}, and the Property~\ref{itm:prop-3} from Appendix~\ref{sec:prelim-tools}, we can see that, 
\[
\gammafn{2}(\calH_{mat},\norm{\cdot}_F) \le \bar{c}_1 p(2-p)J^2(A^*)\norm{A^*}_2 \sqrt{T} \gammafn{2}(\conesetB,\norm{\cdot}_F)
\]
and therefore, with probability at least $1-\delta$,
\begin{align*}
\sup_{h \in H} \dotprod{q}{h} &\le \widetilde{c}p(2-p)J^2(A^*)\norm{A^*}_2 \sqrt{T} \left( \gammafn{2}(\conesetB,\norm{\cdot}_F)+\sqrt{\log (1/\delta)}\right).
\end{align*}
which implies with probability at least $1-\delta$, for any $A \in \conesetB$,
\begin{align*}
&\bigg | p\Tr \left( \Gamma^\top \left((I_{T-1} \otimes {A^*}^\top) \overline{P}_2 (I_{T-1} \otimes A) \right) \Gamma\right)+2p(1-p)\Tr \left(\Gamma^\top \overline{P}_1 \left(I_{T-1} \otimes \diag\left(A^\top A^*\right)\right)  \Gamma\right)\\
&- p\Tr \left( \Gamma^\top \left(\overline{P}_1 \left(I_{T-1} \otimes \left({A^*}^\top A\right)\right)  \right) \Gamma\right) \bigg |\\
& \le \widetilde{c}p(2-p)J^2(A^*)\norm{A^*}_2 \sqrt{T} \left( \gammafn{2}(\conesetB,\norm{\cdot}_F)+\sqrt{\log (1/\delta)}\right).
\end{align*}
for some constant $\widetilde{c}$.
\end{proof}

\begin{proof}[Proof of Claim~\ref{clm:lip-map-calH}]
Let $A_1, A_2 \in \conesetB$ and let $H_1, H_2$, such that $H_1 = H_D(A_1)$ and $H_2=H_D(A_2)$. Then,
\begin{align*}
&\norm{H_1-H_2}_F \\
&= \norm{H_D(A_1)-H_D(A_2)}_F\\
&=\norm{H_D(A_1-A_2)}_F \; \text{(since $H_D(A)$ is linear)}\\
& \le  p\left(\sum_{t=1}^{T-1} \norm{\diag(A^*G_t(A_1-A_2)^\top)}_F^2\right)^{1/2}+2p(1-p)\left(\sum_{t=1}^{T-1} \norm{\diag(G_t\diag({A^*}^\top (A_1-A_2))}_F^2\right)^{1/2}\\
&+p\left(\sum_{t=1}^{T-1} \norm{\diag(G_t(A_1-A_2)^\top A^*)}_F^2\right)^{1/2}\\
& \le  p \norm{G}_2 \norm{A^*}_2 \sqrt{T} \norm{A_1-A_2}_F+2p(1-p)\norm{G}_2 \norm{A^*}_2 \sqrt{T} \norm{A_1-A_2}_F+p\norm{G}_2 \norm{A^*}_2 \sqrt{T} \norm{A_1-A_2}_F\\
&\le 2p(2-p)J^2(A^*)\norm{A^*}_2 \sqrt{T} \norm{A_1-A_2}_F.
\end{align*}
Therefore, the map $(\conesetB, \norm{\cdot}_F) \mapsto (\calH_{mat}, \norm{\cdot}_F)$ is $2p(2-p)J^2(A^*)\norm{A^*}_2 \sqrt{T}$-Lipschitz.
\end{proof}

\section{Proof of Theorem~\ref{thm:pgd-thm}}\label{app-sec:pgd-proof}

By Lemmas \ref{lem:PGD-error-decomp} to \ref{lem:PGD-third-term}, we obtain with probability at least $1 - \delta$, for all $k = 0, 1, \dots$
\begin{align*}
    \norm{A^{(k+1} - B}_F &\leq \norm{A^{(k)} - B}_F \left(1 - \frac{1}{2J^2(A^*)} + \frac{C}{p^2}\left\{\frac{\gamma_1(\calC_{B}, \norm{.}_2) }{T} + \frac{\gamma_2(\calC_{B}, \norm{.}_F)}{\sqrt{T}} + \frac{\log(1/\delta)}{T} + \sqrt{\frac{\log(1/\delta)}{T}} \right\}\right) \\
    &+ \norm{A^* - B}_F\left[\frac{2-p}{p} + \frac{\left\{\overline F_1 + \overline V_1 \log(1/\delta) \right\}\left\{\overline c_1 + p(1-p)\overline c_2 \right\} }{p^2J^2(A^*)T} + \frac{c_3(2-p)\left\{\gamma_2(\calC_{B}, \norm{.}_F) + \sqrt{\log(1/\delta)} \right\}}{pJ(A^*)\sqrt{T}}  \right] \\
    &+ \frac{J(A^*)}{p^2}\left[\frac{\gamma_1(\calC_{B}, \norm{.}_2)}{T} + \frac{\log(1/\delta)}{T} + \frac{\gamma_2(\calC_{B}, \norm{.}_F)}{\sqrt{T}} + \sqrt{\frac{\log(1/\delta)}{T}} \right],
\end{align*}
where, as in \eqref{eq:f1-v1-bar}, we have
\begin{align*}
    \bar{F}_1&=\bar{C}_2^2 J^2(A^*) \left(\gammafn{2}^2(\calC_B ,\norm{\cdot}_2) +\gammafn{2}(\calC_B,\norm{\cdot}_2) \sqrt{T}+\sqrt{T}\right) \; ,\\
\bar{V}_1&=\bar{C}_2 J^2(A^*)\left(\gammafn{2}(\calC_B,\norm{\cdot}_2) +\sqrt{T}\right).
\end{align*}
By simplifying the expressions associated to the term multiplying $\norm{A^* - B}_F$ above, which involves tedious calculations, 
we obtain 
\begin{align*}
      \norm{A^{(k+1)} - B}_F &\leq \norm{A^{(k)} - B}_F \left(1 - \frac{1}{2J^2(A^*)} + \frac{C}{p^2}\left\{\frac{\gamma_1(\calC_{B}, \norm{.}_2) }{T} + \frac{\gamma_2(\calC_{B}, \norm{.}_F)}{\sqrt{T}} + \frac{\log(1/\delta)}{T} + \sqrt{\frac{\log(1/\delta)}{T}} \right\}\right) \\
      &+ \norm{A^* - B}_F\Bigg[\frac{2}{p} + \frac{C^\prime}{p^2}\left(\frac{\gamma_2^2(\calC_{B}, \norm{.}_F)}{T} + \frac{\gamma_2(\calC_{B}, \norm{.}_F) + 1}{\sqrt{T}} + \left\{\frac{\gamma_2(\calC_{B}, \norm{.}_F) + \sqrt{T}}{T} \right\}\log(1/\delta)  \right) \\
      &+ \frac{C^{\prime\prime}\left(\gamma_2(\calC_{B}, \norm{.}_F) + \sqrt{\log(1/\delta)} \right)}{pJ(A^*)\sqrt{T}} \Bigg] \\
      &+ \frac{J(A^*)}{p^2}\left[\frac{\gamma_1(\calC_{B}, \norm{.}_2)}{T} + \frac{\log(1/\delta)}{T} + \frac{\gamma_2(\calC_{B}, \norm{.}_F)}{\sqrt{T}} + \sqrt{\frac{\log(1/\delta)}{T}} \right].
\end{align*}
Thus, we obtain with probability at least $1 - \delta$, for all $k = 0, 1, \dots$
\begin{equation*}
    \norm{A^{(k+1)} - B}_F \leq \rho \norm{A^{(k)} - B}_F + \widetilde \varepsilon, 
\end{equation*}
with 
\begin{equation*}
    \rho = 1 - \frac{1}{2J^2(A^*)} + \frac{C}{p^2}\left\{\frac{\gamma_1(\calC_{B}, \norm{.}_2) }{T} + \frac{\gamma_2(\calC_{B}, \norm{.}_F)}{\sqrt{T}} + \frac{\log(1/\delta)}{T} + \sqrt{\frac{\log(1/\delta)}{T}} \right\},
\end{equation*}
and 
\begin{align*}
    \widetilde \varepsilon &= \frac{2}{p} + \frac{C^\prime}{p^2}\left(\frac{\gamma_2^2(\calC_{B}, \norm{.}_F)}{T} + \frac{\gamma_2(\calC_{B}, \norm{.}_F) + 1}{\sqrt{T}} + \left\{\frac{\gamma_2(\calC_{B}, \norm{.}_F) + \sqrt{T}}{T} \right\}\log(1/\delta)  \right) + \frac{C^{\prime\prime}\left(\gamma_2(\calC_{B}, \norm{.}_F) + \sqrt{\log(1/\delta)} \right)}{pJ(A^*)\sqrt{T}}  \\
      &+ \frac{J(A^*)}{p^2}\left[\frac{\gamma_1(\calC_{B}, \norm{.}_2)}{T} + \frac{\log(1/\delta)}{T} + \frac{\gamma_2(\calC_{B}, \norm{.}_F)}{\sqrt{T}} + \sqrt{\frac{\log(1/\delta)}{T}} \right].
\end{align*}

Under the sufficient condition
\begin{equation*}
        T \geq \frac{CJ^4(A^*)}{p^4}\left\{\gamma_2^2(\calC_{B}, \norm{.}_F) + \log^2(1/\delta) \right\}
      + \frac{C J^2(A^*)}{p^2}\left\{\gamma_1(\calC_{B}, \norm{.}_2) + \log(1/\delta)\right\},
\end{equation*}
we have that $\rho < 1 - (CJ^2(A^*))^{-1}$, for a constant $C \geq 1$.

Now, by unrolling the recurrence relation, and using the finite geometric series formula, we obtain with probability at least $1 - \delta$, for all $k = 0, 1, \dots$
\begin{equation*}
    \norm{A^{(k)} - B}_F \leq \rho^k \norm{A^{(0)} - B}_F + \frac{1 - \rho^k}{1 - \rho}\widetilde \varepsilon. 
\end{equation*}
By using the triangle inequality twice, and the inequality in the last display, we obtain with probability at least $1 - \delta$, for all $k = 0, 1, \dots$
\begin{align*}
    \norm{A^{(k)} - A^*}_F &\leq \norm{A^{(k)} - B}_F + \norm{B - A^*}_F \\
    &\leq \rho^k \norm{A^{(0)} - B}_F + \frac{1 - \rho^k}{1 - \rho}\varepsilon + \norm{B - A^*}_F \\
    &\leq \rho^k \norm{A^{(0)} - A^*}_F + (1 + \rho^k)\norm{A^* - B}_F + \frac{\widetilde \varepsilon}{1 - \rho}.
\end{align*}
Under our sufficient condition on $T$, and simplifying the expression of $\widetilde \varepsilon$ in terms of dominating terms to get $\varepsilon$, we obtain the bound in the statement.

\begin{lemma}\label{lem:PGD-error-decomp}
    For $p \in (0,1]$, consider the objective function
    \begin{equation*}
        g(A) = \norm{\widetilde Z - AZ}_F^2 - (1-p) \norm{A\operatorname{diag}^{1/2}(ZZ^\top)}_F^2.
    \end{equation*}

Recall the projected gradient descent (PGD) algorithm
\begin{equation*}
    A^{(k+1)} = \Pi_{\cal K}(A^{(k)} - \gamma \nabla g(A^{(k)})), 
\end{equation*}
where $A^{(k)}$ is the $k$-th iteration of the algorithm computed from the data, $\gamma > 0$ is the learning rate, and $\Pi_{\cal K}(.)$ is the orthogonal projection onto the set $\calK$. Recall that 
\begin{align*}
    D = \operatorname{diag}^{1/2}(ZZ^\top), \quad M = ZZ^\top - (1-p)D, \quad \conesetB = \calT_{\calK, B} \cap \frobsphere,
\end{align*}
 and denote the normalized directions as
\begin{equation*}
    H^{(k)}_B := \frac{A^{(k)} - B}{\norm{A^{(k)}-B}_F}, \quad H^*_B := \frac{A^* - B}{\norm{A^* - B}_F}.
\end{equation*}
It holds that 
\begin{align*}
    &\norm{A^{(k+1)} -B}_F \leq \norm{A^{(k)} - B}_F \sup_{V \in \calC_B } \dotprod{V}{H_B^{(k)}(I - 2\gamma \mathbb{E}[M]} 
    \\&+ 2\gamma\norm{A^{(k)} - B}_F \sup_{V \in \calC_B } \dotprod{V}{H_B^{(k)} (\mathbb{E}[M] - M)}  
    \\&+ 2\gamma \sup_{V \in \calC_B} \left\{ \norm{VZ}_F \norm{(A^* - B) Z}_F + (1-p)\norm{VD}_F \norm{(A^* - B)D}_F 
    + \dotprod{VZ}{\widetilde Z - A^* Z} + (1-p)\dotprod{VD}{A^* D} \right\} \\
    &=: (I) + (II) + (III).
\end{align*}
\end{lemma}

\begin{proof}[Proof of Lemma \ref{lem:PGD-error-decomp}]

By Lemma \ref{lem:proj-to-cone}, we have
\begin{align*}
    \norm{A^{(k+1)} - B}_F &= \norm{\Pi_{\calK}\left(A^{(k)} - \gamma \nabla g(A^{(k)}) \right) - B}_F \\
    &\leq \sup_{V \in \calC_B}\dotprod{V}{A^{(k)} - \gamma \nabla g(A^{(k)}) - B} \\
    &= \sup_{V \in \calC_B} \dotprod{V}{A^{(k)} - B - \gamma\left[\nabla g(A^{(k)}) - \nabla g(B) \right] - \gamma\left[\nabla g(B) - \nabla g(A^*) \right] - \gamma \nabla g(A^*)}.
\end{align*}
The gradients above can be calculated explicitly as $\nabla g(A^{(k)}) - \nabla g(B) = 2(A^{(k)} - B)M$, $\nabla g(B) - \nabla g(A^*) = 2(B - A^*)M$, and $\nabla g(A^*) = 2[A^*M - \widetilde Z Z^\top]$. Substituting into the last display, we obtain
\begin{align*}
    \norm{A^{(k+1)} - B}_F &\leq \sup_{V \in \calC_B} \left\{ \dotprod{V}{(A^{(k)} - B)(I - 2\gamma M} + 2\gamma \dotprod{V}{(A^* - B)M} - 2\gamma\dotprod{V}{A^*M - \widetilde Z Z^\top} \right\}
\end{align*}
Intercalating by $\mathbb{E}[M]$, we obtain 
\begin{align*}
    &\norm{A^{(k+1)} - B}_F \leq \sup_{V \in \calC_B} \dotprod{V}{(A^{(k)} - B)(I - 2\gamma \mathbb{E}[M])} + 2\gamma \sup_{V \in \calC_B} \dotprod{V}{(A^{(k)} - B)(\mathbb{E}[M] - M)} \\
    &+ 2\gamma \sup_{V \in \calC_B} \left\{ \dotprod{V}{(A^* - B)Z Z^\top} - (1-p)\dotprod{V}{(A^* - B)DD^\top} + \dotprod{V}{(\widetilde Z - A^* Z)Z^\top + (1-p)A^* DD^\top} \right\}.
\end{align*}
Normalizing by $\norm{A^{(k)} - B}_F$, and using the cyclic properties of the trace, we have
\begin{align*}
    &\norm{A^{(k+1)} - B}_F \leq \norm{A^{(k)} - B}_F \sup_{V \in \calC_B} \dotprod{V}{H_B^{(k)}(I_n - 2\gamma \mathbb{E}[M])} \\
    &+ 2\gamma \norm{A^{(k)} - B}_F\sup_{V \in \calC_B} \dotprod{V}{H_B^{(k)}(\mathbb{E}[M] - M)} \\
    &+ 2\gamma \sup_{V \in \calC_B} \left\{ \dotprod{VZ}{(A^* - B)Z} - (1-p)\dotprod{VD}{(A^* - B)D} + \dotprod{VZ}{\widetilde Z - A^* Z} + (1-p)\dotprod{VD}{A^* D} \right\} \\
    &=:(I) + (II) + (IIA). 
\end{align*}
For $(IIA)$, we have by the Cauchy-Schwarz inequality, 
\begin{align*}
    (IIA) &\leq 2\gamma \sup_{V \in \calC_B} &\left\{ \norm{VZ}_F \norm{(A^* - B) Z}_F + (1-p)\norm{VD}_F \norm{(A^* - B)D}_F 
    + \dotprod{VZ}{\widetilde Z - A^* Z} + (1-p)\dotprod{VD}{A^* D} \right\} \\
    &=: (III),
\end{align*}
and putting together $(I), (II)$ and $(III)$ in the last display completes the proof.
\end{proof}

\begin{lemma}\label{lem:term-1-PGD-decomp}
    Suppose that the learning rate is chosen as 
    \begin{equation*}
        \gamma = \frac{1}{2p^2 T J^2(A^*)}.
    \end{equation*}
Then, it holds for all $k = 0,1,2,\dots$ that
    \begin{equation*}
        \sup_{V \in \calC_{B}} \dotprod{V}{H_B^{(k)}(I - 2\gamma \mathbb{E}[M]}_F  \leq \left(1 - \frac{1}{2J^2(A^*)}\right).
    \end{equation*}
\end{lemma}
\begin{proof}[Proof of Lemma \ref{lem:term-1-PGD-decomp}]
We first bound the maximum and minimum eigenvalues of $\mathbb{E}[M]$, and use it to bound $(I)$. 

In Lemma \ref{lem:unb-est}, the conditional expectation of $M$ given $(P_1, \dots, P_T)$ was derived. By the tower property, we have 
\begin{equation}\label{eq:expec-M-pgd}
    \mathbb{E}[M] = \mathbb{E}[\mathbb{E}[M \mid \mathcal (P_t)_{t=1}^T)]] = \mathbb{E}[p^2 XX^\top] = p^2 \mathbb{E}[X X^\top]. 
\end{equation}
Denote $\eta = \operatorname{vec}([\eta_1, \dots, \eta_{T-1}])$ to be formed by column-wise vectorization, and $$\Gamma_t = [(A^*)^{t-1}, (A^*)^{t-2}, \dots, I, 0, \dots 0] \in \mathbb{R}^{n\times n(T-1)}.$$ Moreover, let $\Gamma = [\Gamma_1, \dots, \Gamma_{T-1}^\top$], so that $X_t = \Gamma_t \eta$. Since the entries of $\eta_t$ are i.i.d centered random variables with unit variance, we have by expanding the outer product on the RHS of \eqref{eq:expec-M-pgd}, 
\begin{align*}
    \mathbb{E}[M] &= p^2 \mathbb{E}[XX^\top] = p^2 \sum_{t=1}^{T-1}\Gamma_t \mathbb{E}[\eta \eta^\top]\Gamma_t^\top = p^2 \sum_{t=1}^{T-1} \Gamma_t \Gamma_t^\top.
\end{align*}
By expanding the outer product involving $\Gamma_t$, it is easy to see that $\Gamma_t \Gamma_t \succeq I$. Then from the last display, we have by the symmetry, as soon as $T \geq 2$
of $\mathbb{E}[M]$
\begin{equation} \label{eq:pgd_contrac_lambmin}
    \mathbb{E}[M] \succeq p^2(T-1)I_n \implies \lambda_{\min}(\mathbb{E}[M]) \geq p^2(T-1) \geq p^2T/2. 
\end{equation}
Moreover, by Weyl's inequality for the sum of symmetric matrices, we have 
\begin{align}
    \lambda_{\max}(\mathbb{E}[M]) &\leq p^2 \sum_{t=1}^{T-1}\lambda_{\max}(\Gamma_t \Gamma_t^\top) = p^2 \sum_{t=1}^{T-1} \|\Gamma_t\|_2^2 \nonumber \\
    &\leq p^2 (T-1)\|\Gamma\|_2^2 \nonumber  \leq p^2 T J^2(A^*). \label{eq:lambda-max-EM}
\end{align}
We now proceed to bound $(I)$. Set the learning rate to $\gamma = (2p^2 T J^2(A^*))^{-1}$, and fix $V \in \mathbb{S}_n$. By the Cauchy-Schwarz inequality,  the submultiplicativity of matrix norms, and \eqref{eq:pgd_contrac_lambmin}, 
we have for our choice of $\gamma$ that
\begin{align*}
    \dotprod{V}{H^{(k)}_B (I_n - 2\gamma \mathbb{E}[M])} &\leq  \norm{V}_F\norm{H^{(k)}_B(I_n - 2\gamma \mathbb{E}[M])}_F \\
    &\leq  \norm{H^{(k)}_B}_F \norm{I_n - 2\gamma \mathbb{E}[M]}_2 \\
    &=  \norm{I_n - 2\gamma \mathbb{E}[M]}_2 \\
    &\leq 1 - 2\gamma \lambda_{min}(\mathbb{E}[M])  = 1 - \frac{1}{2J^2(A^*)} , \tag{using \eqref{eq:pgd_contrac_lambmin}} 
\end{align*}
which completes the proof.
\end{proof}

\begin{lemma}\label{lem:PGD-centered-term}
   Suppose the learning rate is chosen as 
       \begin{equation*}
        \gamma = \frac{1}{2p^2 T J^2(A^*) }.
    \end{equation*}
    Then for $\delta \in (0, 1)$, it holds with probability at least $1 - \delta$ for all $k = 0,1,2,\dots$ that
   \begin{equation*}
   2\gamma \sup_{V \in \calC_{B}} \dotprod{V}{H_B^{(k)}(\mathbb{E}[M] - M)} \leq
    \frac{C}{p^2}\left(\frac{\gamma_1(\calC_B, \norm{.}_2)}{T} + \frac{\gamma_2(\calC_B, \norm{.}_F)}{\sqrt{T}} + \frac{\log(1/\delta)}{T} + \frac{\sqrt{\log(1/\delta)}}{\sqrt{T}} \right).
\end{equation*}
\end{lemma}
\begin{proof}[Proof of Lemma \ref{lem:PGD-centered-term}]

Since $A^{(k)}$ and $B$ lies in the tangent cone, we have $H_B^{(k)} \in \calC_{B}$. Thus,  
\begin{align*}
    \sup_{V \in \calC_B} \dotprod{V}{H_B^{(k)}(\mathbb{E}[M] - M)} \leq \sup_{U,V \in \calC_B} \dotprod{U^\top V}{\mathbb{E}[M] - M},
\end{align*}
where we recall that $M = ZZ^\top - (1-p)DD^\top$. We will rewrite the RHS of the last display as a suitable quadratic form, before applying Theorem \ref{thm:non-psd-sup}.

To this end, note that in the proof of Claim \ref{clm:quad-frob-norm} in Appendix \ref{appsec:clm-quad-frob-norm}, we have shown that 
\begin{equation*}
    \dotprod{A^\top A}{\operatorname{diag}(ZZ^\top)} = \eta^\top \Gamma^\top \check P (I_{T-1} \otimes \operatorname{diag}(A^\top A))\check P \Gamma \eta.
\end{equation*}
In fact, the identity above holds for any $U, V \in \mathbb{R}^{n\times n}$ in place of $A^\top$ and $A$, noting that all the terms other than $A$ on the RHS in the last display do not depend on $A$. Thus, we have 
\begin{equation*}
    \dotprod{U^\top V}{DD^\top} = \dotprod{U^\top V}{\operatorname{diag}(ZZ^\top)} = \eta^\top \Gamma^\top \check P (I_{T-1} \otimes \operatorname{diag}(U^\top V))\check P \Gamma \eta 
\end{equation*}
Moreover, we have 
\begin{equation*}
    \dotprod{U^\top V}{Z Z^\top} = \Tr(Z^\top V^\top U Z) = \vect(Z^\top)\left(I_{T-1} \otimes U^\top V \right)\vect(Z
) = \eta^\top \Gamma^\top \check P \left(I_{T-1} \otimes U^\top V \right)\check P \Gamma \eta.
\end{equation*}
Thus, we have
\begin{align*}
    \dotprod{U^\top V}{M} &= \dotprod{U^\top V}{ZZ^\top - (1-p)DD^\top} = \eta^\top \Gamma^\top \check P\left(I_{T-1} \otimes \left\{U^\top V - (1-p)\operatorname{diag}(U^\top V) \right\} \right)P\Gamma \eta. 
\end{align*}
Define the matrix valued-map 
\begin{equation*}
    R:\begin{cases}
        \calC_B \times \calC_B \mapsto \mathbb{R}^{n(T-1)\times n(T-1)}, \\
        (U,V) \mapsto R_{U,V},
    \end{cases}
\end{equation*}
where
\begin{equation*}
    R_{U,V} = \Gamma^\top \check P\left(I_{T-1} \otimes \left\{U^\top V - (1-p)\operatorname{diag}(U^\top V) \right\} \right)P\Gamma,
\end{equation*}
and is surjective onto $\calR := \left\{R_{U,V}: U, V \in \calC_B \right\}$ by construction. Then applying Theorem \ref{thm:non-psd-sup}, we have with probability at least $1 - \delta$
\begin{equation}\label{eq:quad-form-bound-PGD}
    \sup_{w \in \calR} \left|\eta^\top w\eta - \mathbb{E}[\eta^\top w\eta] \right| \leq C \left(\gamma_1(\calR, \|.\|_2 + \gamma_2(\calR, \|.\|_F) + \sqrt{\log(1/\delta)}d_F(\calR) + \log(1/\delta)d_2(\calR) \right),
\end{equation}
where $C > 0$ is some absolute constant. 
Note that although $\eta$ is $L$-subgaussian, a rescaling allows one to obtain a 1-subgaussian random vector, which only affects the constants in the resulting bound above.

We now proceed to bound the terms on the RHS of \eqref{eq:quad-form-bound-PGD}. By the triangle inequality, and repeated applications of the sub-multiplicativity of norms, we have 
\begin{align*}
    d_2(\calR) &= \sup_{w \in \calR} \norm{w}_2 = \sup_{U,V \in \calC_B} \norm{\Gamma^\top \check P\left(I_{T-1} \otimes \left\{U^\top V - (1-p)\operatorname{diag}(U^\top V) \right\} \right)P\Gamma}_2 \\
    &\leq \sup_{U,V \in \calC_B} \norm{\Gamma}_2^2 \norm{\check P}\|_2^2 \norm{U^\top V - (1-p)\operatorname{diag}(U^\top V)}_2  \\
    &\leq J^2(A^*)  \sup_{U,V \in \calC_B} \left(\norm{U^\top V}_2 + (1-p)\norm{\operatorname{diag}(U^\top V)}_2 \right) \\
    &\leq J^2(A^*)(2-p)\sup_{U,V\in \calC_B} \norm{U^\top V}_2 \leq 2J^2(A^*).
\end{align*}

On the other hand, by using the similar lines of arguments as above, we have
\begin{align*}
    d_F(\calR) &= \sup_{w \in \calR} \norm{w}_F = \sup_{U,V \in \calC_B} \norm{\Gamma^\top \check P\left(I_{T-1} \otimes \left\{U^\top V - (1-p)\operatorname{diag}(U^\top V) \right\} \right)P\Gamma}_F \\
    &\leq \sup_{U,V \in \calC_B} \norm{\Gamma}_2^2 \norm{\check P}_2^2 \sqrt{T-1}\left(\norm{U^\top V}_F + (1-p) \norm{\operatorname{diag}(U^\top V}_F  \right) \\
    &\leq 2J^2(A^*)\sqrt{T}.
\end{align*}

It remains to suitably bound the $\gamma_1$ and $\gamma_2$ functionals of the image space by the $\gamma$ functionals of the domain. To this end, let
$\calM := \left\{ (U,V)^\top: U, V \in \calC_B \right\}$, and define the map $\Phi:\calM \rightarrow \calR$, where $\Phi((U,V)^\top = R_{U,V}$, so $\Phi$ is similarly surjective onto $\calR$. Fix $(U_1, V_1)^\top, (U_2,V_2)^\top \in \calM$. Then, again by the repeated application of the sub-multiplicity of norms and the triangle inequality, we obtain
\begin{align*}
    \norm{R_{U_1,V_1} - R_{U_2, V_2}}_2 &= \norm{\Gamma^\top \check P\left(I_{T-1} \otimes \left\{ \left(U_1^\top V_1 - U_2^\top V_2\right) - (1-p)\operatorname{diag}\left(U_1^\top V_1 - U_2^\top V_2  \right) \right\} \right)P\Gamma}_2 \\
    &\leq \norm{\Gamma}_2^2 \norm{\check P}_2^2 \norm{(U_1^\top V_1 - U_2^\top V_2) - (1-p)\operatorname{diag}(U_1^\top V_1 - U_2^\top V_2) }_2 \\
    &\leq 2J^2(A^*)\norm{U_1^\top V_1 - U_2^\top V_2}_2 = 2J^2(A^*) \norm{U_1^\top (V_1 - V_2) + (U_1 - U_2)^\top V_2}_2 \\
    &\leq 2J^2(A^*)\left[\norm{U_1^\top (V_1 - V_2)}_2 + \norm{(U_1 - U_2)^\top V_2}_2 \right] \\
    &\leq 2J^2(A^*)\left[\norm{V_1 - V_2}_2 + \norm{U_1 - U_2}_2 \right] \\
    &\leq 4J^2(A^*)\norm{(U_1, V_1)^\top - (U_2, V_2)^\top}_2,
\end{align*}
so the map $\Phi$ is $4J^2(A^*)$-Lipschitz with respect to the spectral norm. Employing similar arguments as before, we obtain
\begin{equation*}
    \norm{R_{U_1, V_1} - R_{U_2,V_2}}_F \leq 4J^2(A^*)\sqrt{T}\norm{(U_1 V_1)^\top - (U_2, V_2)^\top}_F,
\end{equation*}
and thus $\Phi$ is $4J^2(A^*)\sqrt{T}$-Lipschitz with respect to the Frobenius norm. Thus, we have for some absolute constants $C$ (which may differ from line to line)
\begin{equation*}
    \gamma_1(\calR, \norm{.}_2) \leq C J^2(A^*) \gamma_1(\calM, \norm{.}_2) \leq \gamma_1(\calC_B, \norm{.}_2),
\end{equation*}
and 
\begin{equation*}
    \gamma_2(\calR, \norm{.}_2) \leq CJ^2(A^*)\sqrt{T}\gamma_2(\calM, \norm{}_F) \leq CJ^2(A^*)\sqrt{T}\gamma_2(\calC_B, \norm{.}_F) 
\end{equation*}
Gathering facts, we have with probability at least $1 - \delta$, 
\begin{align*}
      \sup_{w \in \calR} \left|\eta^\top w\eta - \mathbb{E}[\eta^\top w\eta] \right| \leq CJ^2(A^*)\left(\gamma_1(\calC_B, \norm{.}_2) + \gamma_2(\calC_B, \norm{.}_F)\sqrt{T} + \log(1/\delta) + \sqrt{T\log(1/\delta)} \right),
\end{align*}
which implies that by choosing the learning rate to be $\gamma = (2p^2 TJ^2(A^*)^{-1}$, we have with probability at least $1 - \delta$,
\begin{equation*}
   2\gamma \sup_{V \in \calC_B} \dotprod{V}{H_B^{(k)}(\mathbb{E}[M] - M)} \leq
    \frac{C}{p^2}\left(\frac{\gamma_1(\calC_B, \norm{.}_2)}{T} + \frac{\gamma_2(\calC_B, \norm{.}_F)}{\sqrt{T}} + \frac{\log(1/\delta)}{T} + \frac{\sqrt{\log(1/\delta)}}{\sqrt{T}} \right),
\end{equation*}
completing the proof.
\end{proof}

\begin{lemma}\label{lem:PGD-third-term}
There exist positive constants $C_1, C_2$ such that for $\delta \in (0,1)$, it holds with probability at least $1 - \delta$ that
    \begin{align*}
    &2\gamma\sup_{V \in \calC_{B}} \left\{ \norm{VZ}_F \norm{(A^* - B) Z}_F + (1-p)\norm{VD}_F \norm{(A^* - B)D}_F 
    + \dotprod{VZ}{\widetilde Z - A^* Z} + (1-p)\dotprod{VD}{A^* D} \right\} \\
    &\leq \left[\frac{2-p}{p} + \frac{\left\{\overline F_1 + \overline V_1\log(1/\delta)\right\}\left\{\overline c_1 + p(1-p)\overline c_2 \right\}}{p^2 J^2(A^*)T} + \frac{c_3(2-p)\left\{\gammafn{2}(\calC_{B},\norm{\cdot}_F) + \sqrt{\log(1/\delta)} \right\}}{pJ(A^*)\sqrt{T}}  \right] \norm{A^* - B}_F  \\
    &+ \frac{J(A^*)}{p^2}\left[\frac{\gammafn{1}(\calC_{B},\norm{\cdot}_2)}{T} + \frac{\log(1/\delta)}{T} + \frac{\gammafn{2}(\calC_{B},\norm{\cdot}_F)}{\sqrt{T}} + \sqrt{\frac{\log(1/\delta)}{T}} \right].
    \end{align*}
\end{lemma}

\begin{proof}[Proof of Lemma \ref{lem:PGD-third-term}]

Notice that the inner product terms are analogous to the LHS \eqref{eq:dotprod-PGD-unif-bound} with $V$ in place of $A$ (which similarly lies in $\calC_{B})$. Notice that the RHS of \eqref{eq:dotprod-PGD-unif-bound} does not depend on $A$. Thus, we have 
\begin{align*}
 &\sup_{V \in \calC_{B}}  \dotprod{VZ}{\widetilde Z - A^*Z} + (1-p)\dotprod{VD}{A^*D} \\ &\leq c_1\left(J^2(A^*)\norm{A^*}_2+J(A^*)\right) \left(\gammafn{1}(\calC_{B},\norm{\cdot}_2)+\gammafn{2}(\calC_{B},\norm{\cdot}_F)\sqrt{T-1}+\sqrt{(T-1)\log \left(1/\delta\right)}+\log(1/\delta)\right) \\
 &\leq c_1' J^3(A^*) \left(\gammafn{1}(\calC_{B},\norm{\cdot}_2)+\gammafn{2}(\calC_{B},\norm{\cdot}_F)\sqrt{T}+\sqrt{T\log \left(1/\delta\right)}+\log(1/\delta)\right) \\
 &=: \Psi_0.
\end{align*}

On the other hand, $\norm{VZ}_F$ is analogous to \eqref{eq:norm1-unif-bound-PGD}, with $V$ in place of $A$. Thus, we have 
\begin{align*}
    \sup_{V \in \calC_{B}} \norm{VZ}_F^2 &\leq pJ^2(A^*)T+ \left(\bar{F}_1+\bar{V}_1\log (1/\delta)\right)\left\{\overline c_1 + p(1-p)\overline c_2 \right\}
     \\&+c_3 pJ(A^*)\sqrt{T} \left(\gammafn{2} \left(\calC_{B},\norm{\cdot}_F\right)+\sqrt{\log(1/\delta)}\right) \\
     &=: \Psi_1.
\end{align*}

Since $x \mapsto \sqrt{x}$ is monotone increasing and continuous over $x \geq 0$, we have that 
\begin{align*}
   \sup_{V \in \calC_{B}} \norm{VZ}_F \norm{(A^*-B)Z}_F &= \norm{(A^*-B)Z}_F\sqrt{\sup_{V \in \calB} \norm{VZ}_F^2} \\
   &\leq \norm{(A^* - B)Z}_F \sqrt{\Psi_1} \\
   &\leq  \sqrt{\Psi_1} \norm{A^* - B}_F\sup_{V \in \calC_{B}} \norm{VZ}_F  \\
   &\leq \Psi_1 \norm{A^* - B}_F.
\end{align*}

Proceeding similarly as above, we have by \eqref{eq:norm2-unif-bound-PGD}
\begin{align*}
    (1-p) \sup_{V \in \calC_B} \norm{VD}_F^2 &\leq p(1-p) J^2(A^*) T+ \left\{\bar{F}_1+\bar{V}_1\log (1/\delta)\right\}\left(\overline c_1 + p(1-p)^2 \overline c_2\right)\\
&+c_3 p(1-p)J(A^*)\sqrt{T} \left(\gammafn{2} \left(\calC_B,\norm{\cdot}_F\right)+\sqrt{\log (1/\delta)}\right) \\
&=: \Psi_2,
\end{align*}
which thus implies
\begin{align*}
    (1-p)\sup_{V \in \calC_B} \norm{VD}_F \norm{(A^*-B)D}_F &= \sqrt{1-p}\norm{(A^*-B)D}_F\sqrt{1-p} \sqrt{\sup_{V \in \calC_B} \norm{VD}_F^2} \\
    &\leq \sqrt{1-p}\norm{(A^*-B)D}_F \sqrt{\Psi_2} \\
    &\leq \sqrt{\Psi_2}\norm{A^* - B}_F \sqrt{1-p}\sup_{V \in \calC_{B}}\norm{VD}_F \\
    &\leq \Psi_2 \norm{A^* - B}_F.
\end{align*}
Collecting the $\Psi_0, \Psi_1, \Psi_2$ bounds, grouping terms and substituting $\gamma$ completes the proof.
\end{proof}

\begin{lemma}\label{lem:proj-to-cone}
    Let $\calK$ be a closed convex set, and denote $\Pi_{\calK}$ to be the  projection operator onto the set $\calK$.
Then, it holds that
\begin{equation*}
    \norm{\Pi_{\calK}(A) - B}_F \leq \sup_{V \in \calC_B} \dotprod{V}{A-B}.
\end{equation*}
\end{lemma}
\begin{proof}[Proof of Lemma \ref{lem:proj-to-cone}]
The proof of Lemma \ref{lem:proj-to-cone} proceeds analogously to the proof of Lemma 2 in \cite{lv2021}, which relies on the construction in Example 2.32 of \cite{mordukhovich2023}.
\end{proof}

\section{Proofs of corollaries and support lemmas}

\subsection{Proof of Corollary~\ref{cor:stat-lipschitz-regression}}\label{appsec:stat-lipschitz-regression}

From Lemma 4 of \cite{tyagi2024learning}, we can see that since $A^* \in \calK$, $\exists B \in \calK$ such that each row of $B$ is formed by taking $l+1$ affine pieces of slope $\pm L$, and $\absenc{B_{i,j}-B_{i,j+1}} = \frac{L}{T}$ and $\norm{A^*-B}_F^2 \le \frac{4n^4L^2}{T^2 l^2}$. Following the proof of Corollary 4 of \cite{tyagi2024learning}, and letting $\alpha_n=\sqrt{n(l+1)\log \left(\frac{en}{l+1}\right)}$, we can also see that,
\begin{equation*}
\gammafn{2}\left(\calC_{B}, \norm{\cdot}_F\right) \le \alpha_n \quad \text{and} \quad 
\gammafn{1}\left(\calC_{B}, \norm{\cdot}_2\right) \le n\alpha_n \left[\log \left(\frac{n}{\alpha_n}\right) +1\right]. 
\end{equation*}
Then from Theorem~\ref{thm:main-thm} there exists $c_1,c_2>0$ such that for any $\delta \in (0,1)$, if 
\begin{align*}
T \geq c_1\frac{ J^4(A^*) \max \set{ n(l+1)\log \left(\frac{en}{l+1}\right),\log^2 \left(1/\delta\right)}}{p^4},
\end{align*}
it holds with probability at least $1-\delta$,
\begin{align*}
\norm{\hat{A}-A^*}_F 
 &\leq c_2 \bigg[\frac{J^3(A^*)}{p^2} \bigg(\frac{ n\alpha_n \left[\log \left(\frac{n}{\alpha_n}\right) +1\right]}{T}+\frac{\alpha_n}{\sqrt{T}}+\sqrt{\frac{\log \left(1/\delta\right)}{T}}\bigg)+ J^2(A^*)\frac{2n^2L}{T l}\frac{1}{p}\bigg].
\end{align*}
Therefore, using the fact that $\alpha_n \lesssim \sqrt{nl\log n}$ and $n\alpha_n\left[\log\left(\frac{n}{\alpha_n}\right)+1\right] \lesssim \sqrt{n^3l\log^3 n}$, there exist constants $c'_1,c_2' > 1$ such that for any $\delta \in (0,1)$,
\begin{align*}
T \geq c'_1 \frac{ J^4(A^*) \max \set{ nl\log n,\log^2 \left(1/\delta\right)}}{p^4},
\end{align*}
it holds with probability at least $1-\delta$,
\begin{align*}
\norm{\hat{A}-A^*}_F &\leq c'_2  \bigg[\frac{J^3(A^*)}{p^2} \bigg(\frac{ n^{3/2}\sqrt{l}\log^{3/2} n}{T}+\frac{n^{1/2}\sqrt{l}\log^{1/2} n}{\sqrt{T}}+\frac{\sqrt{\log \left(1/\delta\right)}}{\sqrt{T}}\bigg)+\frac{J^2(A^*) n^2L}{T l} \frac{1}{p}\bigg].
\end{align*}

Therefore we have,
\begin{equation*}
\norm{\hat{A}-A^*}_F 
\lesssim  \frac{J^3(A^*)}{p^2}\max\set{\frac{ n^{3/2}\log^{3/2} n}{T},\frac{n^{1/2}\log^{1/2} n}{\sqrt{T}}}  \sqrt{l} + \frac{J^2(A^*) n^2L}{pT }  \frac{1}{l} + \frac{J^3(A^*)\sqrt{\log \left(1/\delta\right)}}{p^2\sqrt{T}}.
\end{equation*}
We now choose $l$ by balancing the first two terms.
\begin{equation*}
\frac{J^3(A^*)}{p^2}\max\set{\frac{ n^{3/2}\log^{3/2} n}{T},\frac{n^{1/2}\log^{1/2} n}{\sqrt{T}}}  \sqrt{l} = \frac{J^2(A^*) n^2L}{pT }  \frac{1}{l}.
\end{equation*}
Using this and the fact that $l \ge 1$, we get, 
\[
l = \ceil{\bigg(\frac{p n^2 L}{J(A^*) T \max\set{\frac{ n^{3/2}\log^{3/2} n}{T},\frac{n^{1/2}\log^{1/2} n}{\sqrt{T}}}}\bigg)^{2/3}} = \ceil{\frac{\frac{p^{2/3} n L^{2/3}}{T^{1/3} \log^{1/3} n }}{J^{2/3}(A^*)  \max\set{ \frac{n^{2/3} \log^{2/3} n }{T^{1/3}},1}}}.
\]
Let $\nu = \frac{\frac{p^{2/3} n L^{2/3}}{T^{1/3} \log^{1/3} n }}{J^{2/3}(A^*)  \max\set{ \frac{n^{2/3} \log^{2/3} n }{T^{1/3}},1}}$. Then $l \leq \nu+1$ and $\frac{1}{l} \leq \frac{1}{\nu}$. 

Since $\max \set{a,b} \asymp a+b$, we can see that there exists a constant $c$ such that for any $\delta \in (0,1)$, the sufficient condition, $T \geq c'_1 \frac{ J^4(A^*) \max \set{ nl\log n,\log^2 \left(1/\delta\right)}}{p^4},$ is satisfied when,
\begin{align*}
T \geq c \frac{ J^4(A^*) \max \set{ n\log n+ \frac{p^{2/3} n^2 L^{2/3} \log^{2/3} n}{J^{2/3}(A^*)  \left( n^{2/3} \log^{2/3} n+T^{1/3}\right)},\log^2 \left(1/\delta\right)}}{p^4}.
\end{align*}
It is easy to verify that this holds whenever,
\begin{align*}
T \gtrsim \frac{J^4(A^*)}{p^4}\max\set{n\log n, \log^{2} \left(1/\delta\right), \frac{p^{2/3} n^{4/3} L^{2/3}}{J^{2/3}(A^*)}, \frac{p^{1/2} n^{3/2} L^{1/2} \log^{1/2} n}{J^{1/2}(A^*)}}
\end{align*}

Under the sufficient condition on $T$, 
\begin{align}
\norm{\hat{A}-A^*}_F 
 &\lesssim  \frac{J^3(A^*)}{p^2}\max\set{\frac{ n^{3/2}\log^{3/2} n}{T},\frac{n^{1/2}\log^{1/2} n}{\sqrt{T}}}  \sqrt{l} + \frac{J^2(A^*) n^2L}{pT }  \frac{1}{l} + \frac{J^3(A^*)\sqrt{\log \left(1/\delta\right)}}{p^2\sqrt{T}} \nonumber\\
&\lesssim \frac{J^3(A^*)}{p^2}\max\set{\frac{ n^{3/2}\log^{3/2} n}{T},\frac{n^{1/2}\log^{1/2} n}{\sqrt{T}}} \nonumber\\
&+\frac{J^3(A^*)}{p^2}\max\set{\frac{ n^{3/2}\log^{3/2} n}{T},\frac{n^{1/2}\log^{1/2} n}{\sqrt{T}}}\bigg(\frac{p n^2 L}{J(A^*) T \max\set{\frac{ n^{3/2}\log^{3/2} n}{T},\frac{n^{1/2}\log^{1/2} n}{\sqrt{T}}}}\bigg)^{1/3}\nonumber\\
&+ \frac{J^2(A^*) n^2L}{pT } \frac{J^{2/3}(A^*)  \max\set{ \frac{n^{2/3} \log^{2/3} n }{T^{1/3}},1}}{\frac{p^{2/3} n L^{2/3}}{T^{1/3} \log^{1/3} n }}   + \frac{J^3(A^*)\sqrt{\log \left(1/\delta\right)}}{p^2\sqrt{T}}\nonumber\\
&\lesssim \frac{J^3(A^*)}{p^2}\max\set{\frac{ n^{3/2}\log^{3/2} n}{T},\frac{n^{1/2}\log^{1/2} n}{\sqrt{T}}} +\frac{J^{8/3}(A^*)n L^{1/3} \log^{1/3} n}{p^{5/3} T^{2/3}}\max\set{\frac{ n^{2/3}\log^{2/3} n}{T^{1/3}},1}\nonumber\\
&+ \frac{J^{5/3}(A^*) n L^{1/3} \log^{1/3} n}{p^{5/3} T^{2/3} }  \max\set{ \frac{n^{2/3} \log^{2/3} n }{T^{1/3}},1}   + \frac{J^3(A^*)\sqrt{\log \left(1/\delta\right)}}{p^2\sqrt{T}}\nonumber\\
&\lesssim \frac{J^3(A^*)}{p^2}\max\set{\frac{ n^{3/2}\log^{3/2} n}{T},\frac{n^{1/2}\log^{1/2} n}{\sqrt{T}}} \nonumber\\
&+ \frac{J^3(A^*) n L^{1/3} \log^{1/3} n}{p^2 T^{2/3} }  \max\set{ \frac{n^{2/3} \log^{2/3} n }{T^{1/3}},1}   + \frac{J^3(A^*)\sqrt{\log \left(1/\delta\right)}}{p^2\sqrt{T}}.\label{eq:lip-err-bnd}
\end{align}

\subsection{Proof of Corollary~\ref{cor:pgd-lipschitz-regression}}\label{appsec:pgd-lipschitz-regression}

Recall from the proof of Corollary~\ref{cor:stat-lipschitz-regression}, the construction of $B$ where each row of $B$ is formed by taking $l+1$ affine pieces, and $\absenc{B_{i,j}-B_{i,j+1}} = \frac{L}{T}$ and $\norm{A^*-B}_F^2 \le \frac{4n^4L^2}{T^2 l^2}$. We have observed that,
\begin{equation*}
\gammafn{2}\left(\calC_{B}, \norm{\cdot}_F\right) \lesssim \sqrt{nl\log n} \quad \text{and} \quad
\gammafn{1}\left(\calC_{B}, \norm{\cdot}_2\right) \lesssim \sqrt{n^3 l \log^3 n}.
\end{equation*}

Now, letting $l = \ceil{\frac{\frac{p^{2/3} n L^{2/3}}{T^{1/3} \log^{1/3} n }}{J^{2/3}(A^*)  \max\set{ \frac{n^{2/3} \log^{2/3} n }{T^{1/3}},1}}}$, we will  establish bounds for the factors used in Theorem~\ref{thm:pgd-thm}. We start with the sufficient condition \eqref{eq:T-PGD-cond},
  \begin{align*}
        T &\geq \frac{CJ^4(A^*)}{p^4}\left\{\gamma_2^2(\calC_{B}, \norm{.}_F) + \log^2(1/\delta) \right\}  
        + \frac{C J^2(A^*)}{p^2}\left\{\gamma_1(\calC_{B}, \norm{.}_2) + \log(1/\delta)\right\},
    \end{align*}
which is satisfied when,
\begin{align*}
        T &\geq \frac{CJ^4(A^*)}{p^4}\left\{nl \log n + \log^2(1/\delta) \right\}  
        + \frac{C J^2(A^*)}{p^2}\left\{n^{3/2} \sqrt{l} \log^{3/2} n + \log(1/\delta)\right\},
    \end{align*}
Note that $l \le \nu+1$ and $1/l \le 1/\nu$ where $\nu = \frac{\frac{p^{2/3} n L^{2/3}}{T^{1/3} \log^{1/3} n }}{J^{2/3}(A^*)  \max\set{ \frac{n^{2/3} \log^{2/3} n }{T^{1/3}},1}}$. This leads to,
\begin{align*}
        T &\geq \frac{CJ^4(A^*)}{p^4}\left\{n \log n+ n\log n \bigg(\frac{\frac{p^{2/3} n L^{2/3}}{T^{1/3} \log^{1/3} n }}{J^{2/3}(A^*)  \max\set{ \frac{n^{2/3} \log^{2/3} n }{T^{1/3}},1}}\bigg)+ \log^2(1/\delta) \right\}\\  
        &+ \frac{C J^2(A^*)}{p^2}\left\{n^{3/2}  \log^{3/2} n+n^{3/2}  \log^{3/2} n\bigg(\frac{\frac{p^{1/3} n^{1/2} L^{1/3}}{T^{1/6} \log^{1/6} n }}{J^{1/3}(A^*)  \max\set{ \frac{n^{1/3} \log^{1/3} n }{T^{1/6}},1}}\bigg) + \log(1/\delta)\right\},
    \end{align*}
which is satisfied when, 
\begin{align}\label{eq:lip-pgd-T}
T &\geq C'\frac{J^4(A^*)}{p^4}\max\set{n\log n, \log^{2} \left(1/\delta\right), \frac{p^{2/3} n^{4/3} L^{2/3}}{J^{2/3}(A^*)}, \frac{p^{1/2} n^{3/2} L^{1/2} \log^{1/2} n}{J^{1/2}(A^*)}}\nonumber\\
&+C'\frac{J^2(A^*)}{p^2}\max\set{n^{3/2}\log^{3/2} n, \log \left(1/\delta\right), \frac{p^{1/3} n^{5/3} L^{1/3}}{J^{1/3}(A^*)}, \frac{p^{2/7} n^{12/7} L^{2/7} \log^{8/7} n}{J^{2/7}(A^*)}}
\end{align}
\paragraph{Bounding $\varepsilon$.} 
For the choice of $T$ satisfying \eqref{eq:lip-pgd-T}, we can see that,
\begin{align*}
    &\varepsilon = \frac{C_2}{p} \norm{A^* - B}_F + C_1 \frac{J(A^*)}{p^2}\Bigg[\frac{\gamma_1(\calC_{B}, \norm{.}_2)}{T} + \frac{\log(1/\delta)}{T} + \frac{\gamma_2(\calC_{B}, \norm{.}_F)}{\sqrt{T}} + \sqrt{\frac{\log(1/\delta)}{T}} \Bigg].\\
    &\le \frac{C_2 n^2L}{pTl}+ C_1 \frac{J(A^*)}{p^2}\Bigg[\frac{ n^{3/2}\sqrt{l}\log^{3/2} n}{T} + \frac{\log(1/\delta)}{T} + \frac{n^{1/2}\sqrt{l}\log^{1/2} n}{\sqrt{T}} + \sqrt{\frac{\log(1/\delta)}{T}} \Bigg]
\end{align*}
From \eqref{eq:lip-err-bnd} in the proof of Corollary~\ref{cor:stat-lipschitz-regression},
\begin{align*}
 &\frac{n^2L}{pTl}+ \frac{J(A^*)}{p^2}\Bigg[\frac{ n^{3/2}\sqrt{l}\log^{3/2} n}{T} + \frac{\log(1/\delta)}{T} + \frac{n^{1/2}\sqrt{l}\log^{1/2} n}{\sqrt{T}} + \sqrt{\frac{\log(1/\delta)}{T}} \Bigg] \\
 &\le \frac{J(A^*)}{p^2}\max\set{\frac{ n^{3/2}\log^{3/2} n}{T},\frac{n^{1/2}\log^{1/2} n}{\sqrt{T}}} \\
&+ \frac{J(A^*) n L^{1/3} \log^{1/3} n}{p^2 T^{2/3} }  \max\set{ \frac{n^{2/3} \log^{2/3} n }{T^{1/3}},1}   + \frac{J(A^*)\sqrt{\log \left(1/\delta\right)}}{p^2\sqrt{T}}
\end{align*}
which implies there exists a constant $C$ such that,
\begin{align*}
\varepsilon &\le C \bigg(
\frac{J(A^*)}{p^2}\max\set{\frac{ n^{3/2}\log^{3/2} n}{T},\frac{n^{1/2}\log^{1/2} n}{\sqrt{T}}} + \frac{J(A^*) n L^{1/3} \log^{1/3} n}{p^2 T^{2/3} }  \max\set{ \frac{n^{2/3} \log^{2/3} n }{T^{1/3}},1}   + \frac{J(A^*)\sqrt{\log \left(1/\delta\right)}}{p^2\sqrt{T}}
\bigg)
\end{align*}

\paragraph{Deriving the final bounds.}
Let $T$ be as given in \eqref{eq:lip-pgd-T}. Then, Theorem~\ref{thm:pgd-thm} implies there exists $\rho \le 1-\frac{1}{cJ^2(A^*)}$ such that,
\begin{align*}
\norm{A^{(k)}-A^*}_F \le \rho^k \norm{A^{(0)}-A^*}_F +\frac{\varepsilon}{1-\rho}
\end{align*} 
Let,
\begin{align*}
\hat{\varepsilon} &= C \bigg(
\frac{J(A^*)}{p^2}\max\set{\frac{ n^{3/2}\log^{3/2} n}{T},\frac{n^{1/2}\log^{1/2} n}{\sqrt{T}}} + \frac{J(A^*) n L^{1/3} \log^{1/3} n}{p^2 T^{2/3} }  \max\set{ \frac{n^{2/3} \log^{2/3} n }{T^{1/3}},1}   + \frac{J(A^*)\sqrt{\log \left(1/\delta\right)}}{p^2\sqrt{T}}
\bigg).
\end{align*}

Then,
\begin{align*}
\norm{A^{(k)}-A^*}_F \le \rho^k \norm{A^{(0)}-A^*}_F +\frac{\hat{\varepsilon}}{1-\rho}.
\end{align*}

\subsection{Proof of Lemma \ref{lem:gamma2-isotone-bound}}

For a matrix $\epsilon = (\epsilon_{ij})_{1 \leq i,j \leq n}$ whose entries are independent centered Gaussian random variables with variance $\sigma^2 > 0$, we denote $\Pi_{\calC}(\epsilon)$ to be the projection onto the closed convex cone $\calC$.

Since $\calT_{\calK, A^*}$ is a closed convex cone, we have by equation (4.4) of \cite{chatterjee2018},
\begin{equation*}
    \mathbb{E}\norm{\Pi_{\calT_{\calK, A^*}}(\epsilon)}^2_F \leq Ck(A^*) \sigma^2 \log^8(n).
\end{equation*}
Now, $\epsilon$ can be written as $\epsilon = \sigma G$, where $G = (G_{ij}) \stackrel{i.i.d}{\sim} \mathcal{N}(0, 1)$. Then since the projection operator onto a cone is positively homogeneous, we have by Proposition 2.4 of \cite{tropp2014},
\begin{equation*}
    \mathbb{E}\norm{\Pi_{\calT_{\calK, A^*}}(\epsilon)}^2_F = \mathbb{E}\left[\norm{\Pi_{\calT_{\calK, A^*}}(\sigma G)}_F^2 \right]
    = \sigma^2 \mathbb{E}\left[\norm{\Pi_{\operatorname{vec}(\calT_{\calK, A^*})}(\operatorname{vec}(G))}_2^2\right] = \sigma^2 \delta(\operatorname{vec}(\calT_{\calK, A^*})),
\end{equation*}
which implies that 
\begin{equation}\label{eq:isotone-stat-dim-bound}
    \delta(\operatorname{vec}(\calT_{\calK, A^*} ))\leq Ck(A^*) \log^8(n).
\end{equation}

Recall Talagrand's majorizing measures theorem \cite[Theorem 2.4.1]{talagrand2014upper}, which gives us 
\begin{equation*}
    \gamma_2(\calC_{A^*}, \norm{.}_F) \asymp w(\calC_{A^*}).
\end{equation*}
Now, for $\operatorname{vec}(\calT_{\calK, A^*}) = \left\{\operatorname{vec}(H): H \in \calT_{\calK, A^*} \right\} \subset \mathbb{R}^{n^2}$, we have by Proposition 10.2 of \cite{tropp2014}
\begin{equation}\label{eq:isotone-width-to-statdim}
    w(\calC_{A^*}) = \mathbb{E}\left[\sup_{H \in \calC_{A^*}} \dotprod{G}{H}\right]
    = \mathbb{E}\left[\sup_{h \in \operatorname{vec}(\calT_{\calK, A^*}) \cap \mathbb{S}^{n^2}} \dotprod{\operatorname{vec}(G)}{h} \right] \leq \sqrt{\delta(\operatorname{vec}(\calT_{\calK, A^*}})),
\end{equation}
which implies that 
\begin{equation*}
    \gamma_2\left(\calC_{A^*}, \norm{.}_F\right) \lesssim \sqrt{\delta(\operatorname{vec}( \calT_{\calK, A^*}))} \leq C\sqrt{k(A^*)}\log^4(n).
\end{equation*}
The proof is now complete.

\subsection{Lemma \ref{lem:isotone-covering-numbers-net} and its proof}
\begin{lemma}\label{lem:isotone-covering-numbers-net}
For any $r, \epsilon > 0$, it holds that 
\begin{equation*}
    \calN\left(\calC_{A^*}, \norm{.}_2, \epsilon\right) \leq \calN\left(\calC_{A^*}, \norm{.}_F, r\right) \times \calN\left(r\calB_n, \norm{.}_2, \epsilon\right).
\end{equation*}
\end{lemma}

\begin{proof}[Proof of Lemma \ref{lem:isotone-covering-numbers-net}]
Here and in the lemmas that follow, we will denote the Frobenius and spectral balls by $\calB_{F}$ (which is just $\frobball$) and $\calB_2$ respectively for added clarity.

We will use a double net argument, first to cover $C_{A^*}$ with Frobenius balls, then to cover each residual term with spectral balls. Take a $r$-net $\{H_1, \dots, H_M\}$ with respect to the Frobenius norm, so that for every $H \in \calC_{A^*}$,  there exists $H_j$ among the centers $H_1, \dots, H_M$ such that $\norm{H - H_j}_F \leq r$. Denote the covering number of this net by $M := \calN(\calC_{A^*}, \norm{.}_F, r)$, and denote the residual term by $R := H - H_j$. Then $\norm{H - H_j}_F \leq r$ is equivalent to saying that $R \in r\calB_F$.

Now, cover $r\calB_F$ with spectral balls, consisting of centers $E_1, \dots, E_L$ such that for any $R \in r\calB_F$, there exists $E_{\ell}$ such that $\norm{R - E_{\ell}}_2 \leq \epsilon$. Denote the covering number of this second net by $L := \calN\left(r\calB_F, \norm{.}_2, \epsilon \right)$. By this construction, we have for every $H \in \calC_{A^*}$, there exists $j \in \{1, \dots, M\}$, $\ell \in \{1, \dots, L\}$ (possibly depending on $j$) such that
\begin{equation*}
    \norm{(H - H_j) - E_{\ell}}_2 = \norm{H - (H_j + E_{\ell})}_2 \leq \epsilon,
\end{equation*}
which means that 
\begin{equation*}
    C_{A^*} \subseteq \cup_{j=1}^M \cup_{\ell=1}^L \calB_2(H_j + E_{\ell}, \epsilon).
\end{equation*}
This implies that
\begin{align*}
    &\calN(C_{A^*}, \norm{.}_2, \epsilon) \leq ML,
\end{align*}
which completes the proof.
\end{proof}

\subsection{Lemma \ref{lem:metric-entropy-frob-ball} and its proof}

\begin{lemma}\label{lem:metric-entropy-frob-ball}
There exists a constant $C > 0$ such that for every $r > 0$, it holds that
\begin{equation*}
    \int_0^r \log\calN\left(r\calB_n, \norm{.}_2, \epsilon\right) d\epsilon \leq Crn^{3/2}.
\end{equation*}
\end{lemma}

\begin{proof}[Proof of Lemma \ref{lem:metric-entropy-frob-ball}]
It suffices to show that 
\begin{equation*}
    \int_0^1 \log\calN\left(\calB_{F}, \norm{.}_2, \varepsilon \right) d\varepsilon \lesssim n^{3/2},
\end{equation*}
since by a change-of-variable $\varepsilon = \epsilon/r$, we have
\begin{equation*}
    \int_0^r \log\calN(r\calB_F, \norm{.}_2, \epsilon) d\epsilon = \int_0^r \log\calN\left(\calB_F, \norm{.}_2, \frac{\epsilon}{r}\right) d\epsilon 
    = r\int_0^1 \log\calN(\calB_F, \norm{.}_2, \varepsilon) d\varepsilon.
\end{equation*}
To this end, we will rely on a key entropy number estimate from \cite{hinrichs2017}. In order to make the connection clear, we will translate their results to our special case using their notation. Recall that for the Schatten norms, we have $\norm{.}_{S_{\infty}} = \norm{.}_2$ and $\norm{.}_{S_2} = \norm{.}_F$. Then, in their notation, denote $\calS_p^n := (\mathbb{R}^{n\times n}, \norm{.}_{S_p})$, and set $X = \calS_2^n = (\mathbb{R}^{n\times n}, \norm{.}_F)$, $Y = \calS_{\infty}^n = (\mathbb{R}^{n\times n}, \norm{.}_2)$. Let $T:X \rightarrow Y$ be the identity operator $T(A) = A$. The respective balls are given by 
\begin{equation*}
    \calB_X = \left\{A \in \mathbb{R}^{n\times n}: \norm{A}_F \leq 1 \right\} = \calB_{F}, \quad \text{and} \quad \calB_Y := \left\{A \in \mathbb{R}^{n\times n}: \norm{A}_2 \leq 1 \right\}.
\end{equation*}
The $j$-th entropy number of the operator $T$ of our setup is given by
\begin{equation*}
    e_j(T) := \inf\bigg\{\varepsilon > 0: T(\calB_X) = \calB_F \subseteq \cup_{\ell=1}^{2^{j-1}} (y_{\ell} + \varepsilon \calB_Y), 
    \text{ for some } y_1, \dots, y_{2^{j-1}}  \bigg\},
\end{equation*}
which is the smallest achievable radius given $2^{j-1}$ spectral balls to cover the Frobenius ball. With our definition of $T$, $e_j(\calS_{p}^n \hookrightarrow \calS_{q}^n) = e_j(\calS_{2}^n \hookrightarrow \calS_{\infty}^n) = e_j(T)$. Using Theorem 1.1 of \cite{hinrichs2017}, we can reverse engineer the covering numbers at different scales $\varepsilon$. 

In particular, the range $n \leq j \leq n^2$ corresponds to $n^{-1/2} \lesssim \varepsilon \lesssim 1$. Moreover, within this range, setting $\varepsilon \asymp e_j(T) \asymp \sqrt{n/j}$ which gives $j \asymp n/\varepsilon^2$, so $2^{Cn/\varepsilon^2}$ balls suffice for a cover, so $\log\calN(\calB_F, \norm{.}_2, \varepsilon ) \asymp n/\varepsilon^2.$ 

It remains to look at the case $j \geq n^2$. In this regime, we have the range $0 \lesssim \varepsilon \lesssim n^{-1/2}$, where $\varepsilon \asymp n^{-1/2}2^{-j/n^2}$ which gives $j \asymp n^2 \log(C/\varepsilon\sqrt{n})$, so $2^{cn^2\log(C/\varepsilon\sqrt{n})}$ balls suffices as a cover. Thus, $\log\calN(\calB_F, \norm{.}_2, \varepsilon) \asymp n^2 \log(C/\varepsilon\sqrt{n})$.

Breaking the integral into two parts, we have 
\begin{equation*}
    \int_0^1 \log\calN(\calB_F, \norm{.}_2, \varepsilon) d\varepsilon  \lesssim  \int_0^{n^{-1/2}} n^2 \log\left(\frac{C}{\varepsilon\sqrt{n}}\right) d\varepsilon 
    + \int_{n^{-1/2}}^1 \frac{n}{\varepsilon^2} d\varepsilon 
    =: I_1 + I_2.
\end{equation*}
Performing a change of variable with $u = \varepsilon\sqrt{n}$, we obtain
\begin{align*}
    I_1 = n^{3/2}\int_0^1 \log\left(\frac{C}{u}\right) du = n^{3/2}\left(1 + \log(C) \right).
\end{align*}
On the other hand, we have $I_2 = n^{3/2} - n$,
so 
\begin{equation*}
    \int_0^1 \log\calN\left(\calB_F, \norm{.}_2, \varepsilon \right) d\varepsilon \lesssim n^{3/2},
\end{equation*}
as desired.
\end{proof}

\subsection{Proof of Lemma \ref{lem:isotonic-gamma1-bound}}
Since 
\begin{equation*}
    \calN(\calC_{A^*}, \norm{.}_2, 1) \leq \calN(\calC_{A^*}, \norm{.}_F, 1) = 1,
\end{equation*}
the entropy integral of $\calC_{A^*}$ with respect to $\norm{.}_2$ has upper limit $1$. By \cite{Talagrand2005-gf}, we have
\begin{equation*}
    \gamma_1\left(\calC_{A^*}, \norm{.}_2\right) \lesssim \int_0^1 \log\calN\left(\calC_{A^*}, \norm{.}_2, \epsilon \right) d\epsilon 
    = \int_0^r \log\calN(\calC_{A^*}, \norm{.}_2, \epsilon) d\epsilon + \int_r^1 \log\calN(\calC_{A^*}, \norm{.}_2, \epsilon) d\epsilon.
\end{equation*}
By Sudakov's minorization inequality and \eqref{eq:isotone-width-to-statdim}, 
\begin{equation*}
    \log\calN\left(\calC_{A^*}, \norm{.}_F, \epsilon \right) \lesssim \frac{\delta(\operatorname{vec}(\calT_{\calK, A^*}))}{\epsilon^2} \implies \int_r^1 \log\calN\left(\calC_{A^*}, \norm{.}_2, \epsilon \right) d\epsilon \lesssim \frac{\delta(\operatorname{vec}( \calT_{\calK, A^*}))}{r}.
\end{equation*}
On the other hand, we have by Lemma \ref{lem:isotone-covering-numbers-net}, 
\begin{equation*}
    \int_0^r \log\calN(\calC_{A^*}, \norm{.}_2, \epsilon) d\epsilon \leq \int_0^r \log\calN\left(\calC_{A^*}, \norm{.}_F, r \right) d\epsilon + \int_0^r \log\calN\left(r\calB_F, \norm{.}_2, \epsilon \right) d\epsilon.
\end{equation*}
Again by \eqref{eq:isotone-width-to-statdim}, we have 
\begin{equation*}
     \int_0^r \log\calN\left(\calC_{A^*}, \norm{.}_F, r \right) d\epsilon \leq \frac{ \delta(\operatorname{vec}(\calT_{\calK, A^*}))}{r}.
\end{equation*}
Now, by Lemma \ref{lem:metric-entropy-frob-ball}, we obtain
\begin{equation*}
    \int_0^r \log\calN\left(r\calB_F, \norm{.}_2, \epsilon \right) d\epsilon \lesssim rn^{3/2}.
\end{equation*}
Gathering facts, we have 
\begin{equation*}
    \gamma_1\left(\calC_{A^*}, \norm{.}_2 \right) \lesssim \frac{\delta(\operatorname{vec}(\calT_{\calK, A^*}))}{r} + rn^{3/2}.
\end{equation*}
Balancing the two terms, the optimal choice of $r$ is given by
\begin{equation*}
    r^* = \sqrt{\delta(\operatorname{vec}(\calT_{\calK, A^*}))}n^{-3/4},
\end{equation*}
and substituting back, we obtain by \eqref{eq:isotone-stat-dim-bound}
\begin{equation*}
    \gamma_1\left(\calC_{A^*}, \norm{.}_2 \right) \lesssim  n^{3/4} \sqrt{\delta(\operatorname{vec}(\calT_{\calK, A^*}))} \lesssim n^{3/4}\sqrt{k(A^*)}\log^4(n).
\end{equation*}
The proof is now complete.

\section{Lipschitz regression} \label{app-sec:lip_reg}
\subsection{Corollaries of Theorems \ref{thm:main-thm} and \ref{thm:pgd-thm}}
Following \cite[Example 4]{tyagi2024learning}, consider the setting where $A^*$ is formed as follows: each row $i$ of $A^*$ is uniformly sampled from a $L$-Lipschitz function $f_i:[0,1] \rightarrow \matR$ with a step size of $1/T$. Denote  $$\calK = \set{A \in \matR^{n \times n}:\absenc{A_{i,j}-A_{i,j+1}} \le \frac{L}{T}, i \in [n], j\in [n-1]}.$$
\begin{corollary}\label{cor:stat-lipschitz-regression}
There exist constants $c_1,c_2>1$ such that for any $\delta \in (0,1)$, if 
\begin{equation*}
T \geq c_1\frac{J^4(A^*)}{p^4}\max\bigg\{n\log n, \log^{2} \left(1/\delta\right),\frac{p^{2/3} n^{4/3} L^{2/3}}{J^{2/3}(A^*)}, \frac{p^{1/2} n^{3/2} L^{1/2} \log^{1/2} n}{J^{1/2}(A^*)}\bigg\},
\end{equation*}
then it holds with probability at least $1-\delta$ that
\begin{align*}
\norm{\hat{A}-A^*}_F &\leq c_2 \frac{J^3(A^*)}{p^2}\bigg(\frac{ n^{3/2}\log^{3/2}(n)}{T}+\frac{n^{1/2}\log^{1/2}(n)}{\sqrt{T}} \\
&+ \frac{n L^{1/3} \log^{1/3}(n)}{ T^{2/3} }  \max\set{ \frac{(n \log n)^{2/3} }{T^{1/3}},1} + \frac{\sqrt{\log \left(1/\delta\right)}}{\sqrt{T}}\bigg).
\end{align*}
\end{corollary}
The proof closely mirrors that of \cite[Corollary 4]{tyagi2024learning}. Using the $\gammafn{1}(\conesetB, \norm{\cdot}_2)$ and $\gammafn{2}(\conesetB, \norm{\cdot}_F)$ bounds derived in \cite{tyagi2024learning}, we show that there exists a $B \in \calK$ close to $A^*$, where each row of $B$ is formed by taking $l+1$ affine pieces whose slopes are equal to $\pm L$ and $\absenc{B_{i,j}-B_{i,j+1}} = \frac{L}{T}$. Then using Theorem~\ref{thm:main-thm} and choosing $l$ to balance the recovery and misspecification error, we get the desired bound. The details are in Appendix~\ref{appsec:stat-lipschitz-regression}

\begin{remark}
    For $p = 1$, we recover the statement of Corollary 4 of \cite{tyagi2024learning} up to constants and a $J(A^*)$ factor. 
\end{remark}

\begin{corollary}\label{cor:pgd-lipschitz-regression}
Suppose that the learning rate is chosen as $\gamma = \frac{1}{2p^2 T J^2(A^*) }$. Then there exist constants $c,C,C' > 1$ such that for any $\delta \in (0,1)$, if 
\begin{align*}
        T &\geq C'\frac{J^4(A^*)}{p^4}\max\left\{n\log n, \log^{2} \left(1/\delta\right), \frac{p^{2/3} n^{4/3} L^{2/3}}{J^{2/3}(A^*)},\frac{p^{1/2} n^{3/2} L^{1/2} \log^{1/2} n}{J^{1/2}(A^*)}\right\}\\
        &+C'\frac{J^2(A^*)}{p^2}\max\left\{n^{3/2}\log^{3/2} n, \log \left(1/\delta\right),\frac{p^{1/3} n^{5/3} L^{1/3}}{J^{1/3}(A^*)}, \frac{p^{2/7} n^{12/7} L^{2/7} \log^{8/7} n}{J^{2/7}(A^*)}\right\},
    \end{align*} 
then with probability at least $1 - \delta$, for all $k = 0,1,\dots$, it holds that
    \begin{align*}
        & \norm{A^{(k)}-A^*}_F \le \rho^k \norm{A^{(0)}-A^*}_F +\frac{\varepsilon}{1-\rho}
    \end{align*}
where $\rho \le 1-\frac{1}{cJ^2(A^*)}$ and
\begin{equation*}
    \varepsilon =C\frac{J(A^*)}{p^2} \bigg(\max\set{\frac{ n^{3/2}\log^{3/2} n}{T},\frac{n^{1/2}\log^{1/2} n}{\sqrt{T}}} + \frac{ n (L \log n)^{1/3}}{ T^{2/3} }  \max\set{ \frac{(n \log n)^{2/3} }{T^{1/3}},1} + \frac{\sqrt{\log \left(1/\delta\right)}}{\sqrt{T}}
\bigg).
\end{equation*}
\end{corollary}

The proof follows from using the choice of $l$ (and therefore $B$) derived in Lemma~\ref{cor:pgd-lipschitz-regression}, and applying Theorem~\ref{thm:pgd-thm}, before deriving the appropriate bounds for $\rho, \varepsilon$ and $T$. The proof is outlined in Appendix~\ref{appsec:pgd-lipschitz-regression}.

\begin{remark}
    In Corollary \ref{cor:stat-lipschitz-regression}, the sharp boundary for $\norm{\widehat A - A^*}_F = o(1)$ when $p = o(1)$ is a long and ugly expression. However, it is easy to see that for a fixed level $\delta$, a sufficient scaling regime is
\begin{equation}\label{eq:p-floor-lipschitz}
    J^{3/2}(A^*)\max\left\{1, L^{1/6} \frac{n^{5/6}\log^{3/4}(n)}{T^{1/4}}  \right\} = o(p),
\end{equation}
with $p = o(1)$. 
Then, we have 
\begin{equation*}
    \limsup_{k \rightarrow \infty} \norm{A^{(k)} - A^*}_F \leq \frac{\varepsilon}{1 - \rho}.
\end{equation*}
\end{remark}

\begin{remark}
For the projection step in Lipschitz regression, each row projection is a convex-quadratic program with $n$ variables and $O(n)$ linear constraints. This can be solved, e.g., using interior-point methods with an overall running time polynomial in $n$.
\end{remark}

\subsection{Experiments}
We sample from a piecewise linear function to construct the transition matrix $A^*$. For $|r_{ik}| \leq 1$, let $g_i(x) = L r_{ik},$ whenever $x \in [k-1/T, k/T)$, be a piecewise constant function. Define $f_i(x) = a_i + \int_0^x g_i(u) du$, so $f_i:[0, (n-1)/T] \rightarrow \mathbb{R}_+$ is piecewise linear, and construct
\begin{equation*}
    A_{ij}^* = f_i\left(\frac{j-1}{T}\right)
    = a_i + \frac{L}{T}\sum_{k=1}^{j-1} r_{ik}.
\end{equation*}
Since $|g_i(u)| \leq L$, $f_i$ is $L$-Lipschitz by construction, so $A^*$ is row-wise Lipschitz. To control the similarities between rows, we define for any $\lambda \in (0, 1)$,
\begin{equation*}
    r_{ij} = \lambda z_{ij} + (1-\lambda)q_j,
\end{equation*}
where $z_{ij} \stackrel{i.i.d}{\sim} \text{Unif}([-1, 1])$ and $q_j \stackrel{i.i.d}{\sim} \text{Unif}([-1, 1])$, $j=1, \dots, n-1$. We set $\lambda = 0.3$, where the generated intercept $a = (a_1, \dots, a_n)$ is a standard gaussian vector with unit norm scaled by $\alpha = 0.3$. The spectral radius was imposed to be at most $0.8$ by rejection sampling (i.e., keeping the first $A^*$ whose spectral radius was at most $0.8$). The $A^*$'s are fixed for each configuration, and only differs across configuration by their scaling $L/T$ to make them more comparable. We set $n = 100$, and explore all combinations of $T \in \{100, 200, 400 \}$, $p \in \{0.3, 0.5, 0.9\}$, $L \in \{1, 1.5\}$
with $200$ replications in each configuration. The projection and PGD tolerance were set to $10^{-4}$ and $10^{-3}$ respectively, with 350 maximum iterations. The projection step was performed by solving another constrained optimization problem using \texttt{cvxpy}, using the \texttt{OSQP} solver.

Median errors together with the 25th to 75th quantile bands can be seen in Figure \ref{fig:lipschitz}. Since the norm can vary significantly between different trajectory lengths $T$, we report the absolute Frobenius norms instead of the relative ones. As expected, the results are worse whenever the Lipschitz constant is bigger, with the errors generally moving in the same direction as the trajectory lengths. For $L = 1.5$, we see a small increase of errors when moving from $T = 100$ to $T = 200$ whenever $p \in \{0.3, 0.5\}$. This can be attributed to the instability of the PGD at short trajectory lengths for this larger Lipschitz constant, since the maximum number of iterations were hit before convergence for these cases.

\begin{figure}
    \centering
    \resizebox{0.7\linewidth}{!}{
    \input{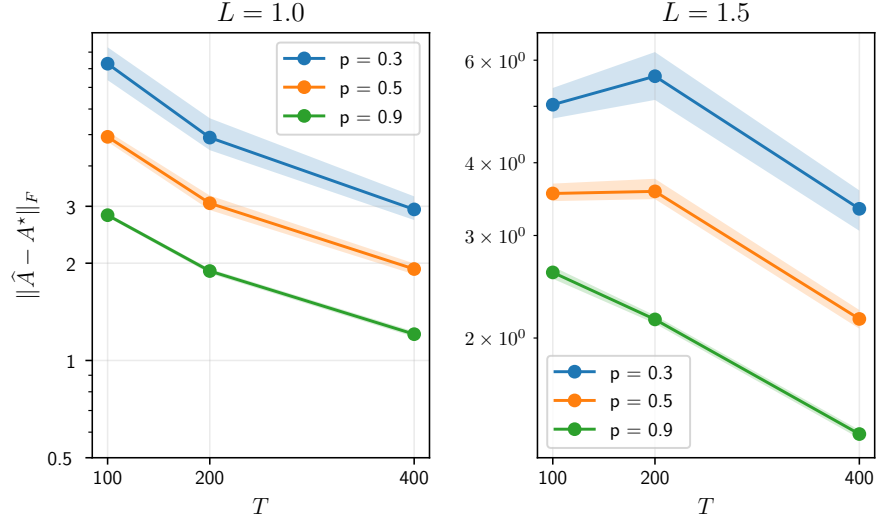}
    }
    \caption{Absolute Frobenius errors for row-wise Lipschitz matrices.}
    \label{fig:lipschitz}
\end{figure}

\end{document}